\documentclass[twoside,11pt]{article}

\usepackage{url}
\usepackage{placeins}
\usepackage{dsfont}
\usepackage{mathrsfs}
\usepackage{amsmath}
\usepackage{bm}
\usepackage[english]{babel}
\usepackage{csquotes}
\usepackage{mathtools}

\usepackage{algorithm}
\usepackage{algorithmic}
\usepackage{eqparbox}

\usepackage{subcaption}

\usepackage[preprint]{jmlr2e}

\DeclarePairedDelimiter\floor{\lfloor}{\rfloor}
\newcommand{\punif}{p^{\rm unif}}

\newcommand\numberthis{\addtocounter{equation}{1}\tag{\theequation}}

\newcommand\blfootnote[1]{%
  \begingroup
  \renewcommand\thefootnote{}\footnote{#1}%
  \addtocounter{footnote}{-1}%
  \endgroup
}

\usepackage{xcolor}         
\newcommand{\ifcomments}{\iftrue}

\newtheorem{assump}[theorem]{Assumption}

\usepackage{lastpage}
\jmlrheading{26}{2025}{1-\pageref{LastPage}}{3/24; Revised
10/24}{1/25}{24-0385}{Boxin Zhao, Lingxiao Wang, Ziqi Liu, Zhiqiang Zhang, Jun Zhou, Chaochao Chen, and Mladen Kolar}

\ShortHeadings{Adaptive Client Sampling in Federated Learning}{Zhao, Wang, Liu, Zhang, Zhou, Chen and Kolar}
\firstpageno{1}

\begin{document}

\title{Adaptive Client Sampling in Federated Learning via \\Online Learning with Bandit Feedback}

\author{\name Boxin Zhao \email boxinz@uchicago.edu \\
       \addr University of Chicago Booth School of Business       \\
       Chicago, IL, 60637, USA
       \AND
       \name Lingxiao Wang \email lw324@njit.edu \\
       \addr New Jersey Institute of Technology Department of Data Science \\
       Newark, NJ, 07106, USA
       \AND
       \name Ziqi Liu \email ziqiliu@antfin.com \\
       \name Zhiqiang Zhang \email lingyao.zzq@antfin.com \\
       \name Jun Zhou \email jun.zhoujun@antfin.com \\
       \addr Ant Financial Group\\
       Hangzhou, Zhejiang, China
       \AND
       \name Chaochao Chen \email zjuccc@zju.edu.cn \\
       \addr College of Computer Science and Technology \\
       Zhejiang University\\
       Hangzhou, Zhejiang, China
       \AND
       \name Mladen Kolar \email mkolar@marshall.usc.edu  \\
       \addr University of Southern California Marshall School of Business \\
       Los Angeles, CA, 90089, USA
       }

\editor{Francesco Orabona}

\maketitle

\begin{abstract}
\blfootnote{Correspondence: Mladen Kolar (mkolar@marshall.usc.edu) and Jun Zhou (jun.zhoujun@antfin.com).}
Due to the high cost of communication, federated learning (FL) systems need to sample a subset of clients that are involved in each round of training. As a result, client sampling plays an important role in FL systems as it affects the convergence rate of optimization algorithms used to train machine learning models. Despite its importance, there is limited work on how to sample clients effectively. In this paper, we cast client sampling as an online learning task with bandit feedback, which we solve with an online stochastic mirror descent (OSMD) algorithm designed to minimize the sampling variance. We then theoretically show how our sampling method can improve the convergence speed of federated optimization algorithms over the widely used uniform sampling.
Through both simulated and real data experiments, we empirically illustrate the advantages of the proposed client sampling algorithm over uniform sampling and existing online learning-based sampling strategies. The proposed adaptive sampling procedure is applicable beyond the FL problem studied here and can be used to improve the performance of stochastic optimization procedures such as stochastic gradient descent and stochastic coordinate descent.
\end{abstract}

\begin{keywords}
federated learning, client sampling, online learning, optimization, \\ variance reduction
\end{keywords}

\section{Introduction}
\label{sec:intro}

Modern edge devices, such as personal mobile phones, wearable devices, and sensor systems in vehicles, collect large amounts of data that are valuable for training of machine learning models. If each device only uses its local data to train a model, the resulting generalization performance will be limited due to the number of available samples on each device. Traditional approaches where data are transferred to a central server, which trains a model based on all available data, have fallen out of fashion due to privacy concerns and high communication costs. Federated Learning (FL) has emerged as a paradigm that allows for collaboration between different devices (clients) to train a global model while keeping data locally and only exchanging model updates~\citep{mcmahan2017communication}.

In a typical FL process, we have clients that contain data and a central server that orchestrates the training process~\citep{kairouz2021advances}. The following process is repeated until the model is trained: (i) the server selects a subset of available clients; (ii) the server broadcasts the current model parameters and sometimes also a training program (e.g., a Tensorflow graph~\citep{abadi2016tensorflow}); (iii) the selected clients make updates to the model parameters based on their local data; (iv) the local model updates are uploaded to the server; (v) the server aggregates the local updates and makes a global update of the shared model. In this paper, we focus on the first step and develop a practical strategy for selecting clients with provable guarantees. 

To train a machine learning model in a FL setting with $M$ clients, we would like to minimize the following objective\footnote{We use $[M]$ to denote the set $\{1,\dots,M\}.$}:
\begin{equation}
\label{eq:SingleFL-obj}
\min_{w} F (w) \coloneqq \sum_{m \in [M]} \lambda_m \phi \left( w  ;  \mathcal{D}_m \right),
\end{equation}
where $\phi (w ; \mathcal{D}_m)$ is the loss function used to assess the quality of a machine learning model parameterized by the vector $w$ based on the local data $\mathcal{D}_m$ on the client $m \in [M]$. The parameter $\lambda_m$ denotes the weight for client $m$. Typically, we have $\lambda_m=n_m/n$, where $n_m=\vert \mathcal{D}_m \vert$ is the number of samples on the client $m$, and the total number of samples is $n = \sum^M_{m=1} n_m$. At the beginning of the $t$-th communication round, the server uses the sampling distribution $p^t=(p^t_1,\dots,p^t_M)^{\top}$ to choose $K$ clients by sampling with replacement from $[M]$\footnote{In this paper, we assume that all clients are available in each round and the purpose of client sampling is to reduce the communication cost, which is also the case considered by some previous research~\citep{chen2020optimal}. However, in practice, it is possible that only a subset of clients are available at the beginning of each round due to physical constraint. In Appendix~\ref{sec:samp-phy-constraint}, we discuss how to extend our proposed methods to deal with such situations. Analyzing such an extension is highly non-trivial and we leave it for further study. See detailed discussion in Appendix~\ref{sec:samp-phy-constraint}.}. Let $S^t \subseteq [M]$ denote the set of chosen clients with $|S^t| = K$. The server transmits the current model parameter vector $w^t$ to each client $m \in S^t$. The client $m$ computes the local update $g^t_m$\footnote{Here by model update, we actually mean the negative direction of model update. For example, when applying gradient descent, we refer the gradient as the model update, while the model parameter makes an update at the direction of the negative gradient. We stick with the term model update since it is more commonly used.} and sends it back to the server\footnote{Throughout the paper, except in Section~\ref{sec:convg-analysis-batchSGD}, we do not specify how $g^t_m$ is obtained. One possibility that the reader could keep in mind for concreteness is the LocalUpdate algorithm~\citep{charles2020outsized}, which covers well-known algorithms such as mini-batch SGD and FedAvg~\citep{mcmahan2017communication}.}. After receiving local updates from clients in $S^t$, the server constructs a stochastic estimate of the global gradient as
\begin{equation}
\label{eq:grad-est}
g^t = \frac{1}{K} \sum_{m \in S^t} \frac{\lambda_m}{p^t_m} g^t_m,
\end{equation}
and makes the global update of the parameter $w^t$ using $g^t$. For example, $w^{t+1} = w^t - \mu^t g^t$, if the server is using stochastic gradient descent (SGD) with the stepsize sequence $\{\mu^t\}_{t\geq 1}$ \citep{bottou2018optimization}. However, the global update can be obtained using other procedures as well.

The sampling distribution in FL is typically uniform over clients, that is, $p^t=\punif=(1/M,\dots,1/M)^{\top}$. However, nonuniform sampling (also called importance sampling) can lead to faster convergence, both in theory and practice, as has been illustrated in stochastic optimization \citep{zhao2015stochastic,needell2014stochastic}. While the sampling distribution can be designed based on prior knowledge \citep{zhao2015stochastic,johnson2018training,needell2014stochastic,stich2017safe}, we cast the problem of choosing the sampling distribution as an online learning task and need no prior knowledge about~\eqref{eq:SingleFL-obj}.

Existing approaches to designing a sampling distribution using online learning takes a stationary online learning framework and focus on matching the best sampling distribution that does not change over the training process.
As a result, the obtained algorithms update the sampling distribution by treating all history information equally.
However, as the training proceeds, the best sampling distribution changes with iterations.
To address this problem, we take a non-stationary online learning framework and use an online stochastic mirror descent (OSMD) algorithm that puts more emphasize on recent feedback and learns to 'forget' the history.
Consequently, our method shows empirical advantages over the previous methods.
Besides, we derive a dynamic regret upper bound that allows the comparators to change with iterations, which generalizes the theoretical results on static regret of previous research.
Moreover, we provably show how our sampling method improves the convergence guarantee of federated optimization methods 
over uniform sampling by reducing the dependency on the heterogeneity of the problem, which is also new to the best of our knowledge.

\subsection{Notation}
Let $\mathbb{R}^M_{+}=[0,\infty)^M$ and $\mathbb{R}^M_{++}=(0,\infty)^M$. For $M \in \mathbb{N}^{+}$, let $\mathcal{P}_{M-1} \coloneqq \{ x \in \mathbb{R}^M_{+} : \sum^M_{i=1} x_i = 1 \}$ be the $(M-1)$-dimensional probability simplex. We use $p=(p_1,\dots,p_M)^{\top}$ to denote a sampling distribution with support on $[M] \coloneqq \{1,\dots,M\}$. We use $p^{1:T}$ to denote a sequence of sampling distributions $\{p^t\}_{t=1}^T$. We use $\Vert \cdot \Vert_p$ to denote the $L_p$-norm for $1\leq p \leq \infty$. For $x \in \mathbb{R}^n$, we have $\Vert x \Vert_p = ( \sum^n_{i=1} x^p_i )^{1/p}$ when $1<p<\infty$,  $\Vert x \Vert_1 = \sum^n_{i=1} \vert x_i \vert $, and $\Vert x \Vert_{\infty} = \max_{1\leq i \leq n} \vert x_i \vert $. Given any $L_p$-norm $\Vert \cdot \Vert$, we define its dual norm as $\Vert z \Vert_{\star} \coloneqq \sup \{ z^{\top}x : \Vert x \Vert \leq 1  \}$. We use $|\mathcal{B}|$ to denote the cardinality of the index set $|\mathcal{B}|$.

Let $\Phi:\mathcal{D}\subseteq\mathbb{R}^M \mapsto\mathbb{R}$ be a differentiable convex function defined on $\mathcal{D}$, where $\mathcal{D}$ is a convex open set, and we use $\bar{\mathcal{D}}$ to denote the closure of $\mathcal{D}$. The Bregman divergence between any $x,y \in \mathcal{D}$ with respect to the function $\Phi$ is given as $D_{\Phi} \left(x \,\Vert\, y\right) = \Phi (x) - \Phi (y) - \langle \nabla \Phi (y), x - y \rangle$. The unnormalized negative entropy is denoted as $\Phi(x)=\sum^M_{m=1} x_m \log x_m - \sum^M_{m=1}x_m$, $x=(x_1,\dots,x_M)^{\top} \in \mathcal{D}=\mathbb{R}^M_{+}$, with $0\log0$ defined as $0$.

For two positive sequences $\{a_n\}$ and $\{b_n\}$, we use $a_n=O(b_n)$ to denote that there exists $C>0$ such that $a_n/b_n \leq C$ for all $n$ large enough.
Similarly, we use $a_n=\Omega(b_n)$ to denote that there exists $c>0$ such that $a_n / b_n \geq c$ for all $n$ large enough.
We denote $a_n=\Theta(b_n)$ if $a_n=O(b_n)$ and $a_n=\Omega(b_n)$ simultaneously.
Besides, we use $a_n=o(b_n)$ if $\lim_n a_n/b_n =0$.
In addition, $a_n=\tilde{O} (b_n)$ if $a_n=O(b_n \log^k b_n)$ for some $k\geq 0$.


\subsection{Organization of the Paper}

We summarize the related work in Section~\ref{sec:related-work}.
We motivate importance sampling in FL and introduce an adaptive client sampling algorithm in Section~\ref{sec:client-sampling-online}.
We develop optimization guarantees of mini-batch SGD and FedAvg using our sampling scheme in Section~\ref{sec:convg-analysis-batchSGD}.
In Section~\ref{sec:thm-osmd-sampler}, we discuss regret analysis which serves as a key component for optimization analysis.
In Section~\ref{sec:adp-osmd-sampler}, we propose an extension of the sampling method that is adaptive to the unknown problem parameters. 
We provide the experimental results on simulated data in Section~\ref{sec:experiments} and real-world data in Section~\ref{sec:real-data-exp}.
We conclude our paper with Section~\ref{sec:conclusion}.
We leave all the technical proofs and additional discussions in Appendix.
Code to replicate the results in this paper is available at\\[1em]
\hphantom{.................}\url{https://github.com/boxinz17/FL-Client-Sampling}.



\section{Related Work}
\label{sec:related-work}

Our paper is related to client sampling in FL, importance sampling in stochastic optimization, and online convex optimization. We summarize only the most relevant literature, without attempting to provide an extensive survey.

For client sampling, \citet{chen2020optimal} proposed to use the theoretically optimal sampling distribution to choose clients. However, their method requires all clients to compute local updates in each round, which is impractical due to stragglers. \citet{ribero2020communication} modelled the parameters of the model during training by an Ornstein-Uhlenbeck process, which was then used to derive an optimal sampling distribution. \citet{cho2020client} developed a biased client selection strategy and analyzed its convergence property. As a result, the algorithm has a non-vanishing bias and is not guaranteed to converge to optimum. Moreover, it needs to involve more clients than our method and is thus communication and computational more expensive. \citet{kim2020accurate, cho2020bandit, yang2020federated} considered client sampling as a multi-armed bandit problem, but provided only limited theoretical results. \citet{wang2020optimizing} used reinforcement learning for client sampling with the objective of maximizing accuracy, while minimizing the number of communication rounds.

Our paper is also closely related to importance sampling in stochastic optimization. \citet{zhao2015stochastic,needell2014stochastic} illustrated that by sampling observations from a nonuniform distribution when using a gradient-based stochastic optimization method, one can achieve faster convergence. They designed a fixed sampling distribution using prior knowledge on the upper bounds of gradient norms. \citet{csiba2018importance} extended the importance sampling to mini-batches. \citet{stich2017safe, johnson2018training, gopal2016adaptive} developed adaptive sampling strategies that allow the sampling distribution to change over time. \citet{nesterov2012efficiency, perekrestenko2017faster, allen2016even, salehi2018coordinate} discussed importance sampling in stochastic coordinate descent methods. \citet{namkoong2017adaptive, salehi2017stochastic, borsos2018online, borsos2019online, el2020adaptive} illustrated how to design the sampling distribution by solving an online learning task with bandit feedback. \citet{namkoong2017adaptive, salehi2017stochastic} designed the sampling distribution by solving a multi-armed bandit problem with the EXP3 algorithm~\citep[Chapter 11]{lattimore2020bandit}. \citet{borsos2018online} used the follow-the-regularized-leader algorithm~\citep[Chapter 28]{lattimore2020bandit} to solve an online convex optimization problem and make updates to the sampling distribution. \citet{borsos2019online} restricted the sampling distribution to be a linear combination of distributions in a predefined set and used an online Newton step to make updates to the mixture weights. The above approaches estimate a stationary distribution, while the best distribution is changing with iterations and, therefore, is intrinsically dynamic. In addition to having suboptimal empirical performance, these papers provide theoretical results that only establish a regret relative to a fixed sampling distribution in hindsight. To address this problem, \citet{el2020adaptive} took a non-stationary approach where the most recent information for each client was kept. A decreasing stepsize sequence is required to establish a regret bound. 
In comparison, we establish a regret bound relative to a dynamic comparator---a sequence of sampling distributions---without imposing assumptions on the stepsize sequence, and this bound includes the dependence on the total variation term characterizing how strong the comparator is.

Our paper also contributes to the literature on online convex optimization. We cast the client sampling problem as an online learning problem~\citep{hazan2016introduction} and adapt algorithms from the dynamic online convex optimization literature to solve it. \citet{hall2015online, yang2016tracking, daniely2015strongly} proposed methods that achieve sublinear dynamic regret relative to dynamic comparator sequences. In particular, \citet{hall2015online} used a dynamic mirror descent algorithm to achieve sublinear dynamic regret with total variation characterizing the intrinsic difficulty of the environment. 
Compared with the problem settings in the above studies, there are two key new challenges that we need to address. First, we only have partial information---bandit feedback---instead of the full information about the loss functions. Second, the loss functions in our case are unbounded, which violates the common boundedness assumption in the online learning literature. To overcome the first difficulty, we construct an unbiased estimator of the loss function and its gradient, which are then used to make an update to the sampling distribution. We address the second challenge by first bounding the regret of our algorithm when the sampling distributions in the comparator sequence lie in a region of the simplex for which the loss is bounded, and subsequently analyze the additional regret introduced by projecting the elements of the comparator sequence to this region.

\section{Adaptive Client Sampling}
\label{sec:client-sampling-online}

We show how to cast the client sampling problem as an online learning task. Subsequently, we solve the online learning problem using the OSMD algorithm.

\subsection{Client Sampling as an Online Learning Problem}
\label{sec:recast-as-ol-problem}

Recall that at the beginning of the $t$-th communication round, the server uses a sampling distribution $p^t$ to choose a set of clients $S^t$, by sampling with replacement $K$ clients from $[M]$, to update the parameter vector $w^t$. For a chosen client $m \in S^t$, the local update is denoted as $g^t_m$. For example, the local update $g^t_m=\nabla \phi(w^t;\mathcal{D}_m)$ may be the full gradient; when mini-batch SGD/FedSGD is used, then $g^t_m=(1/B)\sum^B_{b=1}\nabla \phi ( w^t ; \xi^{t,b}_m )$, where $\xi^{t,b}_m \overset{i.i.d.}{\sim} \mathcal{D}_m$ and $B$ is the batch size; when FedAvg~\citep{mcmahan2017communication} is used, then $g^t_m = w^t - w^{t.B}_m$, where $w^{t,b}_m = w^{t,b-1}_m - \mu^t_l \nabla f \left( w^{t,b-1}_m ; \xi^{t,b-1}_m \right)$, $b =0,\ldots,B-1$, $w^{t,0}_m=w^t$, $\xi^{t,b}_m \overset{i.i.d.}{\sim} \mathcal{D}_m$, and $\mu^t_l$ is the local stepsize at $t$-th communication round.

The randomness comes from two sampling processes. The first sampling happens on clients level, and the second sampling happens locally when computing local updates.
Client sampling is dealing with the first randomness.
Since the two sampling process are independent, we may treat $g^t_m$ as deterministic in this section to ease the understanding. We will include the second randomness when analyzing regret and specific optimization algorithms in following sections.

We define the aggregated oracle update at the $t$-th communication round as
\begin{equation*}
J^t = \sum^M_{m=1} \lambda_m g^t_m.
\end{equation*}
The oracle update $J^t$ is constructed only for theoretical purposes and is not computed in practice. The stochastic estimate $g^t$, defined in~\eqref{eq:grad-est}, is an unbiased estimate of $J^t$, that is, $\mathbb{E}_{S^t}\left[ g^t \right]=J^t$. 
The variance of $g^t$ is
\begin{equation}
\label{eq:var-grad-est}
\begin{aligned}
\mathbb{V}_{S^t}[g^t]
= \frac{1}{K} \left( \sum^M_{m=1} \frac{\lambda_m^2 \Vert g^t_m \Vert^2_2}{p^t_m} - \left\Vert J^t \right\Vert^2_2 \right).
\end{aligned}
\end{equation}
Our goal is to design the sampling distribution $p^t$, used to sample $S^t$, to minimize the variance in \eqref{eq:var-grad-est}. In doing so, we can ignore the second term, as it is independent of $p^t$. 
Minimizing variance is our goal in designing the sampling distribution because we require $g^t$ to be an unbiased estimate of $J^t$. Allowing $g^t$ to be biased, as in the biased client selection literature~\citep{cho2020client,qu2022context,ribero2020communication}, may render minimizing variance ineffective. Our focus is on unbiased client selection, leaving biased client selection for future research.

Let $a^t_m=\lambda^2_m \Vert g^t_m \Vert^2_2$.
For any sampling distribution $q = (q_1,\dots,q_M)^{\top}$, the \emph{variance reduction loss}\footnote{The variance reduction loss $l_t(\cdot)$ should be distinguished from the training loss $\phi(\cdot)$. While the former is always convex, $\phi(\cdot)$ can be non-convex.} is defined as
\begin{equation}
\label{eq:var-reduc-loss}
l_t(q) = \frac{1}{K} \sum^M_{m=1} \frac{a^t_m}{q_m}.
\end{equation}
Then for $S^t$ sampled via $q$, we have
\begin{equation*}
\mathbb{V}_{S^t}[g^t] = l_t(q) - \frac{1}{K} \left\Vert J^t \right\Vert^2_2.
\end{equation*}
Given a sequence of sampling distributions $q^{1:T}$, the cumulative variance reduction loss is defined as $L(q^{1:T}) \coloneqq \sum^T_{t=1} l_t(q^t)$. When the choice of $q^{1:T}$ is random, the expected cumulative variance reduction loss is defined as $\bar{L} (q^{1:T})  \coloneqq \mathbb{E}\left[ L(q^{1:T}) \right]$.

The variance reduction loss appears in the bound on the sub-optimality of a stochastic optimization algorithm. As a motivating example, suppose $F(\cdot)$ in \eqref{eq:SingleFL-obj} is $\nu$-strongly convex. Furthermore, suppose the local update $g^t_m=\nabla \phi(w^t;\mathcal{D}_m)$ is the full gradient of the local loss and the global update is made by SGD with stepsize $\mu^t=2/(\nu t)$. Theorem 3 of~\cite{salehi2017stochastic} then states that for any $T\geq 1$:
\begin{equation}
\label{eq:thm3-salehi2017stochastic}
\mathbb{E} \left[  F \left( \frac{2}{T(T+1)} \sum^T_{t=1}t \cdot w^t \right) \right] - F(w^{\star}) \leq \frac{2}{\nu T (T+1)} \bar{L}(p^{1:T}),
\end{equation}
where $w^{\star}$ is the minimizer of the objective in~\eqref{eq:SingleFL-obj}. Therefore, by choosing the sequence of sampling distributions $p^{1:T}$ to make the $\bar{L}(p^{1:T})$ small, one can achieve faster convergence. This observation holds in other stochastic optimization problems as well.
We develop an algorithm that creates a sequence of sampling distributions $p^{1:T}$ to minimize $\bar{L}(p^{1:T})$ using only the norm of local updates, and without imposing assumptions on the loss functions or how the local and global updates are made. 
As a result, the proposed sampling algorithm is agnostic to both optimization algorithms and optimization problems.
In Section~\ref{sec:convg-analysis-batchSGD}, we show how our sampling method improves over uniform sampling by providing tighter upper bounds on mini-batch SGD and FedAvg with non-convex objective $\phi(\cdot)$.

Suppose that at the beginning of the $t$-th communication round we know all $\{ a^t_m \}^M_{m=1}$. Then the optimal sampling distribution 
\[
p^t_{\star} = (p^t_{\star,1}, \dots, p^t_{\star,M})^{\top} = \arg\min_{p \in \mathcal{P}_{M-1}} l_t(p)
\]
is obtained as $p^t_{\star,m} = \sqrt{a^t_m} / (\sum^M_{m=1}\sqrt{a^t_m})$. Computing the distribution $p^t_{\star}$ is impractical as it requires local updates of all clients, which eradicates the need for client sampling. From the form of $p^t_{\star}$, we observe that clients with a large $a^t_m$ are  more ``important'' and should have a higher probability of being selected. Since we do not know $\{ a^t_m \}^M_{m=1}$, we will need to explore the environment to learn about the importance of clients before we can exploit the best strategy. Finally, we note that the relative importance of clients will change over time, which makes the environment dynamic and challenging.

Based on the above discussion, we cast the problem of creating a sequence of sampling distributions as an online learning task with bandit feedback, where a game is played between the server and environment. Let $p^1$ be the initial sampling distribution. At the beginning of iteration $t$, the server samples with replacement $K$ clients from $[M]$, denoted by $S^t$, using $p^t$. The environment reveals $\{a^t_m\}_{m \in S^t}$ to the server, where $a^t_m=\lambda^2_m \Vert g^t_m \Vert^2_2$. The environment also computes $l_t(p^t)$; however, this loss is not revealed to the server. The server then updates $p^{t+1}$ based on the feedback $\{ \{ a^u_m \}_{m \in S^u} \}^t_{u=1}$ and sampling distributions $\{ p^u \}^t_{u=1}$. Note that in this game, the server only gets information about the chosen clients and, based on this partial information, or bandit feedback, needs to update the sampling distribution. On the other hand, we would like to be competitive with an oracle that can calculate the cumulative variance reduction loss. We will design $p^t$ in a way that is agnostic to the generation mechanism of $\{a^t\}_{t\geq 1}$, and will treat the environment as deterministic, with randomness coming only from $\{S^t\}_{t\geq1}$ when designing $p^t$. We describe an OSMD-based approach to solve this online learning problem.

\subsection{OSMD Sampler}
\label{sec:osmd-sampler}

Note that the variance-reduction loss function $l_t$ is a convex function on $\mathcal{P}_{M-1}$ and
\begin{equation*}
\nabla l_t (q) = - \frac{1}{K} \left( \frac{ a^t_1 }{ (q_1 )^2},\dots,\frac{ a^t_M }{ (q_M )^2} \right)^{\top} \in \mathbb{R}^M \quad \text{for all}\quad q=(q_1,\dots,q_M)^{\top} \in \mathbb{R}^M_{++}.
\end{equation*}
Since we do not observe $a^t$, we cannot compute $l_t (\cdot)$ or $\nabla l_t (\cdot)$. Instead, we can construct unbiased estimates of them. For any $q \in \mathcal{P}_{M-1}$, let $\hat{l}_t (q ; p^t )$ be an estimate of $l_t (q)$ defined as 
\begin{equation}
\label{eq:est-loss}
\hat{l}_t (q ; p^t ) = \frac{1}{K^2} \sum^M_{m=1} \frac{ a^t_m }{ q_m p^t_m } \mathcal{N} \left\{ m \in S^t \right\},
\end{equation}
and $\nabla \hat{l}_t (q ; p^t ) \in \mathbb{R}^M$ has the $m$-th entry defined as 
\begin{equation}
\label{eq:est-grad}
\left[ \nabla \hat{l}_t (q ; p^t ) \right]_m = - \frac{1}{K^2} \cdot \frac{ a^t_m }{ q^2_m p^t_m } \mathcal{N} \left\{ m \in S^t \right\}.
\end{equation}
The set $S^t$ is sampled with replacement from $[M]$ using $p^t$ and $\mathcal{N} \left\{ m \in S^t \right\}$ denotes the number of times that a client $m$ is chosen in $S^t$. Thus, $0 \leq \mathcal{N} \left\{ m \in S^t \right\} \leq K$.
Given $q$ and $p^t$, $\hat{l}_t (q;p^t)$ and $\nabla \hat{l}_t (q;p^t)$ are random variables in $\mathbb{R}$ and $\mathbb{R}^M$ that satisfy
\begin{align*}
\mathbb{E}_{S^t} \left[ \hat{l}_t (q ; p^t) \mid p^t \right] = l_t (q), \qquad
\mathbb{E}_{S^t} \left[ \nabla \hat{l}_t (q ; p^t) \mid p^t \right] = \nabla l_t (q).
\end{align*}
When $S^t$ and $p^t \in \mathbb{R}^M_{++}$ are given, $\hat{l}_t (q;p^t)$ is a convex function with respect to $q$ on $\mathbb{R}^{M}_{++}$ and satisfies $\hat{l}_t (q;p^t)-\hat{l}_t (q^{\prime};p^t) \leq \langle \nabla \hat{l}_t (q;p^t),q - q^{\prime}  \rangle$, for $q,q^{\prime} \in \mathbb{R}^{M}_{++}$. The constructed estimates $\hat{l}_t (q ; p^t) $ and $\nabla \hat{l}_t (q ; p^t )$ are crucial for designing updates to the sampling distribution.

\begin{algorithm}[t]
\caption{OSMD Sampler}
\label{alg:OSMD-sampler}
\begin{algorithmic}[1]
\STATE {\bfseries Input:} Learning rate $\eta$, parameter $\alpha \in (0,1]$, 
and number of iterations $T$.
\STATE {\bfseries Output:} $\hat{p}^{1:T}$.
\STATE {\bfseries Initialize:} $\hat{p}^1=\punif$.
\FOR{$t=1,2,\dots,T-1$}{
\STATE Sample $S^t$ by $\hat{p}^t$. \label{alg:sample-subset}
\STATE Compute $\nabla \hat{l}_t (\hat{p}^t ; \hat{p}^t)$ via \eqref{eq:est-grad}.
\STATE
$\displaystyle \hat{p}^{t+1} = \arg\min_{ p \in \mathcal{A} } \, \eta \langle p,\nabla \hat{l}_t (\hat{p}^t ; \hat{p}^t)  \rangle + D_{\Phi} \left( p \,\Vert\, \hat{p}^t  \right)$. \label{alg:solve-mirror-descent}
}
\ENDFOR
\end{algorithmic}
\end{algorithm}

OSMD Sampler is an online stochastic mirror descent algorithm
for updating the sampling distribution, detailed in Algorithm~\ref{alg:OSMD-sampler}. The sampling distribution is restricted to lie in the space $\mathcal{A}=\mathcal{P}_{M-1}\cap[\alpha/M,\infty)^M$, $\alpha \in (0,1]$, to prevent the server from assigning too small probabilities to devices. 
The learning rates $\{\eta_t\}_{t \geq 1}$ are positive\footnote{We use the term \emph {learning rate} when discussing an online algorithm that learns a sampling distribution, while the term \emph{stepsize} is used in the context of an optimization algorithm.}.
$\Phi(x)=\sum^M_{m=1} x_m \log x_m - \sum^M_{m=1}x_m$, $x=(x_1,\dots,x_M)^{\top} \in \mathcal{D}=\mathbb{R}^M_{+}$, with $0\log0$ defined as $0$, is the unnormalized negative entropy.
The Bregman divergence between any $x,y \in \mathcal{D}$ with respect to the function $\Phi$ is given as $D_{\Phi} \left(x \,\Vert\, y\right) = \Phi (x) - \Phi (y) - \langle \nabla \Phi (y), x - y \rangle$.
Line~\ref{alg:solve-mirror-descent} of Algorithm~\ref{alg:OSMD-sampler} provides an update to the sampling distribution using the mirror descent update. The available feedback is used to construct an estimate of the loss, while the Bregman divergence between the current and next sampling distribution is used as a regularizer, ensuring that the updated sampling distribution does not change too much. The update only uses the most recent information, while forgetting the history, which results in nonstationarity of the sequence of sampling distributions. 

An efficient algorithm to solve the mirror descent update in Line~\ref{alg:solve-mirror-descent} of Algorithm~\ref{alg:OSMD-sampler} is shown in Algorithm~\ref{alg:OSMD-sampler-solver}, justified by Proposition~\ref{prop:OSMD-solver} in Appendix~\ref{sec:proof-prop-1}.
The main cost comes from sorting the sequence $\{ \tilde{p}^{t+1}_m \}^M_{m=1}$, which can be done with the computational complexity of $O(M\log M)$. However, note that we only update a few entries of $\hat{p}^t$ to get $\tilde{p}^{t+1}$ and $\hat{p}^t$ is sorted. Therefore, most entries of $\tilde{p}^{t+1}$ are also sorted. Using this observation, we can usually achieve a much faster running time, for example, by using an adaptive sorting algorithm~\citep{estivill1992survey}.




\begin{algorithm}[t]
\caption{Solver of Step~\ref{alg:solve-mirror-descent} of Algorithm~\ref{alg:OSMD-sampler}}
\label{alg:OSMD-sampler-solver}
\begin{algorithmic}[1]
\STATE {\bfseries Input:} $\hat{p}^t$, $S^t$, $\{ a^t_m \}_{m \in S^t}$, and $\mathcal{A}=\mathcal{P}_{M-1}\cap[\alpha/M,\infty)^M$.
\STATE {\bfseries Output:} $\hat{p}^{t+1}$.
\STATE Let $\tilde{p}^{t+1}_m = p^t_m \exp \left\{ \mathcal{N}\left\{ m \in S^t \right\} \eta_t a^t_m / (K^2 (p^t_m)^3)  \right\}$ for $m \in [M]$.
\STATE Sort $\{ \tilde{p}^{t+1}_m \}^M_{m=1}$ in a non-decreasing order: $\tilde{p}^{t+1}_{\pi(1)}\leq\tilde{p}^{t+1}_{\pi(2)}\leq\dots\leq\tilde{p}^{t+1}_{\pi(M)}$.
\STATE Let $v_m=\tilde{p}^{t+1}_{\pi(m)} \left( 1 -  \frac{m-1}{M} \alpha \right)$ for $m \in [M]$.
\STATE Let $u_m=\frac{\alpha}{M} \sum^M_{j=m} \tilde{p}^{t+1}_{\pi(j)}$ for $m \in [M]$.
\STATE Find the smallest $m$ such that $v_m > u_m$, denoted as $m^t_{\star}$.
\STATE Let $\hat{p}^{t+1}_m =
\begin{cases}
\alpha / M & \text{if } \pi(m) < m^t_{\star} \\
\left((1 - ((m^t_{\star}-1) / M) \alpha ) \tilde{p}^{t+1}_{m}\right)/\left(\sum^M_{j=m^t_{\star}} \tilde{p}^{t+1}_{\pi(j)}\right) & \text{otherwise}.
\end{cases}
$
\end{algorithmic}
\end{algorithm}

\section{Application of OSMD Sampler on Federated Optimization Algorithms}
\label{sec:convg-analysis-batchSGD}

We illustrate how OSMD Sampler can be used to provably improve the convergence rates of federated optimization algorithms by reducing the heterogeneity. We choose two algorithms that are most commonly used in federated learning as our illustrative examples, namely the SGD mini-batch and FedAvg~\citep{mcmahan2017communication}. We use these two algorithms as motivational examples to show how adaptive sampling improves the convergence guarantees of optimization algorithms. However, the analysis here could be generalized to other optimization algorithms as well.

To simplify the notation, we denote $F_m(w)\coloneqq \phi \left( w  ;  \mathcal{D}_m \right)$ and let $\lambda_m=1/M$ for all $m \in [M]$ in problem \eqref{eq:SingleFL-obj}. 
We assume that the client objectives are differentiable and $L$-smooth functions. 
\begin{assump}
\label{assump:diff-smooth}
For all $m \in [M]$, $F_m(\cdot)$ is differentiable and $L$-smooth, that is,
\begin{equation*}
\left\Vert \nabla F_m(x) - \nabla F_m(y) \right\Vert_2 \leq L \Vert x - y \Vert_2, \quad \text{ for all } x,y \in \mathbb{R}^d.
\end{equation*}
\end{assump}
Note that we allow $F_m(\cdot)$ to be non-convex.
We also assume that the objective function $F(\cdot)$ is lower-bounded.
\begin{assump}
\label{assump:lower-bounded}
We assume that $\inf_w F(w) > -\infty$. We then denote $F^{\star} \coloneqq \inf_{w } F(w)$.
\end{assump}
In addition, we make the following assumption about the local stochastic gradient.
\begin{assump}
\label{assump:local-sg}
We assume that
\begin{equation*}
\mathbb{E}_{\xi \sim \mathcal{D}_m} \left[ \nabla \phi ( w ; \xi )  \right] = \nabla F_m (w) \quad\text{and} \quad \mathbb{E}_{\xi \sim \mathcal{D}_m} \left[ \left\Vert \nabla \phi ( w ; \xi ) - \nabla F_m (w) \right\Vert^2_2  \right] \leq \sigma^2
\end{equation*}
for all $w$ and $m \in [M]$; besides, $\Vert \nabla \phi ( w ; \xi ) \Vert_2 \leq G$ for all $w$ and $\xi$.
\end{assump}

{
While bounded gradient variance is often assumed in federated learning literature~\citep{patel2022towards}, bounded gradient norm is less common. We use this assumption to simplify regret analysis, making the loss in~\eqref{eq:var-reduc-loss} bounded. Removing it is feasible but complex, diverging from the main focus of this paper. From a practical viewpoint, many loss functions, like logistic regression loss, naturally satisfy this assumption. For others, one can project the gradient into a bounded norm subspace. If any minimizer $w^{\star}$ has a bounded gradient norm $\Vert \nabla \phi ( w^{\star} ; \xi ) \Vert_2$ and is within this subspace, the projection will not increase the distance to the minimizer. However, while the projection step can provide a slightly stronger theoretical guarantee, it brings few practical benefits and makes the algorithm harder to follow. Therefore, we adopt a stronger assumption to simplify the presentation.
}

We start our analysis by building the connection between the heterogeneity and client sampling.
We first introduce quantities that characterize the heterogeneity of the optimization problem. Specifically, heterogeneity characterizes how the objective functions of different clients differ from each other. 
In a federated learning problem, heterogeneity can be large and it
is important to understand its effect on the convergence of algorithms.
Let
\begin{equation}
\label{eq:def-hetero}
\zeta^2_{\text{unif}} \coloneqq \sup_{w } \frac{1}{M} \sum^M_{m=1} \Vert \nabla F_m(w) - \nabla F(w) \Vert^2_2 = \sup_{w } \left\{ \frac{1}{M} \sum^M_{m=1} \Vert \nabla F_m(w)  \Vert^2_2 - \Vert \nabla F(w) \Vert^2_2 \right\}.
\end{equation}
The quantity $\zeta_{\text{unif}}$ has been commonly used to quantify the first-order heterogeneity in the literature~\citep{karimireddy2020mime,karimireddy2020scaffold}.
To understand the relationship between heterogeneity and client sampling, let $\tilde{m}$ be a random index drawn from $[M]$, with $\mathbb{P}\{\tilde{m}=m\}=p_m$ for all $m\in [M]$. Then a natural unbiased estimator of $\nabla F(w)$ is $\nabla F_{\tilde{m}}(w) / (M \tilde{p}_m)$.
We define
\begin{equation*}
V(p,w) \coloneqq \mathbb{E}_{\tilde{m}} \left[ \left\Vert \frac{1}{M p_{\tilde{m}}} \nabla F_{\tilde{m}}(w) - \nabla F(w)  \right\Vert^2_2 \right],
\end{equation*}
to be the variance of the estimator at parameter $w$, where $\mathbb{E}_{\tilde{m}}[\cdot]$ denotes the expectation taken with respect to the random index $\tilde{m}$.
Note that we have
\begin{equation}
\label{eq:Vpw}
V(p,w) = \frac{1}{M^2} \sum^M_{m=1} \frac{1}{p_m} \left\Vert \nabla F_m(w) \right\Vert^2_2 - \left\Vert \nabla F(w) \right\Vert^2_2.
\end{equation}
Thus, it is clear that $\zeta^2_{\text{unif}} = \sup_w V(\punif,w)$. In other words, the common definition of heterogeneity can be viewed as the worst-case variance of uniform client sampling.

When we use adaptive sampling to do client sampling, there are two sources of flexibility that allow us to reduce the heterogeneity: (i) we can use non-uniform sampling that may depend on parameter $w$; (ii) we allow the sampling distribution to change over iterations. To reflect the consequential effect, we introduce a new concept termed \emph{dynamic heterogeneity}. Let $\text{TV}(q^{1:T})=\sum^{T-1}_{t=1}\Vert q^{t+1} - q^t \Vert_1$ be the total variation of any sequence of sampling distributions $q^{1:T} \in \mathcal{P}^T_{M-1}$. The dynamic heterogeneity is defined as
\begin{align*}
& \zeta^2_T (\alpha, \beta) =\frac{1}{T}  \sup_{w^1} \min_{p^1 \in \mathcal{A}} \cdots \sup_{w^T} \min_{p^T \in \mathcal{A}} \,  \sum^T_{t=1} V(p^t,w^t) \quad \text{subject to } \text{TV} \left( p^{1:T} \right) \leq \beta, 
\end{align*}
where $V(p,w)$ is defined in~\eqref{eq:Vpw} and $\beta \geq 0$ is the total variation budget. The dynamic heterogeneity $\zeta^2_T (\alpha, \beta)$ can be regarded as the worst-case variance of dynamic samplings in $\mathcal{A}^T$ with the total variation budget $\beta$. To see how $\zeta^2_T (\alpha, \beta)$ improves over $\zeta^2_{\text{unif}}$, let
\begin{equation}
\label{eq:fix-hetero}
\zeta^2_{\text{fix}}(\alpha) = \min_{p \in \mathcal{A}} \sup_{w} V(p, w).
\end{equation}
The quantity $\zeta^2_{\text{fix}}(\alpha)$ can be regarded as the minimum heterogeneity by using the best fixed sampling distribution in $\mathcal{A}$ that does not depend on parameter $w$. Let $p_f$ be the solution of $p$ to the min-max problem~\eqref{eq:fix-hetero}, that is, $\sup_w V(p_f, w)=\zeta^2_{\text{fix}}(\alpha)$. Since $\punif \in \mathcal{A}$ for all $\alpha \geq 0$, it is easy to see that
\begin{equation*}
\zeta^2_{\text{fix}}(\alpha) = \min_{p \in \mathcal{A}} \sup_{w} V(p, w) \leq \sup_{w} V(\punif, w) = \zeta^2_{\text{unif}}.
\end{equation*}
Note that $\zeta^2_T (\alpha, \beta)$ is a non-increasing function of $\beta$ with any given $\alpha$. Thus, we have
\begin{equation}
\label{eq:dyn-hetero-smaller-hetero-fix}
\zeta^2_T (\alpha, \beta) \leq \zeta^2_{\text{fix}}(\alpha) \leq \zeta^2_{\text{unif}} \quad \forall\, 0 \leq \alpha \leq 1, \, \beta \geq 0.
\end{equation}
See the proof in Appendix~\ref{sec:proof-dyn-hetero-smaller-hetero-fix}. Note that the above inequality also implies that dynamic sampling distribution may potentially improve over a fixed sampling distribution. As we will see shortly, when $\zeta^2_{\text{unif}} > \zeta^2_{\text{fix}}(\alpha)$, OSMD Sampler can always improve over uniform sampling asymptotically.

When $\beta \geq 2(T-1)$ and $\alpha$ is small enough such that 
\begin{equation*}
p^{\star}_m(w) \coloneqq \frac{ \Vert \nabla F_m (w) \Vert_2 }{ \sum^M_{m^{\prime}=1} \Vert \nabla F_{m^{\prime}} (w) \Vert_2 }  \geq \frac{\alpha}{M}  \quad \text{for all } w \text{ and } m \in [M],
\end{equation*}
we have
\begin{equation*}
\zeta^2_T (\alpha, \beta) = \sup_w \min_{p \in \mathcal{A}} V(w,p) = \sup_{w } \left\{ \left( \frac{1}{M} \sum^M_{m=1}  \Vert \nabla F_m(w)  \Vert_2 \right)^2 - \Vert \nabla F(w) \Vert^2_2 \right\} \overset{\Delta}{=} \zeta^2_{\min},
\end{equation*}
which is the smallest heterogeneity possibly achievable.



\subsection{Convergence Analysis of Mini-batch SGD with OSMD Sampler}
\label{sec:convg-analysis-osmd-batch-sgd}

We introduce the convergence analysis of mini-batch SGD with OSMD Sampler. The detailed algorithm is given in Algorithm~\ref{alg:min-batch-osmd}. Compared to the classical mini-batch SGD, the key ingredients of Algorithm~\ref{alg:min-batch-osmd} are Line~\ref{step:samp-update}, where the server updates the sampling distribution by OSMD Sampler, and Line~\ref{step:sample-nonuniform}, where the server samples clients from a non-uniform sampling distribution. In~\eqref{eq:global-gradient}, we use a weighted average to compute the global gradient.

\begin{algorithm}[t]
\caption{Mini-batch SGD with OSMD Sampler}
\label{alg:min-batch-osmd}
\begin{algorithmic}[1]
\STATE {\bfseries Input:} Number of communication rounds $T$, number of clients chosen in each round $K$, local batch size $B$, initial model parameter $w^1$, stepsizes $\{ \mu^t \}^{T}_{t=1}$, learning rate $\eta$ and parameter $\alpha \in (0,1]$.
\STATE {\bfseries Output:} The final model parameter $w^R$.
\STATE {\bfseries Initialize:} $\hat{p}^1=\punif$.
\FOR{$t=1,\ldots,T$}{
\STATE Sample $S^t$ with replacement from $[M]$ with probability $\hat{p}^t$, such that $\vert S^t \vert = K$. \label{step:sample-nonuniform} \;
\FOR{$m \in S^t$}{
\STATE Download the current model parameter $w^t$.
\STATE Locally sample a mini-batch $\mathcal{B}^t_m=\{ \xi^{t,1}_{m},\ldots,\xi^{t,B}_{m} \}$ i.i.d. uniformly random from $[n_m]$, where $|\mathcal{B}^t_m|=B$. \;
\STATE Locally compute and upload $g^t_m=(1/B)\sum^B_{b=1}\nabla \phi ( w^t ; \xi^{t,b}_{m} )$ to the server.
}
\ENDFOR
\STATE Server computes $a^t_m=\lambda^2_m \Vert g^t_m \Vert^2 $ for $m \in S^t$ and
\begin{equation}
\label{eq:global-gradient}
g^t = \frac{1}{K} \sum_{m \in S^t} \frac{\lambda_m}{\hat{p}^t_m} g^t_m.
\end{equation}\;
\STATE Server makes update of the model parameter $w^{t+1} \leftarrow w^t - \mu^t g^t$. \;
\STATE Server obtains updated sampling distribution $\hat{p}^{t+1}$ by Algorithm~\ref{alg:OSMD-sampler}. \label{step:samp-update}
}
\ENDFOR
\end{algorithmic}
\end{algorithm}

Recall that $B$ is the local batch size in Algorithm~\ref{alg:min-batch-osmd} and $K=\vert S^t \vert$. Let $D^F \coloneqq F(w^1)-F^{\star}$.
We then have the following convergence guarantee for Algorithm~\ref{alg:min-batch-osmd}.
\begin{theorem}
\label{thm:convg-thm-mini-batch}
Assume Assumption~\ref{assump:diff-smooth}---\ref{assump:local-sg} holds.
Let $\{w^1,\ldots,w^T \}$ be the sequence of iterates generated by Algorithm~\ref{alg:min-batch-osmd} and let $w^R$ denote an element of that sequence chosen uniformly at random. Let 
\begin{equation}
\label{eq:lr-mini-batch-sgd}
\eta = \frac{ K \alpha^3 }{ M G^2 } \sqrt{ \frac{ 2\log M + 4 \beta \log (M / \alpha) }{ T } },
\end{equation}
and $\mu_t \equiv \mu$ for all $t \in [T]$, where
\begin{equation*}
\mu=\min \left\{\frac{1}{L}, \frac{1}{\sigma}\sqrt{ \frac{ D^F K B \alpha  }{ L T }}, \frac{1}{\zeta_T(\alpha,\beta)} \sqrt{  \frac{D^F K}{L T} }, \frac{ \sqrt{D^F K} \alpha^\frac{3}{2} }{\sqrt{L M} T^{\frac{1}{4}}  G \left( \frac{1}{2}\log M + \beta \log \left( M / \alpha \right) \right)^{\frac{1}{4}} } \right\},
\end{equation*}
we then have
\begin{equation}
\label{eq:convg-rate-minibatchSGD}
\begin{aligned}
& \quad \mathbb{E} \left[ \left\Vert \nabla F \left( w^R \right) \right\Vert^2 \right] \\
& \lesssim \frac{D^F L}{T} + \frac{ \sigma \sqrt{D^F L} }{ \sqrt{TKB \alpha } } + \frac{\zeta_T(\alpha,\beta) \sqrt{D^F L }}{\sqrt{T K}} + \frac{ \sqrt{ D^F L} M^{\frac{1}{2}} G }{ T^{\frac{3}{4}} K^{\frac{1}{2}} \alpha^{\frac{3}{2}} } \left( \frac{1}{2}\log M + \beta \log \left( M / \alpha \right)\right)^{\frac{1}{4}}.
\end{aligned}
\end{equation}
\end{theorem}
{
\begin{proof}
The key proof technique is the construction of a ghost subset $\tilde{S}^t$ in each round that is sampled from a carefully designed comparator sampling distribution. In this way, the regret can be compared with the difference between the convergence rate under OSMD Sampler and the convergence rate under the comparator sampling distribution. Note that $\tilde{S}^t$ is only constructed for theoretical analysis, and does not need to be actually computed in practice. The rest of the proof then follows the regret analysis as in Theorem~\ref{thm:upp-bd-general}. See detailed proof in Appendix~\ref{sec:proof-convg-thm-mini-batch}.
\end{proof}
}

To see how the convergence rate in Theorem~\ref{thm:convg-thm-mini-batch} is better than the rate of uniform sampling, recall that the convergence rate of mini-batch SGD under uniform sampling~\citep{ghadimi2013stochastic} is
\begin{equation*}
R^{\text{MB}}_{\text{unif}} \coloneqq \frac{D^F L}{T} + \frac{ \sigma \sqrt{D^F L} }{ \sqrt{TKB } } + \frac{\zeta_{\text{unif}} \sqrt{D^F L }}{\sqrt{T K}}.
\end{equation*}
Denote the right hand side of~\eqref{eq:convg-rate-minibatchSGD} as $R^{\text{MB}}_{\text{osmd}}$, then to have $R^{\text{MB}}_{\text{osmd}} \lesssim R^{\text{MB}}_{\text{unif}}$, we only need that
\begin{equation*}
\frac{ \sqrt{ D^F L} M^{\frac{1}{2}} G }{ T^{\frac{3}{4}} K^{\frac{1}{2}} \alpha^{\frac{3}{2}} } \left( \frac{1}{2}\log M + \beta \log \left( M / \alpha \right)\right)^{\frac{1}{4}} \lesssim \frac{ \left( \zeta_{\text{unif}} - \zeta_T(\alpha,\beta) \right) \sqrt{D^F L }}{\sqrt{T K}},
\end{equation*}
which is equivalent to a requirement that
\begin{equation*}
\frac{  M^{\frac{1}{2}} G }{ \alpha^{\frac{3}{2}} } \left( \frac{ \frac{1}{2}\log M + \beta \log \left( M / \alpha \right) }{ T } \right)^{\frac{1}{4}} \lesssim \zeta_{\text{unif}} - \zeta_T(\alpha,\beta).
\end{equation*}
By treating all the other quantities except for $T$, $\zeta_{\text{unif}}$, and $\zeta_T (\alpha,\beta)$ as constants and setting $\beta = o(T)$, we have the left hand side of the above inequality as $o(1)$. On the other hand, note that $\zeta_T(\alpha,\beta) \leq \zeta_{\text{fix}} \leq \zeta_{\text{unif}}$ for any $\beta \geq 0$ and $0 < \alpha \leq 1$, where $\zeta_{\text{fix}}$ is defined in~\eqref{eq:fix-hetero}, so whenever $\zeta_{\text{unif}}>\zeta_{\text{fix}}$, we have $\zeta_{\text{unif}} - \zeta_T(\alpha,\beta) \geq \zeta_{\text{unif}} - \zeta_{\text{fix}} = \Omega(1)$.
In conclusion, when $\zeta_{\text{unif}}>\zeta_{\text{fix}}$ and by setting $\beta=o(T)$, the OSMD Sampler can achieve a better convergence rate than uniform sampling for mini-batch SGD.

\subsection{Convergence Analysis of FedAvg with OSMD Sampler}
\label{sec:convg-fedavg}

\begin{algorithm}[t]
\caption{FedAvg with OSMD Sampler}
\label{alg:fed-avg-osmd}
\begin{algorithmic}[1]
\STATE {\bfseries Input:} Communication rounds $T$, clients per round $K$, local steps $B$, initial model parameter $w^1$, global stepsizes $\{ \mu^t \}^{T}_{t=1}$, local stepsizes $\{ \mu^t_l \}^{T}_{t=1}$, learning rate $\eta$, and parameter $\alpha \in (0,1]$.
\STATE {\bfseries Output:} The final model parameter $w^R$.
\STATE {\bfseries Initialize:} $\hat{p}^1=\punif$.
\FOR{$t=1,\ldots,T$}{
\STATE Sample $S^t$ with replacement from $[M]$ with probability $\hat{p}^t$, such that $\vert S^t \vert = K$. \label{step:sample-nonuniform-fedavg} \;
\FOR{$m \in S^t$}{
\STATE Download the current model parameter $w^t$ and let $w^{t,0}_m=w^t$.
\FOR{$b=0,1,\ldots,B-1$}{
\STATE Sample $\xi^{t,b}_{m}$ from $[n_m]$ uniformly random.
\STATE Compute $w^{t,b+1}_m = w^{t,b}_m - \mu^t_l \nabla \phi ( w^{t,b}_m ; \xi^{t,b}_{m} )$.
}
\ENDFOR
\STATE Locally compute $g^t_m= w^t - w^{t,B}_m$ and upload it to the server.
}
\ENDFOR
\STATE Server computes 
\begin{equation*}
a^t_m=\frac{\lambda^2_m \Vert g^t_m \Vert^2_2}{(\mu^t_l)^2 B} = \frac{\lambda^2_m}{B} \left\Vert \sum^{B-1}_{b=0} \nabla \phi ( w^{t,b}_m ; \xi^{t,b}_{m} ) \right\Vert^2_2
\end{equation*}
for $m \in S^t$ and let
\begin{equation}
\label{eq:global-model-update}
w^{t+1} = w^t -  \frac{\mu^t}{K} \sum_{m \in S^t} \frac{\lambda_m}{\hat{p}^t_m} g^t_m.
\end{equation}\;
\STATE Server obtains updated sampling distribution $\hat{p}^{t+1}$ by Algorithm~\ref{alg:OSMD-sampler}. \label{step:samp-update-fedavg}
}
\ENDFOR
\end{algorithmic}
\end{algorithm}

We introduce the convergence guarantee for FedAvg with OSMD Sampler. The detailed algorithm is given in Algorithm~\ref{alg:fed-avg-osmd}.
Compared to FedAvg in~\cite{mcmahan2017communication}, the differences between Algorithm~\ref{alg:fed-avg-osmd} are Line~\ref{step:samp-update-fedavg}, where the server updates the sampling distribution by OSMD Sampler, and Line~\ref{step:sample-nonuniform-fedavg}, where the server samples clients from a non-uniform sampling distribution. In addition, in~\eqref{eq:global-model-update}, we use a weighted average to update the global model parameter.
Besides, note that when defining $a^t_m$, we rescale $\Vert g^t_m \Vert^2_2$ by $(\mu^t_l)^2 B$, this is to ensure that $\mathbb{E}[\Vert g^t_m \Vert^2_2/(\mu^t_l)^2 B]=\Theta(1)$ as $\mu^t_l \to 0$ and $B \to \infty$.

We have the following result about FedAvg with OSMD Sampler (Algorithm~\ref{alg:fed-avg-osmd}).
\begin{theorem}
\label{thm:convg-thm-fedavg}
Recall that $B$ is the local batch size in Algorithm~\ref{alg:min-batch-osmd} and $K=\vert S^t \vert$. Let $D^F \coloneqq F(w^1)-F^{\star}$.
Assume Assumption~\ref{assump:diff-smooth}---\ref{assump:local-sg} holds.
Let $\{w^1,\ldots,w^T \}$ be the sequence of iterates generated by Algorithm~\ref{alg:fed-avg-osmd} and let $w^R$ denote an element of that sequence chosen uniformly at random.
Let 
\begin{equation}
\label{eq:lr-fedavg}
\eta = \frac{ K \alpha^3 }{ M B G^2 } \sqrt{ \frac{ 2 \log M + 4 \beta \log (M/\alpha) }{ T } },
\end{equation}
$\mu^t = \mu \geq 1$ and $\mu^t_l = \mu_l$ for all $t \in [T]$, where
\begin{align*}
\mu_l = \min \left\{ \frac{1}{4 \mu B L} \sqrt{\frac{1}{2+1/\alpha}},  \frac{(D^F)^{\frac{1}{3}}}{  \left( 4 + \frac{2}{\alpha}  \right)^{\frac{1}{3}} \mu B L^{\frac{2}{3}} \left( \zeta^2_{\text{unif}}  + \frac{\sigma^2}{2 B} \right)^{\frac{1}{3}} T^{\frac{1}{3}} }, \right.\\ 
\left. \frac{ \sqrt{ 2 D^F }  }{ \mu B \sqrt{L} \sqrt{ \frac{2 \zeta^2_T(\alpha,\beta) }{K} +  \frac{\sigma^2}{K B \alpha } + \sqrt{ \frac{ \frac{1}{2} \log M + \beta \log (M/\alpha) }{T} } } \sqrt{T} } \right\},
\end{align*}
we then have
\begin{equation}
\label{eq:convg-rate-fedavg-osmd}
\begin{aligned}
\mathbb{E}\left[ \left\Vert \nabla F(w^R) \right\Vert^2 \right] \lesssim \frac{D^F L \sqrt{2 + \frac{1}{ \alpha}}}{T} + \frac{ \left( 4 + \frac{2}{\alpha}  \right)^{\frac{1}{3}} (D^F L)^{\frac{2}{3}} \zeta^{\frac{2}{3}}_{\text{unif}}  }{ T^{\frac{2}{3}} } + \frac{ \left( 4 + \frac{2}{\alpha}  \right)^{\frac{1}{3}} (D^F L)^{\frac{2}{3}}  \sigma^{\frac{2}{3}}  }{ B ^{\frac{1}{3}} T^{\frac{2}{3}} } \\
+ \frac{ \sqrt{ D^F L }\zeta_T (\alpha, \beta) }{ \sqrt{T K} } + \frac{ \sqrt{ D^F L } \sigma  }{ \sqrt{ T K B \alpha } } + \frac{ \sqrt{D^F L } }{ \sqrt{T} } \left( \frac{ \frac{1}{2} \log M + \beta \log (M/\alpha) }{T} \right)^{\frac{1}{4}}.
\end{aligned}
\end{equation}
\end{theorem}
\begin{proof}
See proof in Appendix~\ref{sec:proof-thm-fedavg}. Similar to the proof of Theorem~\ref{thm:convg-thm-mini-batch}, the key technique is to construct a novel ghost subset sampled from the comparator sampling distribution. The rest of the proof then follows the regret analysis as in Theorem~\ref{thm:upp-bd-general}.
\end{proof}

To see how OSMD Sampler improves over uniform sampling, note that the convergence rate of FedAvg under uniform sampling ~\citep{karimireddy2020scaffold} is
\begin{equation*}
R^{\text{Avg}}_{\text{unif}} \coloneqq \frac{D^F L }{T} + \frac{  (D^F L)^{\frac{2}{3}} \zeta^{\frac{2}{3}}_{\text{unif}}  }{ T^{\frac{2}{3}} } + \frac{  (D^F L)^{\frac{2}{3}}  \sigma^{\frac{2}{3}}  }{ B^{\frac{1}{3}} T^{\frac{2}{3}} }+ \frac{ \sqrt{ D^F L }\zeta_{\text{unif}} }{ \sqrt{T K} } + \frac{ \sqrt{ D^F L } \sigma  }{ \sqrt{ T K B } }.
\end{equation*}
Denote the right hand side of~\eqref{eq:convg-rate-fedavg-osmd} as $R^{\text{Avg}}_{\text{osmd}}$.
By treating all the other quantities except for $T$, $\zeta_{\text{unif}}$ and $\zeta_T (\alpha,\beta)$ as constants, to have $R^{\text{Avg}}_{\text{osmd}} \lesssim R^{\text{Avg}}_{\text{unif}}$, we only need that
\begin{equation*}
\sqrt{K} \left( \frac{ \frac{1}{2} \log M + \beta \log (M/\alpha) }{T} \right)^{\frac{1}{4}} \lesssim  \zeta_{\text{unif}} - \zeta_T (\alpha, \beta) .
\end{equation*}
Similar to the argument in Section~\ref{sec:convg-analysis-osmd-batch-sgd}, by setting $\beta = o(T)$, we have the left hand side of the above inequality as $o(1)$.
On the other hand, when $\zeta_{\text{unif}}>\zeta_{\text{fix}}$, we have $\zeta_{\text{unif}} - \zeta_T(\alpha,\beta) \geq \zeta_{\text{unif}} - \zeta_{\text{fix}} = \Omega(1)$.
Thus, when $\zeta_{\text{unif}}>\zeta_{\text{fix}}$ and by setting $\beta=o(T)$, we have shown that OSMD Sampler can achieve a better convergence rate than uniform sampling for FedAvg.



\section{Regret Analysis of OSMD Sampler}
\label{sec:thm-osmd-sampler}

In this section, we provide the regret analysis that serves as a key component of the optimization analysis in Section~\ref{sec:convg-analysis-batchSGD}. 
We first describe the dynamic regret used to measure the performance of an online algorithm that generates a sequence of sampling distributions $\{\hat{p}\}_{t\geq1}$ in a non-stationary environment. Given any comparator sequence $q^{1:T} \in \mathcal{P}_{M-1}^T$, the dynamic regret is defined as 
\begin{equation}\label{eq:dynamic-regret}
\text{D-Regret}_T (q^{1:T}) = \bar{L} \left( \hat{p}^{1:T} \right) - \bar{L} \left( q^{1:T} \right).
\end{equation}
In contrast, the static regret measures the performance of an algorithm relative to the best fixed sampling distribution, that is, it restricts $q^1=\dots=q^T$ \citep{namkoong2017adaptive, salehi2017stochastic, borsos2018online, borsos2019online}. 
When using a fixed comparator $q^1=\dots=q^T=q$, we write the regret as $\text{D-Regret}_T(q)$.

Recall that the total variation of a comparator sequence $q^{1:T}$ is $\text{TV} \left( q^{1:T} \right) = \sum^{T-1}_{t=1} \Vert q^{t+1} - q^t \Vert_1$.
The total variation measures how variable a sequence is. The larger the total variation $\text{TV} ( q^{1:T} )$, the more variable $q^{1:T}$ is, and such a comparator sequence is harder to match. 

We also need the following quantities that quantify how far $q^t$ is from $\mathcal{A}$. Given $q^t \in \mathcal{P}_{M-1}$ and $\alpha \in (0,1]$, let
\begin{equation}
\begin{gathered}
\label{eq:def-psi-alpha}
\psi(q^t, \alpha) \coloneqq \sum^M_{m=1} \left( \frac{\alpha}{M} - q^t_m \right)  \mathds{1} \left\{ q^t_m < \frac{\alpha}{M} \right\},
\qquad
\omega(q^t, \alpha)  \coloneqq \frac{ \sum^M_{m=1} \left( \frac{\alpha}{M} - q^t_m \right)  \mathds{1} \left\{ q^t_m < \frac{\alpha}{M} \right\} }{ \sum^M_{m=1} \left( q^t_m - \frac{\alpha}{M} \right) \mathds{1} \left\{ q^t_m \geq \frac{\alpha}{M} \right\} }, \\
\phi(q^t, \alpha) \coloneqq \frac{ \omega(q^t, \alpha) }{ 1 - \omega(q^t, \alpha) \left( 1 - \frac{\alpha}{M} \right)  }.
\end{gathered}
\end{equation}
We will use these quantities to characterize the projection error in the following theorem, which is the main result of this section. 

\begin{theorem}
\label{thm:upp-bd-general}
Let $\eta_t \equiv \eta$ for all $t$ and $\hat{p}^{1:T}$ be a sequence generated by Algorithm~\ref{alg:OSMD-sampler}. For any comparator sequence $q^{1:T}$, where $q^t$ is allowed to be random, we have
\begin{multline*}
\text{D-Regret}_T (q^{1:T}) 
\leq \underbrace{\frac{ \log M  }{ \eta } + \frac{ 2 \log (M / \alpha) }{ \eta } \mathbb{E} \left[\text{TV} \left( q^{1:T} \right) \right] + \frac{\eta M^6}{ 2 K^2 \alpha^6 } \sum^T_{t=1} \mathbb{E} \left[  \left( a^t_{\max}  \right)^2 \right] }_{\text{Intrinsic Regret}} +\\ \underbrace{ \frac{8 \log (M / \alpha)}{\eta} \sum^T_{t=1} \mathbb{E}\left[ \psi(q^t, \alpha) \right] + \frac{1}{K} \sum^T_{t=1} \mathbb{E}\left[\phi(q^t, \alpha) l_t (q^t) \right] }_{\text{Projection Error}},
\end{multline*}
where $a^t_{\max} \coloneqq \max_{1\leq m \leq M} a^t_m = \max_{1\leq m \leq M} \lambda^2_m \Vert g^t_m \Vert^2_2$ for all $t \in [T]$.
\end{theorem}
\begin{proof}
The major challenge of the proof is to construct a projection of \emph{any} comparator sequence $q^{1:T}$ onto $\mathcal{A}^T$ and bound the projection error. To the best of our knowledge, this bound on the projection error of a dynamic sequence is novel. 
Since the comparator sequence can be arbitrary, thus analyzing the projection error is nontrivial.
Another challenge is to deal with the dynamic comparator, which requires us to connect the cumulative regret with the total variation of the comparator sequence.
See Appendix~\ref{sec:proof-thm-bd-1} for more details.
\end{proof}

From Theorem~\ref{alg:OSMD-sampler}, we see that the bound on the dynamic regret consists of two parts. The first part is the intrinsic regret, quantifying the difficulty of tracking a comparator sequence in $\mathcal{A}^T$; the second part is the projection error, arising from projecting the comparator sequence onto $\mathcal{A}^T$. As shown in Appendix~\ref{sec:proof-thm-bd-1}, we have $0 \leq \omega(q^t, \alpha) \leq 1$ for all $\alpha \in [0,1]$, which implies that $\phi(q^t, \alpha) \leq M / \alpha$. 
Besides, $\psi(q^t, \alpha) \leq \sum^M_{m=1} (\alpha/M) \mathds{1} \left\{ q^t_m < (\alpha/M) \right\} \leq \alpha$, and the projection error can be upper bounded by $(8 T \alpha \log (M/\alpha)  )/\eta + (M/\alpha) \sum^T_{t=1} \mathbb{E} [ l_t (q^t) ]$. More importantly, when $q^t_m \in \mathcal{A}$, we have $\psi(q^t, \alpha)=\omega(q^t, \alpha)=\phi(q^t, \alpha)=0$. Thus, when the comparator sequence belongs to $\mathcal{A}^T$, the projection error vanishes and we only have the intrinsic regret. As $\alpha$ decreases from one to zero, the intrinsic regret gets larger, while we are allowing a larger class of comparator sequences; on the other hand, the projection error decreases to zero, since the gap between $\mathcal{A}$ and $\mathcal{P}_{M-1}$ vanishes with $\alpha$. An optimal choice of $\alpha$ balances the two sources of regret.

\section{Adaptive-OSMD Sampler}
\label{sec:adp-osmd-sampler}

In this section, we discuss an extension of OSMD Sampler that can automatically choose the learning rate $\eta$ and is agnostic to the optimization method.
There are two main tuning parameters in OSMD Sampler, namely $\alpha$ and $\eta$. As we show empirically in Section~\ref{sec:robu-alpha}, the performance of the algorithm is relatively robust to the choice of $\alpha$.
However, the choice of $\eta$ may have a large effect on the performance of OSMD Sampler.
One way to choose $\eta$ is by minimizing the regret in Theorem~\ref{thm:upp-bd-general}, which is stated in the following corollary.
\begin{corollary}
\label{corollary:regret-bd-strategy-2}
Let $\eta_t \equiv \eta$ for all $t$ and $\hat{p}^{1:T}$ be a sequence generated by Algorithm~\ref{alg:OSMD-sampler}. 
Assume that there exists $A_{\max}>0$ such that $a^t_{\max} \leq A_{\max}$ for all $t$, where $a^t_{\max} = \max_{1\leq m \leq M} a^t_m = \max_{1\leq m \leq M} \lambda^2_m \Vert g^t_m \Vert^2_2$.
For any comparator sequence $q^{1:T}$, where $q^t$ is allowed to be random, such that $q^t \in \mathcal{A}$ for all $t \in [T]$ and $\mathbb{E} [\text{TV}(q^{1:T})] \leq \beta$, let
\begin{equation}
\label{eq:eta-choice-meta-pre-2}
\eta = \frac{ K \alpha^3 }{ M^3 A_{\max} } \sqrt{ \frac{ 2\log M + 4 \beta \log (M / \alpha) }{ T  } },
\end{equation}
then
\begin{equation}
\label{eq:regret-meta-pre}
\text{D-Regret}_T (q^{1:T}) \leq 
\frac{ M^3 A_{\max} }{ K \alpha^3 } \sqrt{ T \left[ \frac{1}{2}\log M + \beta \log \left( M / \alpha \right)   \right] }.
\end{equation}
\end{corollary}

The proof of Corollary~\ref{corollary:regret-bd-strategy-2} follows directly from Theorem~\ref{thm:upp-bd-general}.
Note that under Assumption~\ref{assump:diff-smooth}---\ref{assump:local-sg} and when $\lambda_m=\frac{1}{M}$ for all $m \in [M]$, we have $A_{\max}=\frac{G^2}{M^2}$ for Mini-batch SGD (Algorithm~\ref{alg:min-batch-osmd}) and $A_{\max}=\frac{BG^2}{M^2}$ for FedAvg (Algorithm~\ref{sec:convg-fedavg})\footnote{See Appendix~\ref{sec:proof-convg-thm-mini-batch} and Appendix~\ref{sec:proof-thm-fedavg} for proof.}, thus~\eqref{eq:eta-choice-meta-pre-2} recovers~\eqref{eq:lr-mini-batch-sgd} and~\eqref{eq:lr-fedavg}.

In practice, since the gradient norm is usually decreasing, we can estimate $A_{\max}$ by adding a pre-training phase where we broadcast the initial model parameter $w^0$ to all devices before the start of the training, and collect the returned $\Vert g^0_m \Vert^2_2$ from all responsive devices, which we denote as $S^0$. Then we can estimate $A_{\max}$ by $\hat{A}_{\max} = \max_{m \in S^0} \lambda_m \Vert g^0_m \Vert^2_2$.

On the other hand, the optimal choice of $\beta$ in~\eqref{eq:eta-choice-meta-pre-2} depends on specific problems, and is hard to estimate before training starts. Thus, it is preferable to have a tuning strategy that is adaptive to any $\beta>0$, which we describe in the following.

The main idea is to run a set of expert algorithms, each with a different learning rate for Algorithm~\ref{alg:OSMD-sampler}. We then use a prediction-with-expert-advice algorithm to track the best performing expert algorithm.\footnote{We refer the reader to~\citet[Chapter 2]{cesa2006prediction} for an overview of prediction-with-expert-advice algorithms.} More specifically, we define the set of expert learning rates as
\begin{equation}\label{eq:expert-rates-set}
\mathcal{E} \coloneqq \left\{ \left.  2^{e-1} \cdot \frac{K \alpha^3 }{ M^3  A_{\max}  } \sqrt{ \frac{2 \log M}{ T  } } \, \right\vert \, e=1,2,\dots,E \right\},
\end{equation}
where
\begin{equation}
\label{eq:expert-rates-set-length}
E = \left\lceil \frac{1}{2} \log_2 \left( 1 + \frac{4 \log (M/\alpha)}{\log M}(T-1) \right)  \right\rceil + 1.
\end{equation}
Then for each $\eta_e \in \mathcal{E}$, Adaptive-OSMD Sampler algorithm runs an expert algorithm to generate a sequence of sampling distributions $\hat{p}^{1:T}_e$. Meanwhile, it also runs a meta-algorithm that uses exponentially-weighted-average strategy to aggregate $\{\hat{p}^{1:T}_e\}^E_{e=1}$ into a single output $\hat{p}^{1:T}$, which achieves performance close to the best expert.

Algorithm~\ref{alg:adap-OSMD-sampler-meta} details Adaptive-OSMD Sampler. Note that since we can compute $\hat{l}_t (\hat{p}^t_e ; \hat{p}^t)$ and $\nabla \hat{l}_t (\hat{p}^t_e ; \hat{p}^t)$ directly, there is no need to use a surrogate loss as in~\citet{van2016metagrad} and \citet{zhang2018adaptive}.

From the computational perspective, the major cost comes from solving step~\ref{alg:expert-solve-mirror-descent} of Algorithm~\ref{alg:adap-OSMD-sampler-meta}, which needs to be run for a total number of $T \vert \mathcal{E} \vert= O (T \log_2 T )$ times. Compared with Algorithm~\ref{alg:OSMD-sampler}, the computational complexity only increases by $\log_2 T$ times.

\begin{algorithm}[t]
\caption{Adaptive-OSMD Sampler}
\label{alg:adap-OSMD-sampler-meta}
\begin{algorithmic}[1]
\STATE {\bfseries Input:} Meta learning rate $\gamma$; the set of expert learning rates $\mathcal{E} = \{\eta_1 \leq \eta_2 \leq \dots \leq \eta_E \}$ with $E=\vert \mathcal{E} \vert$; 
parameter $\alpha \in (0,1]$, $\mathcal{A}=\mathcal{P}_{M-1}\cap[\alpha/M,\infty)^M$; 
number of iterations $T$; initial distribution $p^{\text{init}}$.
\STATE {\bfseries Output:} $\hat{p}^{1:T}$.
\STATE Set $\theta^1_e = (1 + 1/E) / (e(e+1))$ and $\hat{p}^1_e=p^{\text{init}}$, $\forall e\in[E]$.
\FOR{$t=1,2,\dots,T-1$}{
\STATE Compute $\hat{p}^t = \sum^E_{e=1} \theta^t_e \hat{p}^t_e $.
\STATE Sample $S^t$ by $\hat{p}^t$.
\FOR{$e=1,2,\dots,E$}{
\STATE Compute $\hat{l}_t (\hat{p}^t_e ; \hat{p}^t)$ via \eqref{eq:est-loss} and $\nabla \hat{l}_t (\hat{p}^t_e ; \hat{p}^t)$ via \eqref{eq:est-grad}.
\STATE Solve $\hat{p}^{t+1}_e = \arg\min_{ p \in \mathcal{A} } \, \eta_e \langle p,\nabla \hat{l}_t (\hat{p}^t_e ; \hat{p}^t)  \rangle + D_{\Phi} \left( p \,\Vert\, \hat{p}^t_e  \right) $ via Algorithm~\ref{alg:OSMD-sampler-solver}. \label{alg:expert-solve-mirror-descent}
}
\ENDFOR
\STATE Update the weight of each expert:
\begin{equation*}
\theta^{t+1}_e = \frac{ \theta^t_e \exp \left\{ - \gamma \hat{l}_t (\hat{p}^t_e ; \hat{p}^t) \right\} }{ \sum^E_{e=1} \theta^t_e \exp \left\{ - \gamma \hat{l}_t (\hat{p}^t_e ; \hat{p}^t) \right\} }, \quad \forall e \in [E]. \label{step:weight-update}
\end{equation*}
}
\ENDFOR
\end{algorithmic}
\end{algorithm}

We have the following regret guarantee on Algorithm~\ref{alg:adap-OSMD-sampler-meta}.
\begin{theorem}
\label{thm:thm-dap-OSMD-sampler}
Assume that there exists $A_{\max}>0$ such that $a^t_{\max} \leq A_{\max}$ for all $t$, where $a^t_{\max} = \max_{1\leq m \leq M} a^t_m = \max_{1\leq m \leq M} \lambda^2_m \Vert g^t_m \Vert^2_2$.
Let $\hat{p}^{1:T}$ be the output of Algorithm~\ref{alg:adap-OSMD-sampler-meta} with $\gamma=\frac{\alpha}{M}\sqrt{\frac{8 K}{T A_{\max}} }$, $p^{\text{init}}=\punif$ and $\mathcal{E}$ as in \eqref{eq:expert-rates-set}.
Then for any comparator sequence $q^{1:T}$, where $q^t$ is allowed to be random, such that $q^t \in \mathcal{A}$ for all $t \in [T]$ and $\mathbb{E} [\text{TV}(q^{1:T})] \leq \beta$, we have
\begin{equation*}
\text{D-Regret}_T (q^{1:T}) \leq 
\frac{ 3 M^3 A_{\max} }{ K \alpha^3 } \sqrt{ T \left[ \frac{1}{2}\log M + \beta \log \left( M / \alpha \right)   \right] } + \frac{M}{\alpha} \sqrt{ \frac{T A_{\max} }{ 8 K } } \left( 1 + 2 \log E \right).
\end{equation*}
\end{theorem}
\begin{proof}
See Appendix~\ref{sec:proof-thm-dap-OSMD-sampler}.
\end{proof}
Since the additional regret term is $\tilde{O}( (M/\alpha) \sqrt{ T/K } )$, which is no larger than the first term asymptotically in its dependency on $T$ except for $\log$ terms, the bound on the regret is of the same order as in~\eqref{eq:regret-meta-pre}. However, we do not need to specify $\beta$ to set the learning rate.

Following Theorem~\ref{thm:thm-dap-OSMD-sampler} and the proofs of Theorem~\ref{thm:convg-thm-mini-batch} and Theorem~\ref{thm:convg-thm-fedavg}, we then have the following optimization guarantees on the Adaptive-OSMD Sampler.
\begin{theorem}
\label{thm:adp-osmd-opt-thm}
Assume Assumption~\ref{assump:diff-smooth}---\ref{assump:local-sg} holds and $\lambda_m=\frac{1}{M}$ for all $m \in [M]$.
\begin{itemize}
\item Mini-batch SGD with Adaptive-OSMD Sampler. Let $\{w^1,\ldots,w^T \}$ be the sequence of iterates generated by Algorithm~\ref{alg:min-batch-osmd}, where in Line~\ref{step:samp-update}, the sampling distribution is updated by Algorithm~\ref{alg:adap-OSMD-sampler-meta} with $A_{\max}=\frac{G^2}{M^2}$.
Let $w^R$ denote an element of that sequence chosen uniformly at random. Besides, let $\mu_t = \mu$ for all $t \in [T]$, where
\begin{multline*}
\mu=\min \left\{\frac{1}{L}, \frac{1}{\sigma}\sqrt{ \frac{ D^F K B \alpha  }{ L T }}, \frac{1}{\zeta_T(\alpha,\beta)} \sqrt{  \frac{D^F K}{L T} }, \right. \\
\left. \frac{ \sqrt{D^F K} \alpha^\frac{3}{2} }{\sqrt{L M} T^{\frac{1}{4}}  G \left( \frac{1}{2}\log M + \beta \log \left( M / \alpha \right) \right)^{\frac{1}{4}} }, \sqrt{  \frac{ \alpha D^F }{ L G } } \left( \frac{K}{T} \right)^{\frac{1}{2}} \sqrt{  \frac{1}{1 + 2 \log E} } \right\},
\end{multline*}
we then have
\begin{equation}
\label{eq:convg-rate-minibatchSGD-ada}
\begin{aligned}
& \quad \mathbb{E} \left[ \left\Vert \nabla F \left( w^R \right) \right\Vert^2 \right] \\
& \lesssim \frac{D^F L}{T} + \frac{ \sigma \sqrt{D^F L} }{ \sqrt{TKB \alpha } } + \frac{\zeta_T(\alpha,\beta) \sqrt{D^F L }}{\sqrt{T K}} + \frac{ \sqrt{ D^F L} M^{\frac{1}{2}} G }{ T^{\frac{3}{4}} K^{\frac{1}{2}} \alpha^{\frac{3}{2}} } \left( \frac{1}{2}\log M + \beta \log \left( M / \alpha \right)\right)^{\frac{1}{4}} \\
& \quad + \sqrt{ \frac{ D^F L G }{ \alpha } } \left( \frac{1}{K} \right)^{\frac{1}{4}} \left( \frac{1}{T} \right)^{\frac{3}{4}} \sqrt{ 1 + 2 \log E } .
\end{aligned}
\end{equation}
\item FedAvg with Adaptive-OSMD Sampler. 
Let $\{w^1,\ldots,w^T \}$ be the sequence of iterates generated by Algorithm~\ref{alg:fed-avg-osmd}, where in Line~\ref{step:samp-update}, the sampling distribution is updated by Algorithm~\ref{alg:adap-OSMD-sampler-meta} with $A_{\max}=\frac{BG^2}{M^2}$.
Let $w^R$ denote an element of that sequence chosen uniformly at random.
Besides, let $\mu^t = \mu \geq 1$ and $\mu^t_l = \mu_l$ for all $t \in [T]$, where
\begin{multline*}
\mu_l = \min \left\{ \frac{1}{4 \mu B L} \sqrt{\frac{1}{2+1/\alpha}},  \frac{(D^F)^{\frac{1}{3}}}{  \left( 4 + \frac{2}{\alpha}  \right)^{\frac{1}{3}} \mu B L^{\frac{2}{3}} \left( \zeta^2_{\text{unif}}  + \frac{\sigma^2}{2 B} \right)^{\frac{1}{3}} T^{\frac{1}{3}} }, \right.\\ 
\left.  \frac{ \sqrt{ 2 D^F }  }{  \mu B \sqrt{L} \sqrt{ \frac{2 \zeta^2_T (\alpha, \beta)}{K} +  \frac{\sigma^2}{K B \alpha } + \sqrt{ \frac{ \frac{1}{2} \log M + \beta \log (M/\alpha) }{T} } } \sqrt{T} }, \right. \\
\left. \frac{1}{\mu}\sqrt{ \frac{ 2 \alpha D^F  }{ L G \left( 1 + 2 \log E \right)  }  } \left( \frac{1}{B} \right)^{\frac{3}{4}} \left( \frac{8 K}{T} \right)^{\frac{1}{4}} \right\}, 
\end{multline*}
we then have
\begin{multline}
\label{eq:convg-rate-fedavg-osmd-ada}
\mathbb{E} \left[ \left\Vert \nabla F(w^R)  \right\Vert^2_2 \right]
\lesssim \frac{D^F L \sqrt{2 + \frac{1}{\alpha}}}{T} + \frac{ \left( 4 + \frac{2}{\alpha}  \right)^{\frac{1}{3}} (D^F L)^{\frac{2}{3}} \zeta^{\frac{2}{3}}_{\text{unif}}  }{ T^{\frac{2}{3}} } + \frac{ \left( 4 + \frac{2}{\alpha}  \right)^{\frac{1}{3}} (D^F L)^{\frac{2}{3}}  \sigma^{\frac{2}{3}}  }{ B ^{\frac{1}{3}} T^{\frac{2}{3}} } \\
+ \frac{ \sqrt{ D^F L }\zeta_T (\alpha, \beta) }{ \sqrt{T K} } + \frac{ \sqrt{ D^F L } \sigma  }{ \sqrt{ T K B \alpha } } + \frac{ \sqrt{D^F L } }{ \sqrt{T} } \left( \frac{ \frac{1}{2} \log M + \beta \log (M/\alpha) }{T} \right)^{\frac{1}{4}} \\
+ \sqrt{ \frac{D^F L G}{\alpha}  } \left(  \frac{1}{K B}\right)^{\frac{1}{4}} \left( \frac{1}{T} \right)^{\frac{3}{4}} \sqrt{1 + 2 \log E}.
\end{multline}
\end{itemize}
\end{theorem}

Compare~\eqref{eq:convg-rate-minibatchSGD} with~\eqref{eq:convg-rate-minibatchSGD-ada} (respectively, \eqref{eq:convg-rate-fedavg-osmd} with~\eqref{eq:convg-rate-fedavg-osmd-ada}), there is an additional error term in~\eqref{eq:convg-rate-minibatchSGD-ada} (respectively, \eqref{eq:convg-rate-fedavg-osmd-ada}). However, this additional error term is no larger than the previous error term when considering its dependency on $T$.
Furthermore, when we let $\beta=T^{c}$ with $0<c<1$ and treat all the other parameters except for $\beta$ and $T$ to be constants, we then have the last term dominated by the penultimate term.
On the other hand, in exchange to the additional error, we do not need to specify $\beta$ when setting the learning rate~$\eta$.

Note that Adaptive-OSMD Sampler still requires the number of iterations $T$ as an input; however, in some applications, $T$ is not available before the training starts. For example, one may use some stopping criterion to determine when to stop.
To deal with such cases, in Appendix~\ref{sec:adp-osmd-sampler-doub}, we introduce an extension of Adaptive-OSMD Sampler that does not need $T$ as input by using doubling trick.
The extension algorithm enjoys similar regret and optimization guarantees as Adaptive-OSMD Sampler. See Appendix~\ref{sec:adp-osmd-sampler-doub} for more details.


\section{Simulation Experiments}
\label{sec:experiments}

In this section, we use simulated data to demonstrate the performance of Adaptive-OSMD Sampler (Algorithm~\ref{alg:adap-OSMD-sampler-meta}). 
We compare our method against uniform sampling in Section~\ref{sec:AdaDoubOSMD-uniform} and compare against other bandit feedback online learning samplers in Section~\ref{sec:AdaDoubOSMD-mabs-vrb}. In addition, we examine the robustness of Adaptive-OSMD Sampler to the choice of $\alpha$ in Section~\ref{sec:robu-alpha}, while in Section~\ref{sec:dyn-vs-import}, we compare Adaptive-OSMD Sampler with the Lipschitz constant based importance sampling.

We generate data as follows. We set the number of clients as $M=100$, and each client has $n_m=100$ samples, $m \in [M]$. Samples on each client are generated as
\begin{equation}
\label{eq:sync-data-generate}
y_{m,i}=\langle w_{\star},x_{m,i} \rangle + N(0,0.1^2), \qquad i \in [n_m],
\end{equation}
where the coefficient vector $w_{\star} \in \mathbb{R}^d$ has elements generated as i.i.d.~$N(10,3)$, and the feature vector $x_{m,i} \in \mathbb{R}^d$ is generated as $x_{m,i} \sim N(0,\Sigma_m)$, where $\Sigma_m = s_m \cdot \Sigma$, $\Sigma$ is a diagonal matrix with $\Sigma_{jj}=\kappa^{(j-1)/(d-1)-1}$, $\forall j \in [d]$ and $\kappa>0$ is the condition number of $\Sigma$. We generate $\{s_m\}^M_{m=1}$ i.i.d.~from $e^{ N(0,\sigma^2) }$ and rescale them as $s_m \leftarrow (s_m / \max_{m \in [M]} s_m) \times 10$ so that $s_m \leq 10$ for all $m \in [M]$. In this setting, $\kappa$ controls the difficulty of each problem when solved separately, while $\sigma$ controls the level of heterogeneity across clients. In all experiments, we fix $\kappa=25$, which corresponds to a hard problem, and change $\sigma$ to simulate different heterogeneity levels. We expect that uniform sampling suffers when the heterogeneity level is high. The dimension $d$ of the problem is set as $d=10$. The results are averaged over $10$ independent runs.

We use the mean squared error loss defined as 
\[
L(w)= \frac1M \sum^M_{m=1} L_m(w),
\qquad \text{where} \qquad 
L_m(w)= \frac{1}{2 n_m} ( y_{m,i} - \langle w_{\star},x_{m,i} \rangle )^2.
\]
We use the stochastic gradient descent to make global updates. At each round $t$, we choose a subset of $K=5$ clients, denoted as $S^t$. For each client $m \in S^t$, we choose a mini-batch of samples, $\mathcal{B}^t_m$, of size $\bar{B}=10$, and compute the mini-batch stochastic gradient. The parameter $w$ is updated as
\begin{equation*}
w^{t+1} = w^t + \frac{\mu_{\text{SGD}}}{M K \bar{B}} \sum_{m \in S^t} \frac{1}{ p^t_m } \sum_{i \in \mathcal{B}^t_m} \left( y_{m,i} - \langle w_{\star},x_{m,i} \rangle \right) \cdot x_{m,i},
\end{equation*}
where $\mu_{\text{SGD}}$ is the learning rate, set as $\mu_{\text{SGD}}=0.1$ in simulations.

In all experiments, we set $\alpha$ in Adaptive-OSMD Sampler as $\alpha=0.4$. The tuning parameters for MABS, VRB and Avare are set as in their original papers. 

\subsection{Adaptive-OSMD Sampler vs Uniform Sampling}
\label{sec:AdaDoubOSMD-uniform}

\begin{figure}[t]
\centering
\includegraphics[width=0.8\textwidth]{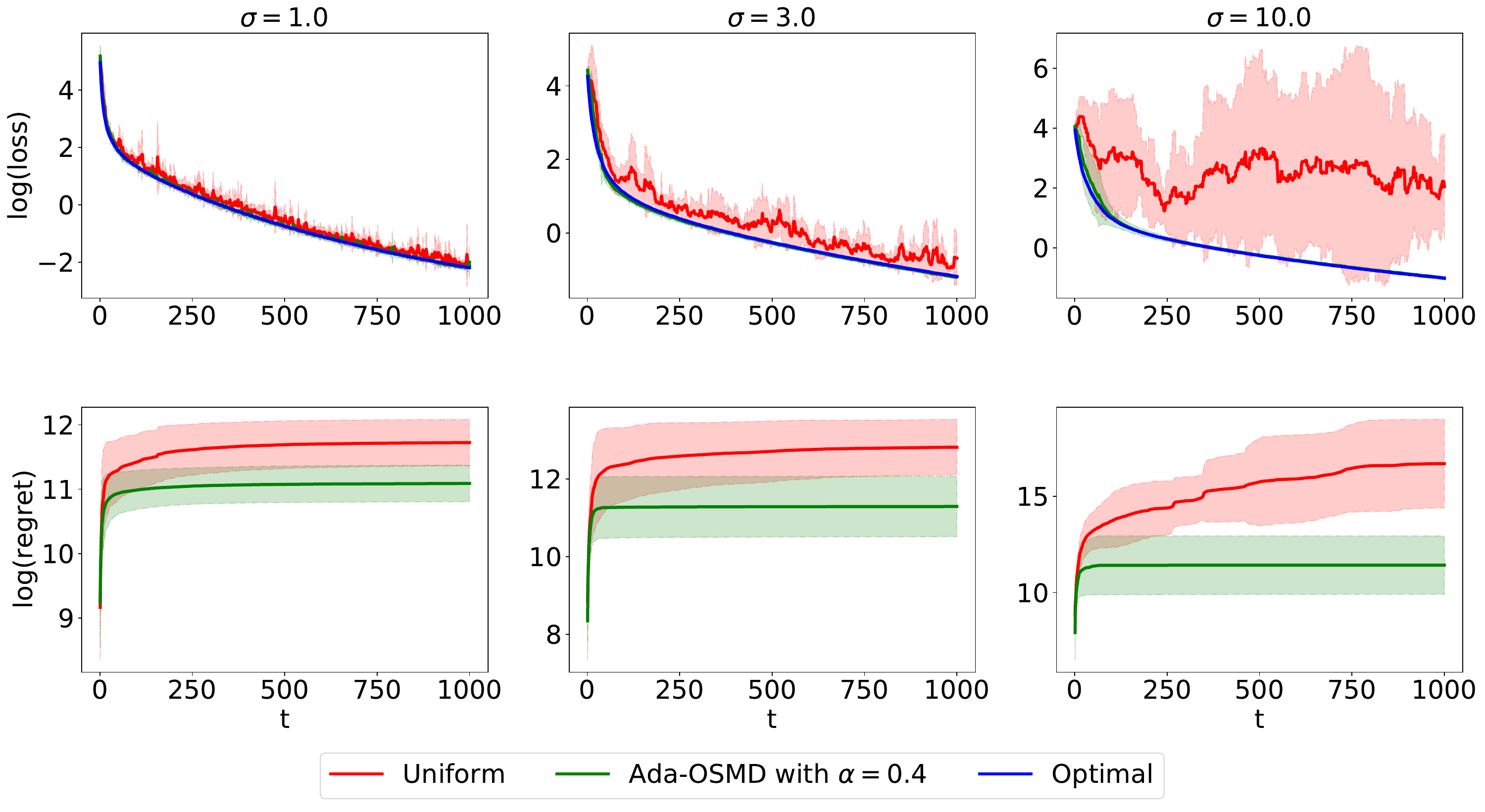}
\caption{The training loss (top row) and cumulative regret (bottom row) are compared for the Adaptive-OSMD Sampler, Uniform Sampler, and Optimal Sampler, under $\sigma = 1.0$, $\sigma = 3.0$, and $\sigma = 10.0$. Solid lines represent the mean values, while the shaded regions indicate $\text{mean} \pm \text{standard deviation}$ across independent runs.}
\label{fig:ada-osmd-uniform-opt-loss-regret}
\end{figure}

The results of the training process and the cumulative regret are shown in Figure~\ref{fig:ada-osmd-uniform-opt-loss-regret}. For the training loss, we see that when the heterogeneity level is low ($\sigma=1.0$), the uniform sampling performs as well as Adaptive-OSMD Sampler and theoretically optimal sampling; however, as the heterogeneity level increases, the performance of uniform sampling gradually suffers; when $\sigma=10.0$, uniform sampling performs poorly. On the other hand, Adaptive-OSMD Sampler performs well across all levels of heterogeneity and is very close to the theoretically optimal sampling. Similarly, for the cumulative regret, when the heterogeneity level is low, the cumulative regret of uniform sampling is close to Adaptive-OSMD Sampler; however, when the heterogeneity level increases, the cumulative regret of uniform sampling gets much larger than Adaptive-OSMD Sampler. Based on the above results, we can conclude that while the widely used choice of uniform sampling may be reasonable when heterogeneity is low, our proposed sampling strategy is robust across different levels of heterogeneity, and thus should be considered as the default option.

\subsection{Adaptive-OSMD Sampler vs MABS vs VRB vs Avare}
\label{sec:AdaDoubOSMD-mabs-vrb}

\begin{figure}[t]
\centering
\includegraphics[width=0.8\textwidth]{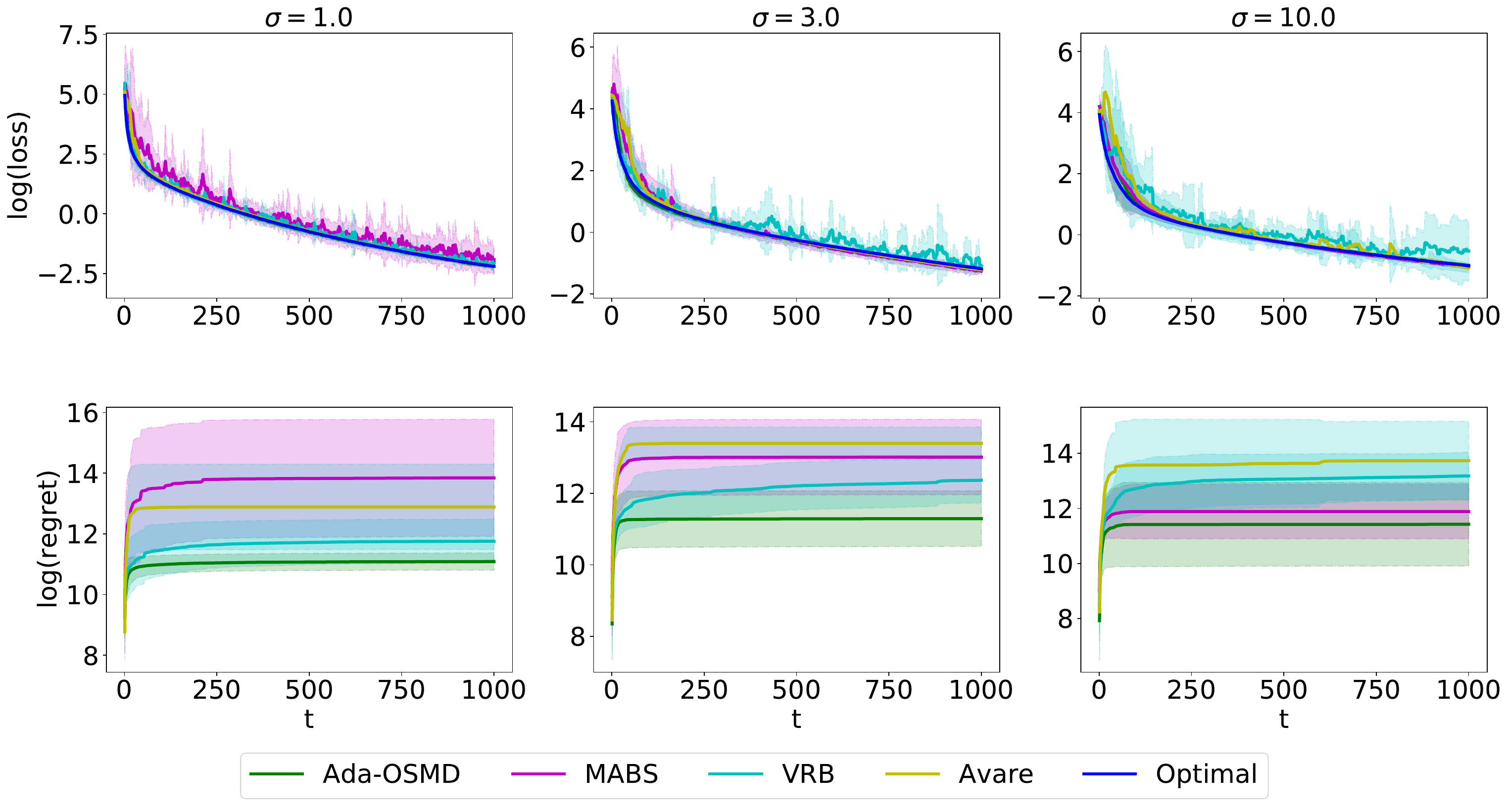}
\caption{The training loss (top row) and cumulative regret (bottom row) are compared across the Adaptive-OSMD Sampler, MABS, VRB, and Avare methods for $\sigma = 1.0$, $\sigma = 3.0$, and $\sigma = 10.0$. Solid lines represent the mean values, while shaded regions indicate $\text{mean} \pm \text{standard deviation}$ across independent runs.}
\label{fig:ada-osmd-MABS-VRB-Avare-loss-regret}
\end{figure}

We compare Adaptive-OSMD Sampler to other bandit feedback online learning samplers: MABS~\citep{salehi2017stochastic}, VRB~\citep{borsos2018online} and Avare~\citep{el2020adaptive}. Training loss and cumulative regret are shown in Figure~\ref{fig:ada-osmd-MABS-VRB-Avare-loss-regret}. We see that while VRB and Avare perform better when the heterogeneity level is low and MABS performs better when the heterogeneity level is high, Adaptive-OSMD Sampler always achieves the best in both training loss and cumulative regret across all different levels of heterogeneity. Thus, we conclude that Adaptive-OSMD is a better choice than other online learning samplers.

\subsection{Robustness of Adaptive-OSMD Sampler to the Choice of $\alpha$}
\label{sec:robu-alpha}

We examine the robustness of Adaptive-OSMD Sampler to the choice of $\alpha$. We run Adaptive-OSMD Sampler separately for each $\alpha \in \{0.01,0.1,0.4,0.7,0.9,1.0\}$. Note that when $\alpha=1.0$, the Adaptive-OSMD Sampler outputs a uniform distribution. Training loss and cumulative regret are shown in Figure~\ref{fig:ada-osmd-alpha-loss-regret}. We observe that Adaptive-OSMD Sampler is robust to the choice of $\alpha$, and performs well as long as $\alpha$ is not too close to zero or too close to one.

\begin{figure}[t]
\centering
\includegraphics[width=0.8\textwidth]{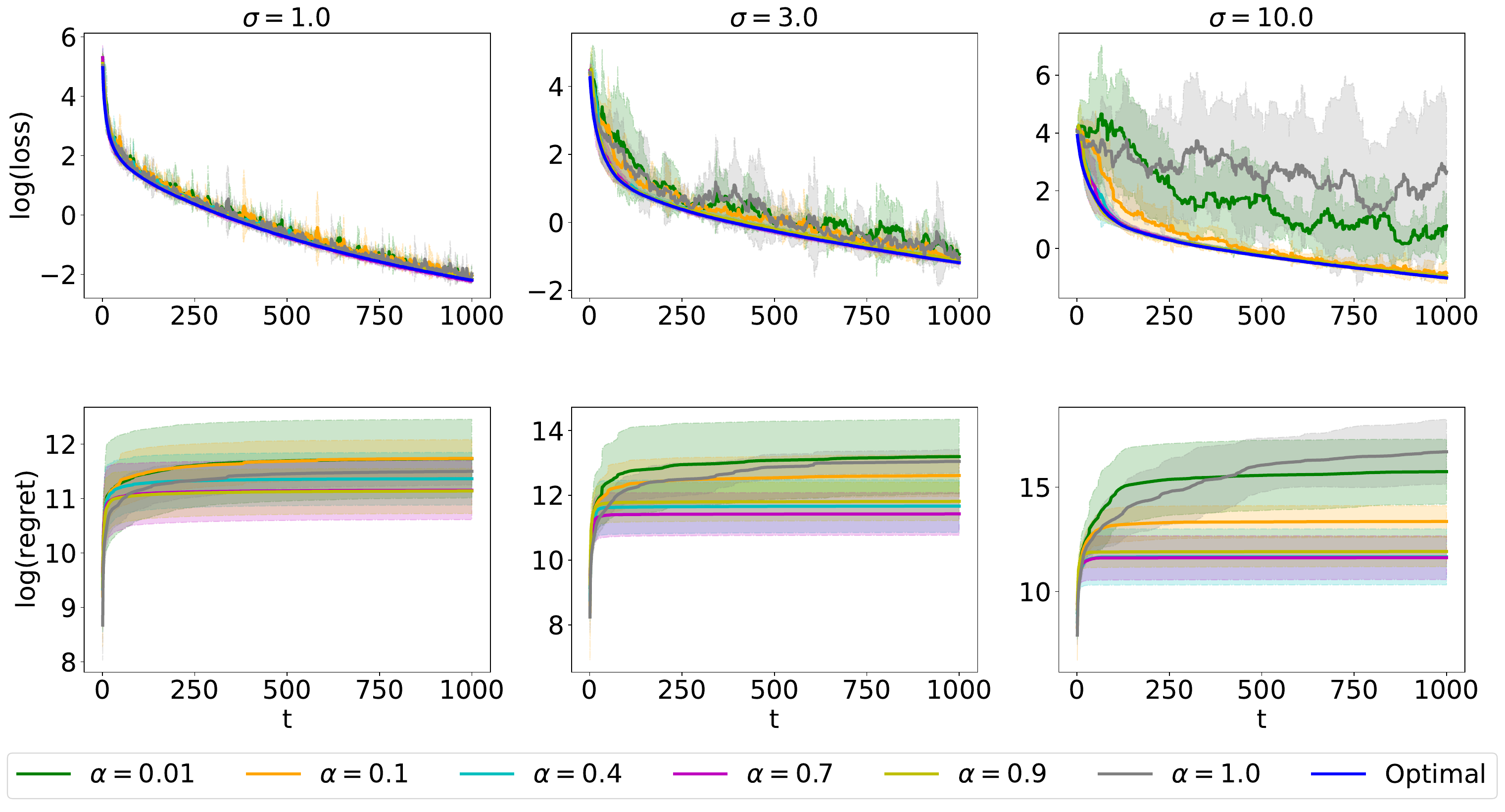}
\caption{The training loss (top row) and cumulative regret (bottom row) are shown for the Adaptive-OSMD Sampler with different values of $\alpha$ under $\sigma = 1.0$, $\sigma = 3.0$, and $\sigma = 10.0$. Solid lines represent the mean values, while shaded regions indicate $\text{mean} \pm \text{standard deviation}$ across independent runs.}
\label{fig:ada-osmd-alpha-loss-regret}
\end{figure}




\subsection{Dynamic Sampling Distribution v.s. Fixed Sampling Distribution}
\label{sec:dyn-vs-import}

In this paper, we allow both our sampling distribution and competitor sampling distribution to change over time, while previous studies either use a fixed sampling distribution~\citep{zhao2015stochastic,needell2014stochastic} or they compare against a fixed sampling distribution~\citep{namkoong2017adaptive, salehi2017stochastic, borsos2018online, borsos2019online}. In this section, we show that under certain settings, a dynamic sampling distribution can achieve a significant advantage over a fixed sampling distribution. More specifically, we compare the Adaptive-OSMD Sampler with the Lipschitz constant-based importance sampling distribution proposed by~\citet{zhao2015stochastic,needell2014stochastic}, which we denote as $p^{\text{IS}}$.

We still use the same model as in~\eqref{eq:sync-data-generate} to generate data. but we generate $w_{\star}$ and $x_{m,i}$ differently. Motivated by~\cite{zhao2022svrg}, for each $m \in [M]$, we choose uniformly at random one dimension among $\mathbb{R}^d$, denoted as $\text{supp}(m) \in [d]$, as the support of $x_{m,i}$ for all $i \in [n_m]$, while the remaining dimensions of $x_{m,i}$ are set to be zero. The nonzero dimension of $x_{m,i}$ is generated from $N(1.0, 0.1^2)$. The entries of $w_{\star}$ are generated i.i.d.~from $e^{N(0, \nu^2)}$. Therefore, $\nu$ controls the variance of entries of $w_{\star}$.

Besides, we choose the optimal stepsize from the set $\{1.0, 0.5, 0.1, 0.05, 0.01\}$ for each method separately. The final result is shown in Figure~\ref{fig:fixed-loss-regret}. We see that Adaptive-OSMD Sampler performs better than $p^{\text{IS}}$ across all levels of $\nu$. Note that in practice, in order to implement $p^{\text{IS}}$, we need prior information about Lipschitz constants of $L_m(\cdot)$'s, while  Adaptive-OSMD Sampler does not need prior information. This way, our proposed method does not only have better practical performance, but also requires less prior information.

\begin{figure}[t]
\centering
\includegraphics[width=0.8\textwidth]{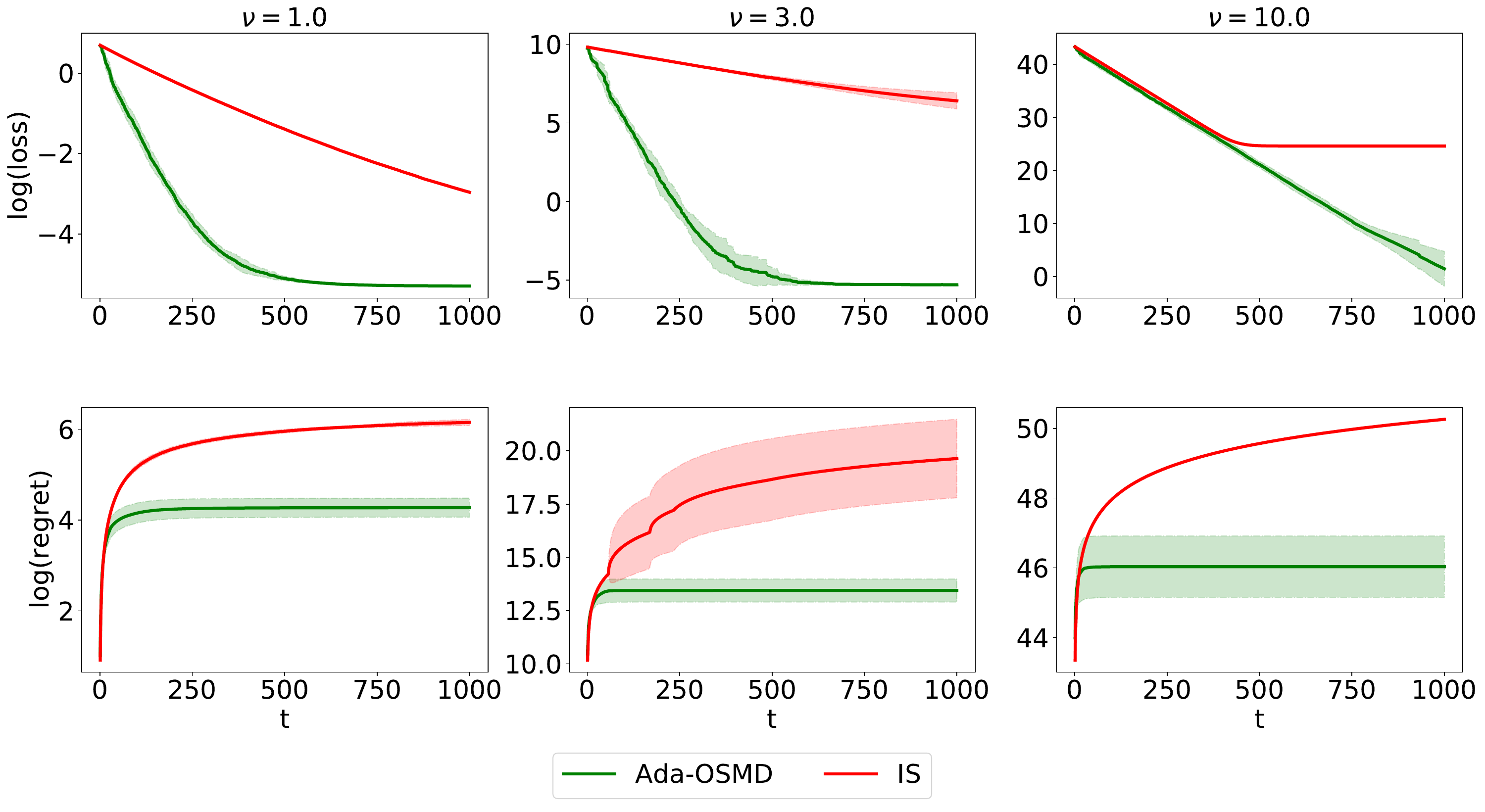}
\caption{The training loss (top row) and cumulative regret (bottom row) are compared between the Adaptive-OSMD Sampler and $p^{\text{IS}}$ for $\nu = 1.0$, $\nu = 3.0$, and $\nu = 10.0$. Solid lines represent the mean values, while shaded regions illustrate $\text{mean} \pm \text{standard deviation}$ across independent runs.}
\label{fig:fixed-loss-regret}
\end{figure}

\section{Real Data Experiment}
\label{sec:real-data-exp}

We compare Adaptive-OSMD Sampler with uniform sampling and other online learning samplers including MABS~\citep{salehi2017stochastic}, VRB~\citep{borsos2018online} and Avare~\citep{el2020adaptive} on real data. We use three commonly used computer vision data sets: MNIST~\citep{lecun-mnisthandwrittendigit-2010}\footnote{Yann LeCun and Corinna Cortes hold the copyright of MNIST data set, which is a derivative work from original NIST data sets. MNIST data set is made available under the terms of the Creative Commons Attribution-Share Alike 3.0 license.}, KMINIST~\citep{clanuwat2018deep}\footnote{KMNIST data set is licensed under a permissive CC BY-SA 4.0 license, except where specified within some benchmark scripts.}, and FMINST~\citep{xiao2017fashion}\footnote{FMNIST data set is under The MIT License (MIT) Copyright © [2017] Zalando SE, https://tech.zalando.com}. We set the number of devices to be $M=500$. To better simulate the situation where our method brings significant convergence speed improvement, we create a highly skewed sample size distribution of the training set among clients: $65$\% of clients have only one training sample, $20$\% of clients have $5$ training samples, $10$\% of clients have $30$ training samples, and $5$\% of clients have $100$ training samples. This setting tries to illustrate a real-life situation where most of the data come from a small fraction of users, while most of the users have only a small number of samples. The skewed sample size distribution is common in other FL data sets, such as LEAF \citep{caldas2018leaf}. The sample size distribution in the training set is shown in Figure~\ref{fig:sample-size-distrib}. In addition, each client has $10$ validation samples used to measure the prediction accuracy of the model over the training process.

\begin{figure}[t]
\centering
\includegraphics[width=0.75\textwidth]{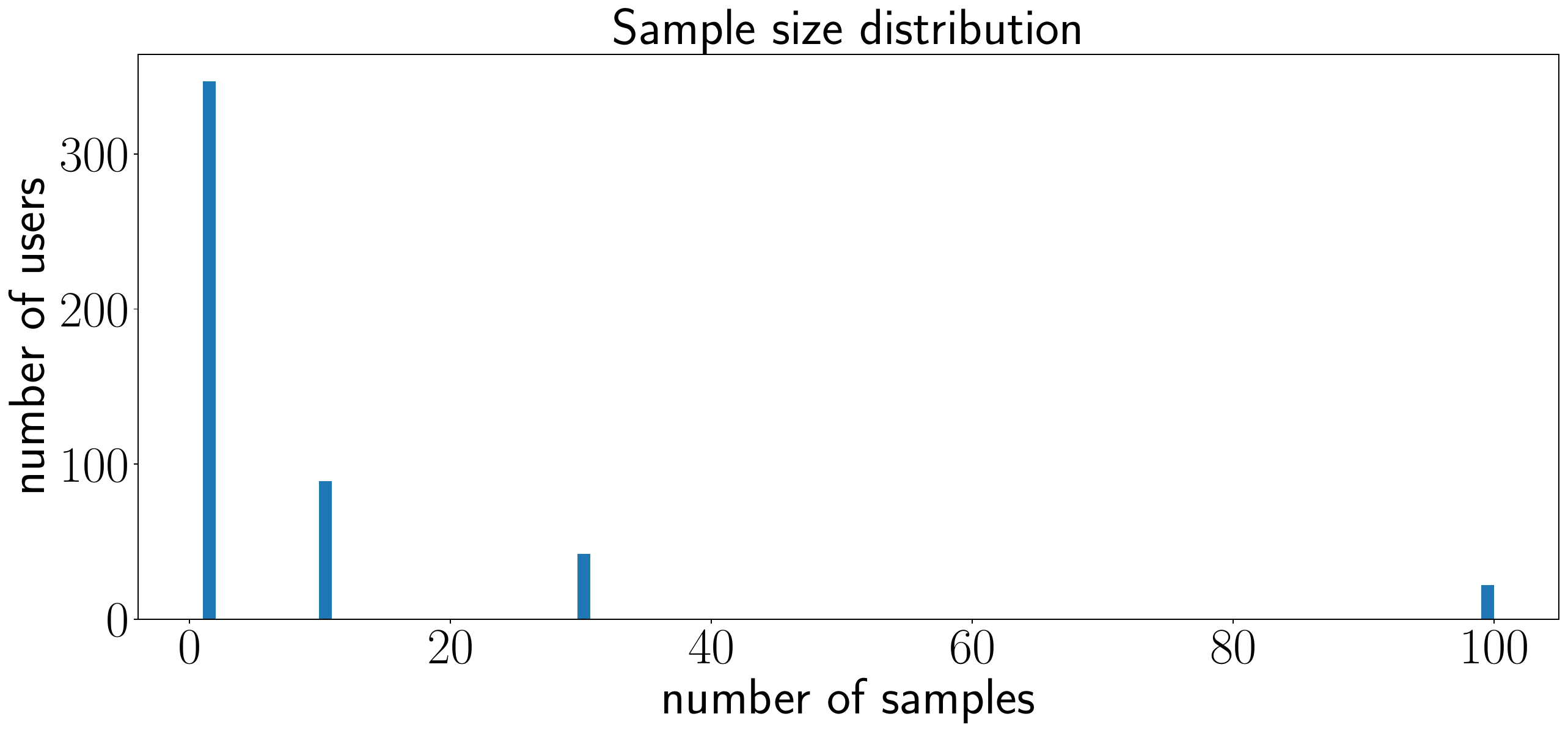}
\caption{The sample size distribution in the training set across clients.}
\label{fig:sample-size-distrib}
\end{figure}

\begin{figure}[t]
\centering
\includegraphics[width=\textwidth]{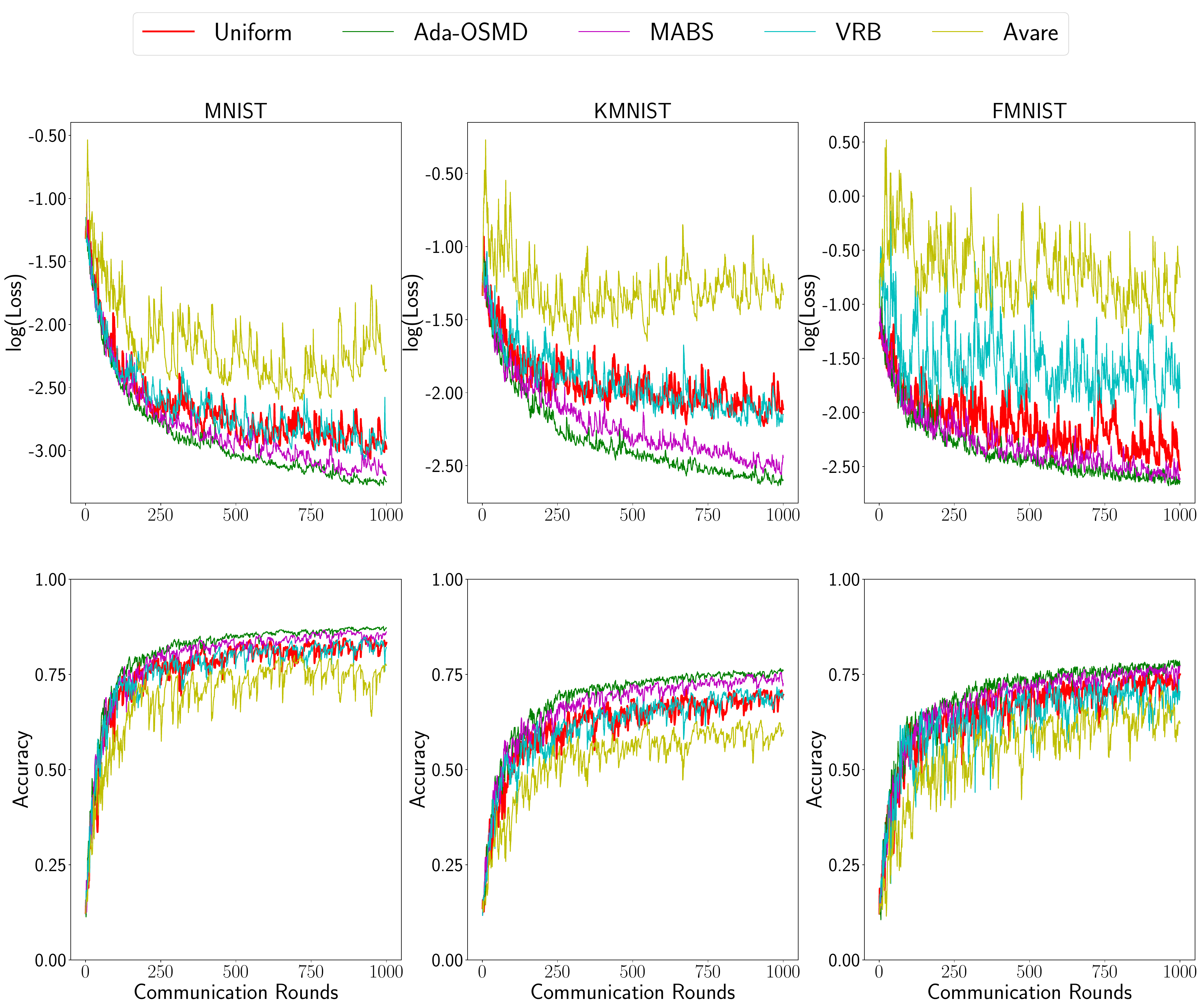}
\caption{Comparison of Adaptive-OSMD Sampler, Uniform Sampler, and Other Online Learning Samplers
The comparison is based on real data, evaluating training loss (top row) and validation accuracy (bottom row). Each column represents a different data set. The Adaptive-OSMD Sampler demonstrates superior performance, being both faster and more stable. The results represent the average performance over five independent runs.}
\label{fig:real-data-cv}
\end{figure}


We use a multi-class logistic regression model. For a given gray scale picture with the label $y \in \{1,2,\dots,C\}$, we unroll its pixel matrix into a vector $x \in \mathbb{R}^{p}$. Given a parameter matrix $W \in \mathbb{R}^{C \times p}$, the training loss function defined in~\eqref{eq:SingleFL-obj} is
\begin{equation*}
\phi (W ; x, y) \coloneqq l_\text{CE} \left(\varsigma (W x) \, ; y \right),
\end{equation*}
where $\varsigma (\cdot) : \mathbb{R}^{C} \rightarrow \mathbb{R}^{C}$ is the softmax function defined as 
\begin{equation*}
\left[ \varsigma (x) \right]_i = \frac{\exp(x_i)}{\sum^K_{j=1} \exp(x_j)}, \quad \text{ for all } x \in \mathbb{R}^C,
\end{equation*}
and $l_\text{CE}(x \,; y)=\sum^C_{i=1} \mathds{1}(y=i) \log x_i$, $x \in \mathbb{R}^C$, $y\in \{1,\dots,C\}$, is the cross-entropy function.

We use the same algorithms and tuning parameters as in Section~\ref{sec:experiments}. Learning rate in SGD is set to $0.075$ for MNIST and KMNIST, and is set to $0.03$ for FMNIST. The total number of communication rounds is to $1,000$. In each round of communication, we choose $K=10$ clients to participate ($2$\% of total number of clients). For a chosen client $m$, we compute its local mini-batch gradient with the batch size equal to $\min\{5,n_m\}$, where $n_m$ is the training sample size on the client $m$.

Figure~\ref{fig:real-data-cv} shows both the training loss and validation accuracy. Each figure shows the average performance over 5 independent runs. We use the same random seed for both Adaptive-OSMD Sampler and competitors, and change random seeds across different runs. The main focus is on minimizing the training loss, and the validation accuracy is only included for completeness. We observe that Adaptive-OSMD Sampler performs better than uniform sampling and other online learning samplers across all data sets. 


\section{Conclusion}
\label{sec:conclusion}

We studied the client sampling problem in FL. We proposed an online learning with bandit feedback approach to tackle client sampling. We used online stochastic mirror descent to solve the online learning problem and applied the online ensemble method to choose the tuning parameters. We established an upper bound on the dynamic regret relative to any sequence of sampling distributions. 
Besides, we provide optimization guarantees for our sampling method when used with mini-batch SGD and FedAvg.
Extensive numerical experiments demonstrated the benefits of our approach over both widely used uniform sampling and other competitors.

In this paper, we have focused on sampling with replacement. However, sampling without replacement would ideally be a more efficient approach. In Section~\ref{sec:samp-wo-replace}, we discussed a natural extension of Adaptive-OSMD Sampler to a setting where sampling without replacement is used. However, this approach does not directly minimize the variance of the gradient $g^t$. When sampling without replacement is used, the variance function becomes more complicated and the design of an algorithm to directly minimize the variance is an interesting future direction.

Besides, in federated learning, privacy is a major concern. In this paper, the non-uniform sampling distribution may make the protection of clients' privacy more challenging than uniform sampling. One possible solution is to add noise to the gradient feedback and protect the clients' privacy under the Differential Privacy (DP) concept~\citep{dwork2008differential}. However, the added noise may hurt the performance of our sampling design and increase the regret. Studying the trade-off between privacy protection and regret is an important direction for addressing societal concerns in real-world applications.

Other fruitful future directions include the design of sampling algorithms for minimizing personalized FL objectives and sampling with physical constraints in the FL system, which we discuss in Appendix~\ref{sec:future-directions}.

\acks{
This work was completed in part with resources provided by the University of Chicago
Booth Mercury Computing Cluster.
The research of MK is supported in part by NSF Grant ECCS-2216912.
}

\FloatBarrier

\newpage

\appendix

\section{Extensions of Adaptive-OSMD Sampler}
\label{sec:extensions}

In this section, we discuss additional extensions of Adaptive-OSMD Sampler. In Section~\ref{sec:adp-osmd-sampler-doub}, we discuss how to choose $\eta$ without knowing $T$ in advance. In Section~\ref{sec:samp-wo-replace}, we discuss how to extend Adaptive-OSMD Sampler to sample without replacement setting.

\subsection{Adaptive-OSMD Sampler with Doubling Trick} 
\label{sec:adp-osmd-sampler-doub}

Algorithm~\ref{alg:adap-OSMD-sampler-meta} requires the total number of iterations $T$ as input, which is not always available in practice. In those cases, we use doubling trick~\citep[Section 2.3]{cesa2006prediction} to avoid this requirement. The basic idea is to restart Adaptive-OSMD Sampler at exponentially increasing time points $T_b=2^{b-1}$, $b\geq1$. The learning rates of experts in Algorithm~\ref{alg:adap-OSMD-sampler-meta} are reset at the beginning of each time interval, and the meta-algorithm learning rate $\gamma$ is chosen optimally for the interval length. 

More specifically, let $A_{\max}>0$ such that $a^t_{\max} \leq A_{\max}$ for all $t$, where $a^t_{\max} = \max_{1\leq m \leq M} a^t_m = \max_{1\leq m \leq M} \lambda^2_m \Vert g^t_m \Vert^2_2$.
At the time point $T_b$, we let
\begin{equation}\label{eq:doub-trick-set}
\mathcal{E}_b \coloneqq \left\{ \left.  2^{e-\frac{b}{2}-\frac{1}{2}} \cdot \frac{K \alpha^3 \sqrt{ 2 \log M } }{ M^3  A_{\max}  }  \, \right\vert \, e=1,2,\dots,E_b \right\},
\end{equation}
where
\begin{equation}
\label{eq:doub-trick-set-length}
E_b = \left\lceil \frac{1}{2} \log_2 \left( 1 + \frac{4 \log (M/\alpha)}{\log M}( 2^{b-1} - 1) \right) \right\rceil   + 1,
\end{equation}
and $\gamma_b=\frac{\alpha}{M} \sqrt{\frac{8 K}{2^{b-1} A_{\max}}}$, $b\geq 1$. 

In a practical implementation, at the time point $t=T_b$, instead of initializing all expert algorithms using uniform distribution, we can initialize them with the output of the meta-algorithm for $t=T_b-1$.
Besides, since the gradient norm is usually decreasing, we can estimate $A_{\max}$ by adding a pre-training phase where we broadcast the initial model parameter $w^0$ to all devices before the start of the training, and collect the returned $\Vert g^0_m \Vert^2_2$ from all responsive devices, which we denote as $S^0$. Then we can estimate $A_{\max}$ by $\hat{A}_{\max} = \max_{m \in S^0} \lambda_m \Vert g^0_m \Vert^2_2$.

Adaptive-Doubling-OSMD Sampler is detailed in Algorithm~\ref{alg:adap-OSMD-sampler-doub-trick}.
From the computational perspective, by the proof of Theorem~\ref{thm:adap-OSMD-sampler-doub-trick}, Algorithm~\ref{alg:adap-OSMD-sampler-doub-trick} needs to run Step~\ref{alg:expert-solve-mirror-descent} of Algorithm~\ref{alg:adap-OSMD-sampler-meta} for a total number of $O(T \vert \mathcal{E} \vert^2)=O(T \floor{\log_2 T})$ times. Therefore, the computational complexity of Adaptive-Doubling-OSMD Sampler is asymptotically the same as that of Adaptive-OSMD Sampler, while it increases by only a $\log(T)$ factor compared to OSMD Sampler. The following theorem provides a bound on the dynamic regret for Adaptive-Doubling-OSMD Sampler.

\begin{theorem}\label{thm:adap-OSMD-sampler-doub-trick}
Suppose the training is stopped after $T$ iterations.
Let $\hat{p}^{1:T}$ be the output of Algorithm~\ref{alg:adap-OSMD-sampler-doub-trick}, where $\punif$ is used in Step~\ref{line:alg4-re}.
Then for any comparator sequence $q^{1:T}$, where $q^t$ is allowed to be random, such that $q^t \in \mathcal{A}$ for all $t \in [T]$ and $\mathbb{E} [\text{TV}(q^{1:T})] \leq \beta$, we have
\begin{multline*}
\text{D-Regret}_T (q^{1:T}) \leq \frac{ 6 M^3 A_{\max} }{ (\sqrt{2}-1) K \alpha^3 } \sqrt{ T \left[ \frac{1}{2}\log M + \beta \log \left( M / \alpha \right) \right] } \\ 
+ \frac{2 M}{(\sqrt{2}-1)\alpha} \sqrt{ \frac{ T A_{\max} }{ 8 K } } (1 + 2 \log E )  .
\end{multline*}
\end{theorem}
\begin{proof}
See Appendix~\ref{sec:proof-thm-5}.
\end{proof}

Compare Theorem~\ref{thm:adap-OSMD-sampler-doub-trick} with Theorem~\ref{thm:thm-dap-OSMD-sampler}, we see that the regret bound of Adaptive-Doubling-OSMD has the same order as that of Adaptive-OSMD Sampler. However, Adaptive-Doubling-OSMD Sampler does not need to know $T$ in advance.
Mimicking the proof of Theorem~\ref{thm:adp-osmd-opt-thm}, we can also show the optimization guarantees on Adaptive-Doubling-OSMD with mini-batch SGD and FedAvg, which is basically the same as Theorem~\ref{thm:adp-osmd-opt-thm}, thus is omitted here.

\begin{algorithm}[t]
\caption{Adaptive-OSMD Sampler with Doubling Trick (Adaptive-Doubling-OSMD)}
\label{alg:adap-OSMD-sampler-doub-trick}
\begin{algorithmic}[1]
\STATE {\bfseries Input:} Paramter $\alpha$ and $A_{\max}$.
\STATE {\bfseries Output:} $\hat{p}^t$ for $t=1,\dots,T$.
\WHILE{True}{
\STATE Set $\mathcal{E}_b$ as in \eqref{eq:doub-trick-set}. 
\STATE Let $\gamma_b=\frac{\alpha}{M} \sqrt{\frac{8 K}{2^{b-1} A_{\max}}}$.
\STATE Obtain $\{\hat{p}^t \}^{2^b-1}_{t=2^{b-1}}$ from Algorithm~\ref{alg:adap-OSMD-sampler-meta} with parameters: $\gamma_b$, $\mathcal{E}_b$, $\alpha$, the number of iterations $2^{b-1}$, and the initial distribution $\punif$ or $\hat{p}^{2^{b-1}-1}$ (when $b>1$). \label{line:alg4-re}
\IF{Training Process is Converged}
\STATE Break.
\ENDIF
\STATE Let $b \leftarrow b + 1$.
}
\ENDWHILE
\end{algorithmic}
\end{algorithm}

\subsection{Adaptive Sampling Without Replacement}
\label{sec:samp-wo-replace}

\begin{algorithm}[ht]
\caption{Adaptive sampling without replacement}
\label{alg:adap-samp-wo-replace}
\begin{algorithmic}[1]
\STATE {\bfseries Input:} $w^1$ and $\hat{p}^1$.
\FOR{$t=1,2,\dots,T-1$}{
\STATE Let $\hat{p}^{t}_{(1)}=\hat{p}^t$ and sample $m^t_1$ from $[M]$ by $\hat{p}^{t}_{(1)}$.
\FOR{$k=2,\cdots,K$}{
\STATE \COMMENT{Design the sampling distribution for sampling the $k$-th client \\in the $t$-th round}
\STATE Construct $\hat{p}^{t}_{(k)}$ by letting 
\begin{equation*}
\hat{p}^{t}_{(k),m} = 
\left\{
\begin{array}{ll}
\left( 1 - \sum^{k-1}_{l=1} \hat{p}^{t}_{m^t_{l} }   \right)^{-1} \hat{p}^{t}_m &  \text{if } m \in [M] \backslash \{ m^t_1, \dots, m^t_{k-1} \} \\
0 & \text{otherwise.}
\end{array}
\right.
\end{equation*}
\STATE \COMMENT{Sample the $k$-th client}
\STATE Sample $m^t_k$ from $[M] \backslash  \{ m^t_1, \dots, m^t_{k-1} \}$ by $\hat{p}^{t}_{(k)}$.
}
\ENDFOR
}
\STATE Let $S^t=\{m^t_1,\cdots,m^t_K\}$. 
\STATE The server broadcasts the model parameter $w^t$ to clients in $S^t$. 
\STATE The clients in $S^t$ compute and upload the set of local gradients $\left\{g^t_{m^t_1},\cdots,g^t_{m^t_K}\right\}$.
\STATE \COMMENT{Construct global gradient estimate}
\STATE Let $g^t_{(1)} = \lambda^t_{m^t_1} g^t_{m^t_1}/\hat{p}^{t}_{(1),m^t_1}$.
\FOR{$k=2,\cdots,K$}{
\STATE Let $g^t_{(k)} = \lambda^t_{m^t_{k}}  g^t_{m^t_k} / \hat{p}^{t}_{(k),m^t_k} + \sum^{k-1}_{l=1} \lambda^t_{m^t_{l}} g^t_{m^t_{l}}$.
}
\ENDFOR
\STATE Let $\tilde{g}^t = K^{-1} \sum^K_{k=1} g^t_{(k)}$. \label{eq:glad_no_replace}
\STATE \COMMENT{Update the model weight based on the global gradient estimate}
\STATE Obtain the updated model parameter $w^{t+1}$ using $w^{t}$ and $\tilde{g}^t$.
\STATE \COMMENT{Update sampling distribution}
\STATE Let $a^t_m=\lambda^2_m \Vert g^t_m \Vert^2$ for $m \in S^t$. 
\STATE Input $\{a^t_m\}_{m \in S^t}$ into Adaptive-OSMD Sampler to get $\hat{p}^{t+1}$.
\ENDFOR
\end{algorithmic}
\end{algorithm}

In the discussion so far, we have assumed that the set $S^t$ is obtained by sampling with replacement from $p^t$. When $K$ is relatively large compared to $M$ and $p^t$ is far from uniform distribution, sampling without replacement can be more efficient than sampling with replacement. However, when sampling without replacement using $p^t$, the variance reduction loss does not have a clean form as in~\eqref{eq:var-reduc-loss}. As a result, an online design of the sampling distribution is more challenging. In this section, we discuss how to use the sampling distribution obtained by Adaptive-OSMD Sampler to sample clients without replacement, following the approach taken in~\cite{el2020adaptive}.

The detailed sampling procedure is described in Algorithm~\ref{alg:adap-samp-wo-replace}. We still use Adaptive-OSMD Sampler to update the sampling distribution. However, we use the designed sampling distribution in a way that no client is chosen twice. Furthermore, Step~\ref{eq:glad_no_replace} of Algorithm~\ref{alg:adap-samp-wo-replace} constructs the gradient estimate with the following properties.

\begin{proposition}[Proposition 3 of~\cite{el2020adaptive}]
\label{prop:propert-no-replacement}
Let $\hat{p}^t=p$ and let $\tilde{g}^t$ be as in Step~\ref{eq:glad_no_replace} of Algorithm~\ref{alg:adap-samp-wo-replace}. Note that $\tilde{g}^t = \tilde{g}^t(p)$ depends on $p$. Recall that $J^t=\sum^M_{m=1}\lambda_m g^t_m$. We have
\[
\mathbb{E}_{S^t}\left[ \tilde{g}^t \right] = J^t
\qquad 
\text{and}
\qquad
\arg\min_{p \in \mathcal{P}_{M-1}} \mathbb{E}_{S^t}\left[ \Vert \tilde{g}^t - J^t \Vert^2_2 \right] = \arg\min_{p \in \mathcal{P}_{M-1}} l_t(p),
\]
where $l_t(\cdot)$ is defined in~\eqref{eq:var-reduc-loss} and the expectation is taken over $S^t$.
\end{proposition}

From Proposition~\ref{prop:propert-no-replacement}, we see that $\tilde{g}^t$ is an unbiased stochastic gradient. Furthermore, the variance of $\tilde{g}^t$ is minimized by the same sampling distribution that minimizes the variance reduction loss in~\eqref{eq:var-reduc-loss}. Therefore, it is reasonable to use the sampling distribution generated by Adaptive-OSMD Sampler to design $\tilde{g}^t$.

Following the same simulation setup as in Section~\ref{sec:experiments}, we empirically compare sampling with replacement and sampling without replacement when used together with Adaptive-OSMD sampler. Training loss and cumulative regret are shown in Figure~\ref{fig:ada-osmd-replace-loss-regret}. We observe that using sampling with replacement results in a slightly smaller cumulative regret and a slightly better training loss. However, these differences are not significant.

\begin{figure}[t]
\centering
\includegraphics[width=0.8\textwidth]{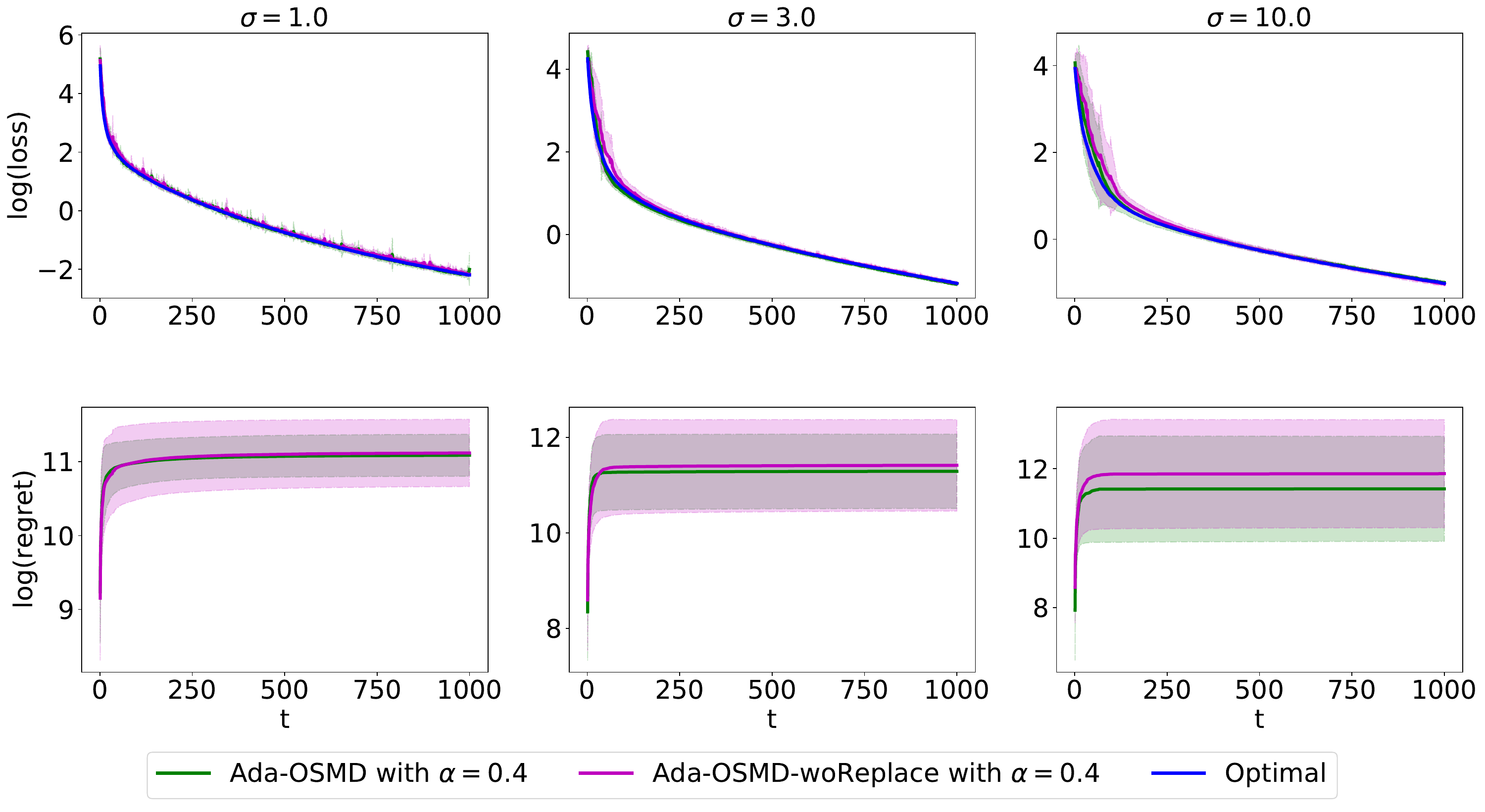}
\caption{The training loss (top row) and cumulative regret (bottom row) are compared for the Adaptive-OSMD Sampler with replacement and without replacement, across $\sigma = 1.0$, $\sigma = 3.0$, and $\sigma = 10.0$. Solid lines represent the mean values, while shaded regions indicate $\text{mean} \pm \text{standard deviation}$ across independent runs.}
\label{fig:ada-osmd-replace-loss-regret}
\end{figure}

\section{Technical proofs}

\subsection{Proof of~\eqref{eq:dyn-hetero-smaller-hetero-fix}}
\label{sec:proof-dyn-hetero-smaller-hetero-fix}

Note that
\begin{align*}
\zeta^2_T (\alpha, \beta) & = \frac{1}{T} \sup_{w^1} \min_{p^1 \in \mathcal{A}} \cdots \sup_{w^T} \min_{p^T \in \mathcal{A}} \, \sum^T_{t=1} V(p^t,w^t) \quad \text{subject to } \text{TV} \left( p^{1:T} \right) \leq \beta \\
& \leq \frac{1}{T} \sup_{w^1} \min_{p^1 \in \mathcal{A}} \cdots \sup_{w^T} \min_{p^T \in \mathcal{A}} \, \sum^T_{t=1} V(p^t,w^t) \quad \text{subject to } \text{TV} \left( p^{1:T} \right) = 0 \\
& = \frac{1}{T} \sup_{w^1} \min_{p \in \mathcal{A}} \sup_{w^2} \cdots \sup_{w^T}  \, \sum^T_{t=1}  V(p,w^t) \\
& \leq \frac{1}{T}  \min_{p \in \mathcal{A}} \sup_{w^1} \cdots \sup_{w^T}  \,  \sum^T_{t=1} V(p,w^t) \\
& \leq \frac{1}{T} \sup_{w^1} \cdots \sup_{w^T}  \,   \sum^T_{t=1} V(p_f,w^t) = \sup_w V(p_f, w) = \zeta^2_{\text{fix}}(\alpha) \leq \zeta^2_{\text{unif}}.
\end{align*}

\subsection{Proposition~\ref{prop:OSMD-solver} and Its Proof}
\label{sec:proof-prop-1}

\begin{proposition}\label{prop:OSMD-solver}
Let 
\[
\tilde{p}^{t+1}_m = p^t_m \exp \left\{ \mathcal{N}\left\{ m \in S^t \right\} \eta_t a^t_m / (K^2 (p^t_m)^3)  \right\}, \qquad m \in [M].
\]
Let $\pi:[M]\mapsto[M]$ be a permutation such that $\tilde{p}^{t+1}_{\pi(1)}\leq\tilde{p}^{t+1}_{\pi(2)}\leq\dots\leq\tilde{p}^{t+1}_{\pi(M)}$. Let $m^t_{\star}$ be the smallest integer $m$ such that
\begin{equation*}
\tilde{p}^{t+1}_{\pi(m)} \left( 1 -  \frac{m-1}{M} \alpha \right) > \frac{\alpha}{M} \sum^M_{j=m} \tilde{p}^{t+1}_{\pi(j)}.
\end{equation*}
Then
\begin{equation*}
\hat{p}^{t+1}_m =
\begin{cases}
\alpha / M & \text{if } \pi(m) < m^t_{\star} \\
\left((1 - ((m^t_{\star}-1) / M) \alpha ) \tilde{p}^{t+1}_{m}\right)/\left(\sum^M_{j=m^t_{\star}} \tilde{p}^{t+1}_{\pi(j)}\right) & \text{otherwise}.
\end{cases}
\end{equation*}
\end{proposition}

\begin{proof}
First, we show that the solution $\hat{p}^{t+1}$ in Step~\ref{alg:solve-mirror-descent} of Algorithm~\ref{alg:OSMD-sampler} can be found as
\begin{align*}
\tilde{p}^{t+1} & =  \arg\min_{ p \in \mathcal{D} } \, \eta_t \langle p,\nabla \hat{l}_t ( \hat{p}^t ; \hat{p}^t)  \rangle + D_{\Phi} \left( p \,\Vert\, \hat{p}^t  \right), \\
\hat{p}^{t+1} & =  \arg\min_{ p \in \mathcal{A} } \, D_{\Phi} \left( p \,\Vert\, \tilde{p}^{t+1}  \right).
\end{align*}
The optimality condition for $\tilde{p}^{t+1}$ implies that
\begin{equation}\label{eq:proof-prop1-1}
\eta_t \nabla \hat{l}_t ( \hat{p}^t ; \hat{p}^t) + \nabla \Phi ( \tilde{p}^{t+1} ) - \nabla \Phi ( \hat{p}^t ) = 0.
\end{equation}
By Lemma~\ref{lemma:convx-opt-cond}, the optimality condition for $\hat{p}^{t+1}$ implies that
\begin{equation*}
\langle p - \hat{p}^{t+1}, \nabla \Phi (\hat{p}^{t+1}) - \nabla \Phi (\tilde{p}^{t+1}) \rangle \geq 0, \quad \text{for all } p \in \mathcal{A}.
\end{equation*}
Combining the last two displays, we have
\begin{equation*}
\langle p - \hat{p}^{t+1}, \eta_t \nabla \hat{l}_t ( \hat{p}^t ; \hat{p}^t) + \nabla \Phi (\hat{p}^{t+1}) - \nabla \Phi (\hat{p}^{t}) \rangle \geq 0, \quad \text{for all } p \in \mathcal{A}.
\end{equation*}
By Lemma~\ref{lemma:convx-opt-cond}, this is the optimality condition for $\hat{p}^{t+1}$ to be the solution in Step~\ref{alg:solve-mirror-descent} of Algorithm~\ref{alg:OSMD-sampler}.

Note that \eqref{eq:proof-prop1-1} implies that
\begin{equation*}
- \frac{\eta_t }{K^2} \cdot \frac{ a^t_m }{ ( p^t_m )^3 } \mathcal{N} \left\{ m \in S^t \right\} + \log (\tilde{p}^{t+1}_m) - \log (\hat{p}^t_m) = 0, \qquad  m \in [M].
\end{equation*}
Therefore, 
\[
\tilde{p}^{t+1}_m = \hat{p}^t_m \exp \left( \frac{\eta_t a^t_m }{ K^2 \left( \hat{p}^t_m \right)^3 } \mathcal{N} \left\{ m \in S^t \right\} \right), \qquad  m \in [M], 
\]
and the final result follows from Lemma~\ref{lemma:mirror-solver-lemma}.

\end{proof}

\subsection{Proof of Theorem~\ref{thm:convg-thm-mini-batch}}
\label{sec:proof-convg-thm-mini-batch}

Our proof follows the similar technique used in the proof of Theorem 2.1 of~\cite{ghadimi2013stochastic} except for the novel technique of construction of a ghost subset that is drawn from $[M]$ from the comparator sampling distribution. Note that the ghost subset is only constructed for theoretical purpose and does not need to be computed in practice.

Given any $\{ w^t \}^T_{t=1} \overset{\Delta}{=}w^{1:T}$, let $q^{1:T}_s$ be the solution of the problem
\begin{equation}
\label{eq:solu-dynamic-game}
\begin{aligned}
\tilde{\zeta}(w^{1:T}) &= \frac{1}{T} \min_{p^1} \ldots \min_{p^T} \sum^T_{t=1} V(p^t,w^t) , \\
& \quad\quad \text{subject to } \text{TV} \left( p^{1:T} \right) \leq \beta \, \text{ and } \, p^t \in \sigma(w^1,\ldots,w^t),
\end{aligned}
\end{equation}
where $\sigma(w^1,\ldots,w^t)$ is the $\sigma$-algebra generated by $\{w^1,\ldots,w^t \}$.
We then have
\begin{align*}
\tilde{\zeta}(w^{1:T}) = \frac{1}{T} \sum^T_{t=1} V(q^t_s(w^t), w^t) \, \text{ and } \, \text{TV}(q^{1:T}_s) \leq \beta.
\end{align*}
Note that by the definition of $\zeta^2_T(\alpha,\beta)$, we have
\begin{equation}
\label{eq:connect-dyn-game-hetero}
\frac{1}{T} \sup_{w^1}\cdots\sup_{w^T}  \sum^T_{t=1} V(q^t_s, w^t) = \zeta^2_T(\alpha,\beta).
\end{equation}


Let $\delta^t=g^t-\nabla F(w^t)$.
Under Assumption~\ref{assump:diff-smooth}, by Lemma~\ref{lemma:lips-upper-bound}, we have
\begin{align*}
F \left( w^{t+1} \right) & \leq F \left( w^t \right) + \left\langle \nabla F \left( w^t \right), w^{t+1} - w^t \right\rangle + \frac{L}{2} \mu^2 \left\Vert g^t \right\Vert^2_2 \\
& = F \left( w^t \right) - \mu \left\langle \nabla F \left( w^t \right), g^t \right\rangle + \frac{L}{2} \mu^2 \left\Vert g^t \right\Vert^2_2 \\
& = F \left( w^t \right) - \mu \left\Vert \nabla F \left( w^t \right) \right\Vert^2_2 - \mu \left\langle \nabla F \left( w^t \right), \delta^t \right\rangle \\
& \quad + \frac{L}{2} \mu^2 \left[  \left\Vert \nabla F \left( w^t \right) \right\Vert^2_2 + 2 \left\langle \nabla F \left( w^t \right), \delta^t \right\rangle +  \left\Vert \delta^t \right\Vert^2_2  \right] \\
& = F \left( w^t \right) - \left( \mu - \frac{L}{2} \mu^2 \right) \left\Vert \nabla F \left( w^t \right) \right\Vert^2_2 - \left( \mu - L \mu^2 \right) \left\langle \nabla F \left( w^t \right), \delta^t \right\rangle + \frac{L}{2} \mu^2 \left\Vert \delta^t \right\Vert^2_2. \numberthis \label{eq:proof-opt-3}
\end{align*}
We use the notation $\mathbb{E}_{S^t}[\cdot]$ to denote the expectation taken with respect to $S^t$.
Note that $\mathbb{E}_{S^t}\left[ g^t \mid w^t, \hat{p}^t \right]=\nabla F(w^t)$, thus we have $\mathbb{E}\left[ \delta^t \mid w^t, \hat{p}^t \right]=0$, and
\begin{equation*}
\mathbb{E}\left[ \left\langle \nabla F \left( w^t \right), \delta^t \right\rangle \right] = \mathbb{E}\left[ \mathbb{E}\left[ \left\langle \nabla F \left( w^t \right), \delta^t \right\rangle \mid w^t, \hat{p}^t \right]\right] = 0. \numberthis \label{eq:proof-opt-2}
\end{equation*}

On the other hand, conditioned on $w^1,\ldots,w^t$, $q^t_s$ is a deterministic sampling distribution.
We can then assume that there is a ghost subset of clients $\tilde{S}^t$ with $\vert \tilde{S}^t \vert=K$, which is drawn from $[M]$ with sampling distribution $q^t_s$.
Recall that $J^t=(1/M)\sum^M_{m=1}g^t_m$.
Besides, we let
\begin{equation*}
\tilde{g}^t \coloneqq \frac{1}{MK} \sum_{m \in \tilde{S}^t} \frac{g^t_m}{q^t_{s,m}}.
\end{equation*}
Then we have
\begin{align*}
\mathbb{E}_{S^t} \left[  \left\Vert \delta^t \right\Vert^2_2 \mid w^1,\ldots,w^t, \hat{p}^t \right] &= \mathbb{E}_{S^t} \left[  \left\Vert g^t - J^t + J^t - \nabla F(w^t) \right\Vert^2_2 \mid w^1,\ldots,w^t, \hat{p}^t \right] \\
& = \mathbb{E}_{S^t} \left[  \left\Vert g^t - J^t \right\Vert^2_2 \mid w^1,\ldots,w^t, \hat{p}^t \right] + \left\Vert J^t - \nabla F(w^t) \right\Vert^2_2 \\
& = l_t \left( \hat{p}^t \right) - \frac{1}{K} \left\Vert J^t \right\Vert^2_2 + \left\Vert J^t - \nabla F(w^t) \right\Vert^2_2 \\
& = l_t \left( q^t_s \right) - \frac{1}{K} \left\Vert J^t \right\Vert^2_2 + \left\Vert J^t - \nabla F(w^t) \right\Vert^2_2 + l_t \left( \hat{p}^t \right) - l_t \left( q^t_s \right) \\
& = \mathbb{E}_{\tilde{S}^t} \left[  \left\Vert \tilde{g}^t - \nabla F(w^t) \right\Vert^2_2 \mid w^1,\ldots,w^t, \hat{p}^t \right] + l_t \left( \hat{p}^t \right) - l_t \left( q^t_s \right).
\end{align*}
Since
\begin{align*}
& \quad \mathbb{E} \left[  \left\Vert \tilde{g}^t - \nabla F(w^t) \right\Vert^2_2 \mid w^1,\ldots,w^t, \hat{p}^t \right]  \\
&= \mathbb{E} \left[  \left\Vert \frac{1}{MK} \sum_{m \in \tilde{S}^t} \frac{g^t_m}{ q^t_{s,m} } - \nabla F(w^t) \right\Vert^2_2 \mid w^1,\ldots,w^t, \hat{p}^t\right] \\
& \leq 2\mathbb{E} \left[  \left\Vert \frac{1}{MK} \sum_{m \in \tilde{S}^t} \frac{g^t_m}{ q^t_{s,m} } - \frac{1}{MK} \sum_{m \in \tilde{S}^t} \frac{\nabla F_m(w^t)}{ q^t_{s,m} } \right\Vert^2_2 \mid w^1,\ldots,w^t, \hat{p}^t\right] \\
& \quad + 2 \mathbb{E} \left[  \left\Vert \frac{1}{MK} \sum_{m \in \tilde{S}^t} \frac{\nabla F_m(w^t)}{ q^t_{s,m} } - \nabla F(w^t) \right\Vert^2_2 \mid w^1,\ldots,w^t, \hat{p}^t \right] \\
& = \frac{2}{M^2 K^2} \mathbb{E} \left[ \sum_{m \in \tilde{S}^t} \mathbb{E} \left[ \frac{ \left\Vert g^t_m - \nabla F_m(w^t) \right\Vert^2_2}{ (q^t_{s,m})^2 } \right] \mid w^1,\ldots,w^t, \hat{p}^t\right] \\
& \quad + \frac{2}{K} \left( \frac{1}{M^2} \sum^M_{m=1} \frac{ \left\Vert \nabla F_m(w^t) \right\Vert^2_2 }{ q^t_{s,m} } - \left\Vert \nabla F(w^t) \right\Vert^2_2 \right) \\
& = \frac{2 \sigma^2}{M^2 K B} \sum^M_{m=1} \frac{1}{ q^t_{s,m} } + \frac{2}{K} V (q^t_s, w^t) \leq \frac{2 \sigma^2}{K B \alpha} + \frac{2}{K} V (q^t_s, w^t),
\end{align*}
where the penultimate line follows that $\mathbb{E} \left[\left\Vert g^t_m - \nabla F_m(w^t) \right\Vert^2_2 \right] \leq \sigma^2 / B$ and the definition of $V(p,w)$, and the last line follows that $q^t_{s,m} \geq \alpha / M$ since $q^t_s \in \mathcal{A}$.

Thus, we have
\begin{align*}
\mathbb{E} \left[  \left\Vert \delta^t \right\Vert^2_2 \mid w^1,\ldots,w^t, \hat{p}^t \right] \leq \frac{2 \sigma^2}{K B\alpha} + \frac{2}{K} V (q^t_s, w^t) + l_t \left( \hat{p}^t \right) - l_t \left( q^t_s \right),
\end{align*}
which implies that
\begin{align*}
\sum^{T}_{t=1} \mathbb{E} \left[  \left\Vert \delta^t \right\Vert^2_2  \right] &\leq \frac{2 T \sigma^2}{K B \alpha} + \frac{2}{K} \mathbb{E} \left[ \sum^T_{t=1} V (q^t_s, w^t) \right] + \mathbb{E} \left[ \sum^{T-1}_{t=0} l_t \left( \hat{p}^t \right) - \sum^{T-1}_{t=0} l_t \left( q^t_s \right) \right] \\
& = \frac{2 T \sigma^2}{K B \alpha} + \frac{2}{K} \mathbb{E} \left[ \sum^T_{t=1} V (q^t_s, w^t) \right] + \text{D-Regret}_T(q^{1:T}_s). \\
& \leq \frac{2 T \sigma^2}{K B \alpha} + \frac{2 T \zeta^2_T(\alpha,\beta)}{K}  + \text{D-Regret}_T(q^{1:T}_s), \numberthis \label{eq:proof-opt-1}
\end{align*}
where the last inequality follows the fact that
\begin{equation*}
\mathbb{E} \left[ \frac{1}{T} \sum^T_{t=1} V (q^t_s, w^t) \right] = \frac{1}{T} \,\mathbb{E}_{w^1,\ldots,w^T} \left[ \sum^T_{t=1} V (q^t_s, w^t) \right]  \leq \frac{1}{T}  \sup_{w^1} \cdots \sup_{w^T} \, \sum^T_{t=1} V (q^t_s, w^t) = \zeta^2_T(\alpha,\beta),
\end{equation*}
and the last equality is by~\eqref{eq:connect-dyn-game-hetero}.
Combine \eqref{eq:proof-opt-3}, \eqref{eq:proof-opt-2} and \eqref{eq:proof-opt-1}, we have
\begin{align*}
& \quad \left( \mu - \frac{L}{2} \mu^2 \right) \sum^{T}_{t=1} \mathbb{E} \left[ \left\Vert \nabla F \left( w^t \right) \right\Vert^2_2 \right] \\
& \leq F(w^1) - F(w^T) + \frac{L}{2} \mu^2 \left( \frac{2 T \sigma^2}{K B \alpha} + \frac{2 T \zeta^2_T(\alpha,\beta)}{K} + \text{D-Regret}_T(q^{1:T}_s) \right)\\
& \leq D^F + \frac{L}{2} \mu^2 \left( \frac{2 T \sigma^2}{K B \alpha} + \frac{2 T \zeta^2_T(\alpha,\beta)}{K} + \text{D-Regret}_T(q^{1:T}_s) \right).
\end{align*}
Since $\mu \leq 1/L$, thus $(\mu - \frac{L}{2} \mu^2)=\mu(1-\frac{L}{2} \mu) \geq \mu/2$, thus
\begin{equation}
\label{eq:proof-opt-new-1}
\frac{1}{T} \sum^{T}_{t=1} \mathbb{E} \left[ \left\Vert \nabla F \left( w^t \right) \right\Vert^2_2 \right] \leq \frac{2 D^F}{T \mu} + L \mu \left( \frac{2 \sigma^2}{K B \alpha} + \frac{2 \zeta^2_T(\alpha,\beta)}{K} + \frac{\text{D-Regret}_T(q^{1:T}_s)}{T}  \right) .
\end{equation}

Next we bound $\text{D-Regret}_T(q^{1:T}_s)$.
Note that
\begin{equation*}
a^t_{\max} = \frac{1}{M^2} \max_{1\leq m \leq M} \Vert g^t_m \Vert^2_2 = \frac{1}{M^2} \max_{1\leq m \leq M} \left\Vert \frac{1}{B} \sum^B_{b=1}\nabla \phi ( w^t ; \xi^{t,b}_{m} ) \right\Vert^2_2 \leq \frac{G^2}{M^2}
\end{equation*}
since $\Vert \nabla \phi (w ; \xi) \Vert_2 \leq G$ for all $w$ and $\xi$, then by Theorem~\ref{thm:upp-bd-general} and the fact that $q^t_s \in \mathcal{A}$ for all $t \in [T]$ and $\text{TV}(q^{1:T}_s) \leq \beta$, we have
\begin{align*}
\text{D-Regret}_T(q^{1:T}_s) & \leq \frac{ \log M  }{ \eta } + \frac{ 2 \log (M / \alpha) }{ \eta } \mathbb{E} \left[\text{TV} \left( q^{1:T} \right) \right] + \frac{\eta M^6}{ 2 K^2 \alpha^6 } \sum^T_{t=1} \mathbb{E} \left[  \left( a^t_{\max}  \right)^2 \right] \\
& \leq \frac{ \log M  }{ \eta } + \frac{ 2 \beta \log (M / \alpha) }{ \eta } + \frac{\eta T M^2 G^4}{ 2 K^2 \alpha^6 }.
\end{align*}
Let
\begin{equation*}
\eta = \frac{ K \alpha^3 }{ M G^2 } \sqrt{ \frac{ 2\log M + 4 \beta \log (M / \alpha) }{ T } },
\end{equation*}
we then have
\begin{equation*}
\text{D-Regret}_T (q^{1:T}_s) \leq 
\frac{ \sqrt{T}  M G^2 }{ K \alpha^3 } \sqrt{ \frac{1}{2}\log M + \beta \log \left( M / \alpha \right) }.
\end{equation*}
Plug the above inequality into~\eqref{eq:proof-opt-new-1}, we have
\begin{align*}
& \quad \frac{1}{T} \sum^{T}_{t=1} \mathbb{E} \left[ \left\Vert \nabla F \left( w^t \right) \right\Vert^2_2 \right] \\ 
& \leq \frac{2 D^F}{T \mu} + L \mu \left( \frac{2 \sigma^2}{K B \alpha} + \frac{2 \zeta^2_T(\alpha,\beta)}{K} + \frac{  M G^2 }{ K \alpha^3 } \sqrt{ \frac{ \frac{1}{2}\log M + \beta \log \left( M / \alpha \right) }{T}  }  \right).
\end{align*}

Finally, let 
\begin{equation}
\label{eq:choice-mu}
\mu=\min \left\{\frac{1}{L}, \frac{1}{\sigma}\sqrt{ \frac{ D^F K B \alpha  }{ L T }}, \frac{1}{\zeta_T(\alpha,\beta)} \sqrt{  \frac{D^F K}{L T} }, \sqrt{ \frac{ D^F K \alpha^3 }{L \sqrt{T} M G^2 \sqrt{\frac{1}{2}\log M + \beta \log \left( M / \alpha \right)} } } \right\},
\end{equation}
then we have
\begin{align*}
& \quad \mathbb{E} \left[ \left\Vert \nabla F \left( w^R \right) \right\Vert^2_2 \right] = \frac{1}{T} \sum^{T}_{t=1} \mathbb{E} \left[ \left\Vert \nabla F \left( w^t \right) \right\Vert^2_2 \right] \\
& \lesssim \frac{D^F}{T} \max \left\{ L, \sigma \sqrt{ \frac{L T}{ D^F K B \alpha } }, \zeta_T(\alpha,\beta) \sqrt{ 
\frac{T L }{D^F K} }, \sqrt{ \frac{ L \sqrt{T} M G^2 \sqrt{\frac{1}{2}\log M + \beta \log \left( M / \alpha \right)}  }{ D^F K \alpha^3 } } \right\} \\
& \quad + \frac{ \sigma \sqrt{D^F L} }{ \sqrt{TKB \alpha } } + \frac{\zeta_T(\alpha,\beta) \sqrt{D^F L }}{\sqrt{T K}} + \frac{ \sqrt{ D^F L} M^{\frac{1}{2}} G }{ T^{\frac{3}{4}} K^{\frac{1}{2}} \alpha^{\frac{3}{2}} } \left( \frac{1}{2}\log M + \beta \log \left( M / \alpha \right)\right)^{\frac{1}{4}} \\
& \lesssim \frac{D^F}{T} \left( L + \sigma \sqrt{ \frac{L T}{ D^F K B \alpha } } + \zeta_T(\alpha,\beta) \sqrt{ 
\frac{T L }{D^F K} } + \sqrt{ \frac{ L \sqrt{T} M G^2 \sqrt{\frac{1}{2}\log M + \beta \log \left( M / \alpha \right)}  }{ D^F K \alpha^3 } }  \right) \\
& \quad + \frac{ \sigma \sqrt{D^F L} }{ \sqrt{TKB \alpha } } + \frac{\zeta_T(\alpha,\beta) \sqrt{D^F L }}{\sqrt{T K}} + \frac{ \sqrt{ D^F L} M^{\frac{1}{2}} G }{ T^{\frac{3}{4}} K^{\frac{1}{2}} \alpha^{\frac{3}{2}} } \left( \frac{1}{2}\log M + \beta \log \left( M / \alpha \right)\right)^{\frac{1}{4}}\\
& \lesssim \frac{D^F L}{T} + \frac{ \sigma \sqrt{D^F L} }{ \sqrt{TKB \alpha } } + \frac{\zeta_T(\alpha,\beta) \sqrt{D^F L }}{\sqrt{T K}} + \frac{ \sqrt{ D^F L} M^{\frac{1}{2}} G }{ T^{\frac{3}{4}} K^{\frac{1}{2}} \alpha^{\frac{3}{2}} } \left( \frac{1}{2}\log M + \beta \log \left( M / \alpha \right)\right)^{\frac{1}{4}}.
\end{align*}

\subsection{Proof of Theorem~\ref{thm:convg-thm-fedavg}}
\label{sec:proof-thm-fedavg}

Similar to the proof in Appendix~\ref{sec:proof-convg-thm-mini-batch}, the key novel technique is the construction of a ghost subset that is drawn from $[M]$ by the comparator sampling distribution. The ghost subset is only constructed for theoretical purpose and does not need to be computed in practice.
Similar to Appendix~\ref{sec:proof-convg-thm-mini-batch}, given any $\{ w^t \}^T_{t=1}$, let $q^{1:T}_s$ be the solution of the problem~\eqref{eq:solu-dynamic-game}.

By Assumption~\ref{assump:diff-smooth} and Lemma~\ref{lemma:lips-upper-bound}, we have
\begin{align*}
F \left( w^{t+1} \right) & \leq F \left( w^t \right) + \left\langle \nabla F \left( w^t \right), w^{t+1} - w^t \right\rangle + \frac{L}{2} \left\Vert w^{t+1} - w^t \right\Vert^2_2,
\end{align*}
where
\begin{equation*}
w^{t+1} - w^t = - \frac{\mu}{MK} \sum_{m \in S^t} \frac{1}{\hat{p}^t_m} g^t_m , \quad g^t_m = \mu_l \sum^{B-1}_{b=0} \nabla \phi ( w^{t,b}_m ; \xi^{t,b}_{m} ).
\end{equation*} 
Let $\mathcal{B}^t_m= \{ \xi^{t,0}_{m},\ldots,\xi^{t,B-1}_{m} \}$ and $\mathbb{E}_{S^t}[\cdot]$ denote the expectation taken with respect to $S^t$. We then have
\begin{align*}
\mathbb{E}_{S^t} \left[ F \left( w^{t+1} \right) \mid w^1, \ldots, w^t, \hat{p}^t \right] & \leq F \left( w^t \right) - \mu \left\langle \nabla F \left( w^t \right), \frac{1}{M} \sum^M_{m=1} g^t_m \right\rangle \\
& \quad + \frac{L \mu^2}{2} \mathbb{E}_{S^t} \left[ \left\Vert \frac{1}{K} \sum_{m \in S^t} \frac{1}{M \hat{p}^t_m} g^t_m \right\Vert^2_2 \mid w^1, \ldots, w^t, \hat{p}^t \right]. 
\end{align*}
Recall that $J^t=(1/M)\sum^M_{m=1}g^t_m$, then we have
\begin{align*}
\mathbb{E}_{S^t} \left[ F \left( w^{t+1} \right) \right] & \leq F \left( w^t \right) - \mu \left\langle \nabla F \left( w^t \right), J^t \right\rangle  \\
& \quad + \frac{L \mu^2}{2 } \mathbb{E}_{S^t} \left[ \left\Vert \frac{1}{K} \sum_{m \in S^t} \frac{1}{M \hat{p}^t_m} g^t_m - J^t + J^t\right\Vert^2_2 \mid w^1, \ldots, w^t, \hat{p}^t \right] \\
& = F \left( w^t \right) - \mu \left\langle \nabla F \left( w^t \right), J^t \right\rangle + \frac{L \mu^2}{2} \left\Vert J^t \right\Vert^2_2\\
& \quad + \frac{L \mu^2}{2 } \mathbb{E}_{S^t} \left[ \left\Vert \frac{1}{K} \sum_{m \in S^t} \frac{1}{M\hat{p}^t_m} g^t_m - J^t \right\Vert^2_2 \mid w^1, \ldots, w^t, \hat{p}^t \right].
\end{align*}

Recall that in Algorithm~\ref{alg:fed-avg-osmd}, we let $a^t_m= \frac{ \Vert g^t_m \Vert^2 }{B  (M \mu_l)^2}$, thus we have
\begin{align*}
\mathbb{E}_{S^t} \left[ \left\Vert \frac{1}{K} \sum_{m \in S^t} \frac{1}{M\hat{p}^t_m} g^t_m - J^t \right\Vert^2_2 \mid w^1, \ldots, w^t, \hat{p}^t \right] & = \frac{1}{K} \sum^M_{m=1} \frac{\left\Vert g^t_m \right\Vert^2_2}{M^2 \hat{p}^t_m  } - \frac{1}{K} \left\Vert J^t \right\Vert^2_2 \\
& = B \mu^2_l \frac{1}{K} \sum^M_{m=1} \frac{\left\Vert g^t_m \right\Vert^2_2}{ B (\mu_l M)^2 \hat{p}^t_m  } - \frac{1}{K} \left\Vert J^t \right\Vert^2_2 \\
& = B \mu^2_l \frac{1}{K} \sum^M_{m=1} \frac{a^t_m}{  \hat{p}^t_m  } - \frac{1}{K} \left\Vert J^t \right\Vert^2_2 \\
& = B \mu^2_l l_t \left( \hat{p}^t \right) - \frac{1}{K} \left\Vert J^t \right\Vert^2_2, \numberthis \label{eq:proof-thm-fedavg-0}
\end{align*}
which then implies that
\begin{equation}
\label{eq:proof-thm-fedavg-1}
\begin{aligned}
& \quad \mathbb{E}_{S^t} \left[ F \left( w^{t+1} \right) \mid w^1, \ldots, w^t, \hat{p}^t \right] \\
& \leq  F \left( w^t \right) - \mu \left\langle \nabla F \left( w^t \right), J^t \right\rangle + \frac{L \mu^2}{2 } \left\Vert J^t \right\Vert^2_2 + \frac{ L \mu^2 }{2} \left( B \mu^2_l l_t(\hat{p}^t) - \frac{1}{K} \left\Vert J^t  \right\Vert^2_2 \right),
\end{aligned}
\end{equation}
where $l_t(\cdot)$ is the variance reduction loss defined in~\eqref{eq:var-reduc-loss}.

Conditioned on $w^1,\ldots,w^t$, $q^t_s$ is a deterministic sampling distribution.
We assume that there is a ghost subset of clients $\tilde{S}^t$ with $\vert \tilde{S}^t \vert=K$, which is drawn from $[M]$ with replacement by sampling distribution $q^t_s$.
Besides, we let
\begin{equation*}
\tilde{g}^t \coloneqq \frac{1}{K} \sum_{m \in \tilde{S}^t} \frac{g^t_m}{M q^t_{s.m}},
\end{equation*}
then similar to~\eqref{eq:proof-thm-fedavg-0}, we have
\begin{equation*}
\mathbb{E}_{\tilde{S}^t} \left[ \left\Vert \tilde{g}^t - J^t   \right\Vert^2_2  \mid w^1, \ldots, w^t, \hat{p}^t \right] = B \mu^2_l l_t(q^t_s) - \frac{1}{K} \left\Vert J^t  \right\Vert^2_2 .
\end{equation*}
Combine the above equation with~\eqref{eq:proof-thm-fedavg-1}, we have
\begin{align*}
& \mathbb{E}_{S^t} \left[ F \left( w^{t+1} \right) \mid w^1, \ldots, w^t, \hat{p}^t \right] \\
& \leq F \left( w^t \right) - \mu \left\langle \nabla F \left( w^t \right), J^t \right\rangle + \frac{L \mu^2}{2 } \left\Vert J^t \right\Vert^2_2  + \frac{ L \mu^2 }{2} \left(  B \mu^2_l l_t(q^t_s) - \frac{1}{K} \left\Vert J^t  \right\Vert^2_2 \right) \\
& \quad + \frac{ B L \mu^2 \mu^2_l }{2} \left( l_t(\hat{p}^t) - l_t(q^t_s) \right) \\
& = F \left( w^t \right) - \mu \left\langle \nabla F \left( w^t \right), J^t \right\rangle + \frac{L \mu^2}{2 } \left\Vert J^t \right\Vert^2_2 +  \frac{B L \mu^2 \mu^2_l }{2} \left( l_t(\hat{p}^t) - l_t(q^t_s) \right) \\
& \quad + \frac{L \mu^2}{2 } \mathbb{E}_{\tilde{S}^t} \left[ \left\Vert \tilde{g}^t - J^t  \right\Vert^2_2  \mid w^1, \ldots, w^t, \hat{p}^t \right] \\
&= F \left( w^t \right) - \mu \left\langle \nabla F \left( w^t \right), J^t \right\rangle + \frac{L \mu^2}{2 } \mathbb{E}_{\tilde{S}^t} \left[ \left\Vert \tilde{g}^t \right\Vert^2_2 \mid w^1, \ldots, w^t, \hat{p}^t \right] +  \frac{B L \mu^2 \mu^2_l }{2} \left( l_t(\hat{p}^t) - l_t(q^t_s) \right).
\end{align*}
Let $\mathbb{E}_t[\cdot]$ denote $\mathbb{E}[ \cdot \mid w^1, \ldots, w^t, \hat{p}^t]$. We thus have
\begin{equation}
\label{eq:proof-thm-fedavg-5}
\begin{aligned}
\mathbb{E}_t \left[ F \left( w^{t+1} \right) \right] & \leq F \left( w^t \right) - \frac{\mu\mu_l}{M} \mathbb{E}_t \left[  \left\langle \nabla F \left( w^t \right), \sum^M_{m=1} \sum^{B-1}_{b=0} \nabla F_m(w^{t,b}_m) \right\rangle \right] \\
& \quad + \frac{L \mu^2}{2 } \mathbb{E}_t \left[ \left\Vert \tilde{g}^t \right\Vert^2_2  \right] +  \frac{B L \mu^2 \mu^2_l }{2} \mathbb{E}_t \left[l_t(\hat{p}^t) - l_t(q^t_s) \right].
\end{aligned}
\end{equation}

Note that for any $u,v \in \mathbb{R}^d$, we have $-\langle u,v \rangle = -\frac{1}{2} \Vert u \Vert^2 + \frac{1}{2} \Vert u - v \Vert - \frac{1}{2} \Vert v \Vert^2 \leq -\frac{1}{2} \Vert u \Vert^2 + \frac{1}{2} \Vert u - v \Vert$. Thus we have
\begin{align*}
&\quad \mathbb{E}_t \left[ - \frac{\mu \mu_l}{M}\sum^M_{m=1} \sum^{B-1}_{b=0} \left\langle \nabla F \left( w^t \right),  \nabla F_m(w^{t,b}_m) \right\rangle \right] \\
& = \mathbb{E}_t \left[ - \mu \mu_l \sum^{B-1}_{b=0} \left\langle \nabla F \left( w^t \right),  \frac{1}{M}\sum^M_{m=1}\nabla F_m(w^{t,b}_m) \right\rangle \right] \\
& \leq -\frac{\mu \mu_l B}{2}  \left\Vert \nabla F \left( w^t \right) \right\Vert^2_2 + \frac{\mu \mu_l}{2} \sum^{B-1}_{b=0} \left\Vert \nabla F \left( w^t \right) -\frac{1}{M}\sum^M_{m=1} \nabla F_m(w^{t,b}_m) \right\Vert^2_2 \\
& = -\frac{\mu \mu_l B}{2}  \left\Vert \nabla F \left( w^t \right) \right\Vert^2_2 + \frac{\mu \mu_l}{2} \sum^{B-1}_{b=0} \left\Vert  \frac{1}{M}\sum^M_{m=1} \left( \nabla F_m \left( w^t \right) - \nabla F_m(w^{t,b}_m) \right) \right\Vert^2_2 \\
& \leq -\frac{\mu \mu_l B}{2}  \left\Vert \nabla F \left( w^t \right) \right\Vert^2_2 + \frac{\mu \mu_l}{2M} \mathbb{E}_t \left[ \sum^M_{m=1} \sum^{B-1}_{b=0} \left\Vert \nabla F_m \left( w^t \right) - \nabla F_m(w^{t,b}_m) \right\Vert^2_2 \right] \\
& \leq -\frac{\mu \mu_l B}{2}  \left\Vert \nabla F \left( w^t \right) \right\Vert^2_2 + \frac{\mu \mu_l L^2}{2M} \mathbb{E}_t \left[ \sum^M_{m=1} \sum^{B-1}_{b=0} \left\Vert w^t - w^{t,b}_m \right\Vert^2_2 \right]. \numberthis \label{eq:proof-thm-fedavg-6}
\end{align*}

Besides, we also have
\begin{align*}
\mathbb{E}_t \left[ \left\Vert \tilde{g}^t \right\Vert^2_2 \right] &= \mathbb{E}_t \left[ \left\Vert \frac{1}{K} \sum_{m \in \tilde{S}^t} \frac{g^t_m}{M q^t_{s,m}} \right\Vert^2_2 \right] = \mu^2_l \mathbb{E}_t \left[ \left\Vert \frac{1}{K} \sum_{m \in \tilde{S}^t} \frac{1}{M q^t_{s,m}} \sum^{B-1}_{b=0} \nabla \phi ( w^{t,b}_m ; \xi^{t,b}_{m} ) \right\Vert^2_2 \right] \\
& = \mu^2_l \underbrace{\mathbb{E}_t \left[ \left\Vert \frac{1}{K} \sum_{m \in \tilde{S}^t} \frac{1}{M q^t_{s,m}} \sum^{B-1}_{b=0} \left( \nabla \phi ( w^{t,b}_m ; \xi^{t,b}_{m} ) - \nabla F_m (w^{t,b}_m \right) \right\Vert^2_2 \right]}_\text{$I_1$} \\
& \quad + \mu^2_l \underbrace{\mathbb{E}_t \left[ \left\Vert \frac{1}{K} \sum_{m \in \tilde{S}^t} \frac{1}{M q^t_{s,m}} \sum^{B-1}_{b=0} \nabla F_m (w^{t,b}_m) \right\Vert^2_2 \right]}_\text{$I_2$}. \numberthis \label{eq:proof-thm-fedavg-2}
\end{align*}

To bound $I_1$, note that
\begin{align*}
I_1 & = \frac{1}{K} \sum^M_{m=1} \frac{1}{M^2 q^t_{s,m} }  \mathbb{E}_t \left[ \left\Vert \sum^{B-1}_{b=0} \nabla \phi ( w^{t,b}_m ; \xi^{t,b}_{m} ) - \nabla F_m (w^{t,b}) \right\Vert^2_2 \right] \\
& = \frac{1}{K} \sum^M_{m=1} \frac{1}{M^2 q^t_{s,m} }  \sum^{B-1}_{b=0}\mathbb{E}_t \left[  \left\Vert \nabla \phi ( w^{t,b}_m ; \xi^{t,b}_{m} ) - \nabla F_m (w^{t,b}) \right\Vert^2_2 \right] \\
& \leq \frac{\sigma^2 B}{M^2K} \sum^M_{m=1} \frac{1}{ q^t_{s,m} }.
\end{align*}
Since we have $q^t_{s,m} \geq \alpha/M$, we then have $I_1 \leq \frac{\sigma^2 B}{K \alpha}$.
We then give an upper bound on $I_2$. Note that
\begin{align*}
I_2 & \leq 2 \mathbb{E}_t \left[  \sum^M_{m=1} \frac{1}{M^2 q^t_{s,m}}  \left\Vert \sum^{B-1}_{b=0} \left( \nabla F_m (w^{t,b}_m) - \nabla F_m(w^t)  \right)  \right\Vert^2_2  \right] \\
& \quad + 2 B^2 \mathbb{E}_t \left[ \left\Vert \frac{1}{K} \sum_{m \in \tilde{S}^t} \frac{1}{M q^t_{s,m}} \nabla F_m(w^t) - \nabla F (w^t)  \right\Vert^2_2 \right] +2 B^2 \left\Vert \nabla F (w^t) \right\Vert^2_2 \\
& \leq 2 B \mathbb{E}_t \left[  \sum^M_{m=1} \frac{1}{M^2 q^t_{s,m}}   \sum^{B-1}_{b=0} \left\Vert  \nabla F_m (w^{t,b}_m) - \nabla F_m(w^t)   \right\Vert^2_2  \right] + 2 B^2 \left\Vert \nabla F (w^t) \right\Vert^2_2 \\
& \quad + \frac{2 B^2}{K} \left( \frac{1}{M^2} \sum^M_{m=1} \frac{1}{q^t_{s,m}} \left\Vert \nabla F_m(w^t)  \right\Vert^2_2 -  \left\Vert \nabla F(w^t)  \right\Vert^2_2 \right) \\
& \leq 2 B L^2 \mathbb{E}_t \left[  \sum^M_{m=1} \frac{1}{M^2 q^t_{s,m}}   \sum^{B-1}_{b=0} \left\Vert  w^{t,b}_m - w^t   \right\Vert^2_2  \right] + 2 B^2 \left\Vert \nabla F (w^t) \right\Vert^2_2 + \frac{2 B^2 }{K} \mathbb{E}_t \left[ V(q^t_s, w^t) \right] , \numberthis \label{eq:proof-thm-fedavg-4}
\end{align*}
where the first three inequalities follow Jensen's inequality and Lemma~\ref{lemma:relax-triangle-ineq}, the fourth inequality follows the definition of $V(p,w)$ in~\eqref{eq:Vpw}, and the final inequality follows Assumption~\ref{assump:diff-smooth}.
Finally, since $q^t_{s,m} \geq \alpha/M$, we have
\begin{equation}
\label{eq:proof-thm-fedavg2-4}
I_2 \leq \frac{2 B L^2}{M\alpha} \mathbb{E}_t \left[  \sum^M_{m=1} \sum^{B-1}_{b=0} \left\Vert  w^{t,b}_m - w^t   \right\Vert^2_2  \right] + 2 B^2 \left\Vert \nabla F (w^t) \right\Vert^2_2 + \frac{2 B^2 }{K} \mathbb{E}_t \left[ V(q^t_s, w^t) \right].
\end{equation}

Combine \eqref{eq:proof-thm-fedavg-2}---\eqref{eq:proof-thm-fedavg2-4}, we then have
\begin{align*}
\frac{L \mu^2 }{2} \mathbb{E}_t \left[ \left\Vert \tilde{g}^t \right\Vert^2_2 \right] & \leq \frac{\mu^2 \mu^2_l  L \sigma^2 B}{2 K \alpha} + \frac{\mu^2 \mu^2_l  L^3 B}{M \alpha}  \mathbb{E}_t \left[  \sum^M_{m=1}  \sum^{B-1}_{b=0} \left\Vert  w^{t,b}_m - w^t   \right\Vert^2_2  \right] \\
& \quad + \mu^2 \mu^2_l  B^2 L \left\Vert \nabla F(w^t)  \right\Vert^2_2 + \frac{\mu^2 \mu^2_l  L B^2}{K} \mathbb{E}_t \left[ V(q^t_s, w^t) \right].
\end{align*}

Combine the above equation with \eqref{eq:proof-thm-fedavg-5} and \eqref{eq:proof-thm-fedavg-6}, we have

\begin{align*}
\mathbb{E}_t \left[ F \left( w^{t+1} \right) \right] & \leq F \left( w^t \right) - \frac{\mu \mu_l B}{2} \left( 1 - 2 \mu \mu_l B L \right) \left\Vert \nabla F(w^t)  \right\Vert^2_2 \\
& \quad + \frac{\mu \mu_l L^2}{M} \left( \frac{1}{2} + \frac{\mu \mu_l L B}{\alpha} \right) \mathbb{E}_t \left[  \sum^M_{m=1} \sum^{B-1}_{b=0} \left\Vert  w^{t,b}_m - w^t   \right\Vert^2_2  \right] \\
& \quad + \frac{\mu^2 \mu^2_l L \sigma^2 B}{2 K \alpha} + \frac{\mu^2 \mu^2_l L B^2}{K} \mathbb{E}_t \left[ V(q^t_s, w^t) \right] + \frac{B L \mu^2 \mu^2_l }{2} \mathbb{E}_t \left[l_t(\hat{p}^t) - l_t(q^t_s) \right].
\end{align*}
By letting $\tilde{\mu}=\mu \mu_l$ and $\tilde{\mu} \leq \frac{1}{4 B L}$, we have
\begin{equation}
\label{eq:proof-thm-fedavg-8}
\begin{aligned}
\mathbb{E}_t \left[ F \left( w^{t+1} \right) \right] & \leq F \left( w^t \right) - \frac{\tilde{\mu} B}{4} \left\Vert \nabla F(w^t)  \right\Vert^2_2 \\
& \quad + \tilde{\mu} L^2 \left( \frac{1}{2} + \frac{1}{4\alpha} \right) \mathbb{E}_t \left[ \frac{1}{M} \sum^M_{m=1} \sum^{B-1}_{b=0} \left\Vert  w^{t,b}_m - w^t   \right\Vert^2_2  \right] \\
& \quad + \frac{\tilde{\mu}^2 B L \sigma^2}{2 K \alpha} + \frac{\tilde{\mu}^2 B^2 L}{K} \mathbb{E}_t \left[ V(q^t_s, w^t) \right] + \frac{B L \mu^2 \mu^2_l}{2} \mathbb{E}_t \left[l_t(\hat{p}^t) - l_t(q^t_s) \right].
\end{aligned}
\end{equation}

Following Lemma 8 of~\cite{karimireddy2020scaffold}, we then show the following claim
\begin{equation}
\label{eq:proof-thm-fedavg-7}
\mathbb{E}_t \left[ \frac{1}{M} \sum^M_{m=1} \sum^{B-1}_{b=0} \left\Vert  w^{t,b}_m - w^t   \right\Vert^2_2  \right] \leq 8 B^3 \mu^2_l \zeta^2_{\text{unif}} + 8 B^3 \mu^2_l \left\Vert \nabla F(w^t)  \right\Vert^2_2 + 4 B^2 \mu^2_l \sigma^2.
\end{equation}
To show the above inequality, first note that when $b=0$, we have $w^{t,0}_m=w^t$, thus we have $\mathbb{E}_t [ \Vert  w^{t,0}_m - w^t   \Vert^2_2 ]=0$. Besides, for $1 \leq b \leq B - 1$, we have
\begin{align*}
& \quad \mathbb{E}_t \left[ \left\Vert w^{t,b}_m - w^t \right\Vert^2_2 \right] \\
& = \mathbb{E}_t \left[ \left\Vert w^{t,b-1}_m - \mu_l \nabla \phi \left( w^{t,b-1}_m ; \xi^{t,b-1}_m \right) - w^t \right\Vert^2_2 \right] \\
& = \mathbb{E}_t \left[ \left\Vert w^{t,b-1}_m - \mu_l \nabla F_m \left( w^{t,b-1}_m  \right) - w^t \right\Vert^2_2 \right] \\
& \quad + \mu^2_l \mathbb{E}_t \left[ \left\Vert  \nabla \phi \left( w^{t,b-1}_m ; \xi^{t,b-1}_m \right) - \nabla F_m \left( w^{t,b-1}_m  \right) \right\Vert^2_2 \right] \\
& \leq \left( 1 + \frac{1}{B-1} \right) \mathbb{E}_t \left[ \left\Vert w^{t,b-1}_m  - w^t \right\Vert^2_2 \right] + B \mu^2_l \mathbb{E}_t \left[ \left\Vert  \nabla F_m \left( w^{t,b-1}_m  \right) \right\Vert^2_2 \right] + \mu^2_l \sigma^2 \\
& \leq \left( 1 + \frac{1}{B-1} \right) \mathbb{E}_t \left[ \left\Vert w^{t,b-1}_m  - w^t \right\Vert^2_2 \right] + 2 B \mu^2_l \left\Vert \nabla F_m \left( w^{t,b-1}_m  \right) - \nabla F_m (w^t) \right\Vert^2_2 \\
& \quad + 2 B \mu^2_l \left\Vert  \nabla F_m (w^t) \right\Vert^2_2 + \mu^2_l \sigma^2 \\
& \leq \left( 1 + \frac{1}{B-1} + 2 B \mu^2_l L^2 \right) \mathbb{E}_t \left[ \left\Vert w^{t,b-1}_m  - w^t \right\Vert^2_2 \right] + 2 B \mu^2_l \left\Vert  \nabla F_m (w^t) \right\Vert^2_2 + \mu^2_l \sigma^2,
\end{align*}
where the first inequality follows Assumption~\ref{assump:local-sg} and Lemma~\ref{lemma:relax-triangle-ineq}, the second inequality follows Jensen's inequality, the third inequality follows Assumption~\ref{assump:diff-smooth}. Let $\mu_l \leq \frac{1}{\sqrt{2} B L}$, then we have $2 B \mu^2_l L^2 \leq \frac{1}{B} \leq \frac{1}{B-1}$ and
\begin{equation*}
\mathbb{E}_t \left[ \left\Vert w^{t,b}_m - w^t \right\Vert^2_2 \right] \leq \left( 1 + \frac{2}{B-1}  \right) \mathbb{E}_t \left[ \left\Vert w^{t,b-1}_m  - w^t \right\Vert^2_2 \right] + 2 B \mu^2_l \left\Vert  \nabla F_m (w^t) \right\Vert^2_2 + \mu^2_l \sigma^2.
\end{equation*}
By induction, we then have
\begin{equation*}
\mathbb{E}_t \left[ \left\Vert w^{t,b}_m - w^t \right\Vert^2_2 \right] \leq \left( 2 B \mu^2_l \left\Vert  \nabla F_m (w^t) \right\Vert^2_2 + \mu^2_l \sigma^2 \right) \sum^{b-1}_{\tau=0} \left( 1 + \frac{2}{B-1} \right)^\tau
\end{equation*}
Since
\begin{align*}
\sum^{b-1}_{\tau=0} \left( 1 +  \frac{2}{B-1} \right)^\tau &= \frac{\left( 1 + \frac{2}{B-1} \right)^{b} - 1  }{  \left( 1 + \frac{2}{B-1} \right) - 1  } = \frac{B-1}{2} \left\{ \left( 1 + \frac{2}{B-1} \right)^{b} - 1  \right\} \\
& \leq \frac{B-1}{2} \left\{ \left( 1 + \frac{2}{B-1} \right)^{B-1} - 1  \right\} \leq \frac{B-1}{2} ( e^2 - 1) \leq 4 (B-1) \leq 4 B.
\end{align*}
thus we have
\begin{equation*}
\mathbb{E}_t \left[ \left\Vert w^{t,b}_m - w^t \right\Vert^2_2 \right] \leq 8 B^2 \mu^2_l \left\Vert  \nabla F_m (w^t) \right\Vert^2_2 + 4 B \mu^2_l \sigma^2,
\end{equation*}
which then implies that
\begin{align*}
& \quad \frac{1}{M} \sum^M_{m=1} \sum^{B-1}_{b=0} \mathbb{E}_t \left[ \left\Vert w^{t,b}_m - w^t \right\Vert^2_2 \right] \\
& \leq 8 B^3 \mu^2_l \frac{1}{M} \sum^M_{m=1}\left\Vert  \nabla F_m (w^t) \right\Vert^2_2 + 4 B^2 \mu^2_l \sigma^2 \\
& \leq 8 B^3 \mu^2_l \left(\frac{1}{M} \sum^M_{m=1}\left\Vert  \nabla F_m (w^t) - \nabla F(w^t) \right\Vert^2_2 +  \left\Vert  \nabla F(w^t) \right\Vert^2_2 \right) + 4 B^2 \mu^2_l \sigma^2 \\
& \leq 8 B^3 \mu^2_l \zeta^2_{\text{unif}} + 8 B^3 \mu^2_l \left\Vert  \nabla F(w^t) \right\Vert^2_2 + 4 B^2 \mu^2_l \sigma^2,
\end{align*}
where the last inequality follows the definition of $\zeta^2_{\text{unif}}$. We thus have proved~\eqref{eq:proof-thm-fedavg-7}. 

Besides, we have
\begin{align*}
& \quad \tilde{\mu} L^2 \left( \frac{1}{2} + \frac{1}{4\alpha} \right) \mathbb{E}_t \left[ \frac{1}{M} \sum^M_{m=1} \sum^{B-1}_{b=0} \left\Vert  w^{t,b}_m - w^t   \right\Vert^2_2  \right] \\
& \leq \left(4 + \frac{2}{\alpha} \right) \tilde{\mu} \mu^2_l B^3 L^2 \zeta^2_{\text{unif}} + \left(4 + \frac{2}{\alpha} \right) \tilde{\mu} \mu^2_l B^3 L^2 \left\Vert \nabla F(w^t) \right\Vert^2_2 + \left( 2 + \frac{1}{ \alpha} \right) \tilde{\mu} \mu^2_l B^2 L^2 \sigma^2.
\end{align*}
Combine the above result with~\eqref{eq:proof-thm-fedavg-8}, we have
\begin{align*}
\mathbb{E}_t \left[ F \left( w^{t+1} \right) \right] & \leq F \left( w^t \right) - \frac{\tilde{\mu} B}{4} \left( 1 - 4 \left( 4 + \frac{2}{\alpha} \right) \mu^2_l B^2 L^2 \right) \left\Vert \nabla F(w^t)  \right\Vert^2_2 \\
& \quad + \left( 4 + \frac{2}{\alpha} \right) \tilde{\mu} \mu^2_l B^2 L^2 \left( B \zeta^2_{\text{unif}} + \frac{\sigma^2}{2} \right) + \frac{ \tilde{\mu}^2 B L }{ K } \left( \frac{ \sigma^2 }{2 \alpha} + B \mathbb{E}_t \left[ V(q^t_s, w^t) \right] \right) \\
& \quad + \frac{B L \mu^2 \mu^2_l}{2} \mathbb{E}_t \left[l_t(\hat{p}^t) - l_t(q^t_s) \right].
\end{align*}
Let $\mu_l \leq \frac{1}{4 B L} \sqrt{\frac{1}{2+1/\alpha}}$, we then have
\begin{align*}
\mathbb{E}_t \left[ F \left( w^{t+1} \right) \right] & \leq F \left( w^t \right) - \frac{\tilde{\mu} B}{8} \left\Vert \nabla F(w^t)  \right\Vert^2_2  + \left( 4 + \frac{2}{\alpha} \right) \tilde{\mu} \mu^2_l B^2 L^2 \left( B \zeta^2_{\text{unif}}  + \frac{\sigma^2}{2} \right) \\
& \quad + \frac{ \tilde{\mu}^2 B L }{ K } \left( \frac{ \sigma^2 }{2 \alpha} + B \mathbb{E}_t \left[ V(q^t_s, w^t) \right]  \right)  + \frac{ B L \tilde{\mu}^2 }{2} \mathbb{E}_t \left[l_t(\hat{p}^t) - l_t(q^t_s) \right],
\end{align*}
which then implies that
\begin{align*}
\left\Vert \nabla F(w^t)  \right\Vert^2_2 & \leq \frac{ 8 }{\tilde{\mu} B } \left( F \left( w^t \right) - \mathbb{E}_t \left[ F \left( w^{t+1} \right) \right] \right) + 8 \left( 4 + \frac{2}{\alpha}  \right) \mu^2_l B^2 L^2 \left( \zeta^2_{\text{unif}}  + \frac{\sigma^2}{2 B} \right) \\
&\quad + \frac{ 8 \tilde{\mu} B L }{ K } \left( \frac{\sigma^2}{ 2 B \alpha } + \mathbb{E}_t \left[ V(q^t_s, w^t) \right]  \right) + 4 L \tilde{\mu} \mathbb{E}_t \left[l_t(\hat{p}^t) - l_t(q^t_s) \right].
\end{align*}
Taking full expectation on both sides, summing over $t=0$ to $t=T-1$ and taking average, we then have
\begin{align*}
\frac{1}{T} \sum^{T}_{t=1} \mathbb{E} \left[ \left\Vert \nabla F(w^t)  \right\Vert^2_2 \right] & \leq \frac{ 8 }{\tilde{\mu} B } \left( F (w^0) - \mathbb{E} \left[ F(w^T) \right] \right) + 8 \left( 4 + \frac{2}{\alpha}  \right) \mu^2_l B^2 L^2 \left( \zeta^2_{\text{unif}}  + \frac{\sigma^2}{2 B} \right) \\
& \quad + \frac{ 8 \tilde{\mu} B L }{ K } \left( \frac{\sigma^2}{ 2 B \alpha } + \frac{1}{T} \sum^T_{t=1} \mathbb{E}_t \left[ V(q^t_s, w^t) \right]  \right) + 4 L \tilde{\mu} \times \frac{ \text{D-Regret}_T \left( q^{1:T} \right) }{T} \\
& \leq \frac{ 8 }{\tilde{\mu} B } \left( F (w^0) - F^{\star} \right) + 8 \left( 4 + \frac{2}{\alpha}  \right) \mu^2_l B^2 L^2 \left( \zeta^2_{\text{unif}}  + \frac{\sigma^2}{2 B} \right) \\
& \quad + \frac{ 8 \tilde{\mu} B L }{ K } \left( \frac{\sigma^2}{ 2 B \alpha } + \frac{1}{T} \sum^T_{t=1} \mathbb{E}_t \left[ V(q^t_s, w^t) \right]  \right) + 4 L \tilde{\mu} \times \frac{ \text{D-Regret}_T \left( q^{1:T} \right) }{T} \\
& \leq \frac{ 8 }{\tilde{\mu} B } \left( F (w^0) - F^{\star} \right) + 8 \left( 4 + \frac{2}{\alpha}  \right) \mu^2_l B^2 L^2 \left( \zeta^2_{\text{unif}}  + \frac{\sigma^2}{2 B} \right) \\
& \quad + \frac{ 8 \tilde{\mu} B L }{ K } \left( \frac{\sigma^2}{ 2 B \alpha } + \zeta^2_T (\alpha,\beta)  \right) + 4 L \tilde{\mu} \times \frac{ \text{D-Regret}_T \left( q^{1:T} \right) }{T},
\end{align*}
where the last inequality follows the fact that
\begin{align*}
\mathbb{E} \left[ \frac{1}{T} \sum^T_{t=1} V (q^t_s, w^t) \right] = \frac{1}{T} \,\mathbb{E}_{w^1,\ldots,w^T} \left[ \sum^T_{t=1} V (q^t_s , w^t) \right] \leq \frac{1}{T}  \sup_{w^1} \cdots \sup_{w^T} \, \sum^T_{t=1} V (q^t_s, w^t) = \zeta^2_T(\alpha,\beta),
\end{align*}
and the last equality is by~\eqref{eq:connect-dyn-game-hetero}.

In summary, when $\tilde{\mu} \leq \frac{1}{4 B L}$ and $\mu_l \leq \frac{1}{4 B L} \sqrt{\frac{1}{2+1/\alpha}}$, we have
\begin{align*}
\mathbb{E} \left[ \left\Vert \nabla F(w^R)  \right\Vert^2_2 \right] & =
\frac{1}{T} \sum^{T}_{t=1} \mathbb{E} \left[ \left\Vert \nabla F(w^t)  \right\Vert^2_2 \right] \\
& \leq \frac{ 8 }{\tilde{\mu} B T } \left( F (w^0) - F^{\star} \right) + 8 \left( 4 + \frac{2}{\alpha}  \right) \mu^2_l B^2 L^2 \left( \zeta^2_{\text{unif}}  + \frac{\sigma^2}{2 B} \right) \\
& \quad +  4 \tilde{\mu} B L  \left( \frac{2 \zeta^2_T (\alpha, \beta)}{K} + \frac{\sigma^2}{K B \alpha } + \frac{ \text{D-Regret}_T \left( q^{1:T}_s \right) }{B T} \right).
\end{align*}
By letting $\mu \geq 1$, $\tilde{\mu}\leq \frac{1}{4 B L} \sqrt{\frac{1}{1+1/(2\alpha)}}$ and recall that $D^F=F (w^1) - F^{\star}$, we then have
\begin{equation}
\label{eq:proof-thm-fedavg-new-1}
\begin{aligned}
\mathbb{E} \left[ \left\Vert \nabla F(w^R)  \right\Vert^2_2 \right] & \leq \frac{ 8 D^F}{\tilde{\mu} B T } + 8 \left( 4 + \frac{2}{\alpha}  \right) \tilde{\mu}^2 B^2 L^2 \left( \zeta^2_{\text{unif}}  + \frac{\sigma^2}{2 B} \right) \\
& \quad +  4 \tilde{\mu} B L  \left( \frac{2 \zeta^2_T (\alpha, \beta)}{K} + \frac{\sigma^2}{K B \alpha } + \frac{ \text{D-Regret}_T \left( q^{1:T}_s \right) }{B T} \right).
\end{aligned}
\end{equation}

We then turn to bound the regret, note that
\begin{equation*}
g^t_m = \mu_l \sum^{B-1}_{b=0} \nabla \phi ( w^{t,b}_m ; \xi^{t,b}_{m} ),
\end{equation*}
thus by Assumption~\ref{assump:local-sg}, we have
\begin{equation*}
\Vert g^t_m \Vert^2_2 \leq B^2 \mu^2_l G^2,
\end{equation*}
which then implies that
\begin{equation*}
a^t_m= \frac{ \Vert g^t_m \Vert^2_2 }{B  (M \mu_l)^2} \leq \frac{BG^2}{M^2} \quad \text{ and } \quad a^t_{\max} \leq \frac{BG^2}{M^2}.
\end{equation*}
Thus, by Theorem~\ref{thm:upp-bd-general} and note that $q^t_s \in \mathcal{A}$ for all $t \in [T]$ and $\text{TV}(q^{1:T}_s) \leq \beta$, we have
\begin{equation*}
\text{D-Regret}_T (q^{1:T}_s) 
\leq \frac{ \log M  }{ \eta } + \frac{ 2 \beta \log (M / \alpha) }{ \eta } + \frac{\eta T M^2 B^2 G^4}{ 2 K^2 \alpha^6 }.
\end{equation*}
Let
\begin{equation*}
\eta = \frac{ K \alpha^3 }{ M B G^2 } \sqrt{ \frac{ 2 \log M + 4 \beta \log (M/\alpha) }{ T } },
\end{equation*}
we then have
\begin{equation*}
\text{D-Regret}_T (q^{1:T}_s) \leq \sqrt{T} \frac{ M B G^2 }{ K \alpha^3 } \sqrt{ \frac{1}{2} \log M + \beta \log (M/\alpha)  }.
\end{equation*}
Putting the above inequality to~\eqref{eq:proof-thm-fedavg-new-1}, we have
\begin{equation}
\label{eq:proof-thm-fedavg-new-2}
\begin{aligned}
\mathbb{E} \left[ \left\Vert \nabla F(w^R)  \right\Vert^2_2 \right] & \leq \frac{ 8 D^F}{\tilde{\mu} B T } + 8 \left( 4 + \frac{2}{\alpha}  \right) \tilde{\mu}^2 B^2 L^2 \left( \zeta^2_{\text{unif}}  + \frac{\sigma^2}{2 B} \right) \\
& \quad +  4 \tilde{\mu} B L  \left( \frac{2 \zeta^2_T (\alpha, \beta)}{K} + \frac{\sigma^2}{K B \alpha } + \frac{ M G^2 }{ K \alpha^3 } \sqrt{ \frac{ \frac{1}{2} \log M + \beta \log (M/\alpha) }{T}   } \right).
\end{aligned}
\end{equation}

When letting
\begin{equation*}
\frac{ 8 D^F}{\tilde{\mu} B T } = 8 \left( 4 + \frac{2}{\alpha}  \right) \tilde{\mu}^2 B^2 L^2 \left( \zeta^2_{\text{unif}}  + \frac{\sigma^2}{2 B} \right),
\end{equation*}
we have
\begin{equation*}
\tilde{\mu} = \frac{(D^F)^{\frac{1}{3}}}{  \left( 4 + \frac{2}{\alpha}  \right)^{\frac{1}{3}} B L^{\frac{2}{3}} \left( \zeta^2_{\text{unif}}  + \frac{\sigma^2}{2 B} \right)^{\frac{1}{3}} T^{\frac{1}{3}} },
\end{equation*}
which implies that
\begin{multline*}
\frac{ 8 D^F}{\tilde{\mu} B T } + 8 \left( 4 + \frac{2}{\alpha}  \right) \tilde{\mu}^2 B^2 L^2 \left( \zeta^2_{\text{unif}}  + \frac{\sigma^2}{2 B} \right) \leq \frac{ 16 \left( 4 + \frac{2}{\alpha}  \right)^{\frac{1}{3}} (D^F)^{\frac{2}{3}} L^{\frac{2}{3}} \left( \zeta^2_{\text{unif}}  + \frac{\sigma^2}{2 B} \right)^{\frac{1}{3}}  }{ T^{\frac{2}{3}} } \\
\leq \frac{ 16 \left( 4 + \frac{2}{\alpha}  \right)^{\frac{1}{3}} (D^F)^{\frac{2}{3}} L^{\frac{2}{3}} \left( \zeta^{\frac{2}{3}}_{\text{unif}}  + \frac{\sigma^{\frac{2}{3}}}{ (2 B)^{\frac{1}{3}} } \right)  }{ T^{\frac{2}{3}} }
\end{multline*}
On the other hand, when letting
\begin{equation*}
\frac{ 8 D^F}{\tilde{\mu} B T } = 4 \tilde{\mu} B L  \left( \frac{2 \zeta^2_T (\alpha, \beta)}{K} + \frac{\sigma^2}{K B \alpha } + \frac{ M G^2 }{ K \alpha^3 } \sqrt{ \frac{ \frac{1}{2} \log M + \beta \log (M/\alpha) }{T}} \right),
\end{equation*}
we have
\begin{equation*}
\tilde{\mu} = \frac{ \sqrt{ 2 D^F }  }{  B \sqrt{L} \sqrt{ \frac{2 \zeta^2_T (\alpha, \beta)}{K} +  \frac{\sigma^2}{K B \alpha } + \sqrt{ \frac{ \frac{1}{2} \log M + \beta \log (M/\alpha) }{T}} } \sqrt{T} },
\end{equation*}
which implies that
\begin{multline*}
\frac{ 8 D^F}{\tilde{\mu} B T } + 4 \tilde{\mu} B L  \left( \frac{2 \zeta^2_T (\alpha, \beta)}{K} + \frac{\sigma^2}{K B \alpha } + \frac{ \text{D-Regret}_T \left( q^{\star} \right) }{B T} \right) \\
\leq \frac{ 8 \sqrt{2} \sqrt{ D^F } \sqrt{ L } \sqrt{ \frac{2 \zeta^2_T (\alpha, \beta)}{K} +  \frac{\sigma^2}{K B \alpha } + \sqrt{ \frac{ \frac{1}{2} \log M + \beta \log (M/\alpha) }{T}}  } }{ \sqrt{T} } \\
\leq \frac{ 8 \sqrt{2} \sqrt{ D^F } \sqrt{ L } }{ \sqrt{T} } \left( \frac{\sqrt{2}\zeta_T (\alpha, \beta)}{\sqrt{K}} + \frac{\sigma}{\sqrt{K B \alpha}} + \left( \frac{ \frac{1}{2} \log M + \beta \log (M/\alpha) }{T} \right)^{\frac{1}{4}} \right).
\end{multline*}
Thus, when $\mu \geq 1$ and
\begin{multline*}
\tilde{\mu} = \min \left\{ \frac{1}{4 B L} \sqrt{\frac{1}{2+1/\alpha}},  \frac{(D^F)^{\frac{1}{3}}}{  \left( 4 + \frac{2}{\alpha}  \right)^{\frac{1}{3}} B L^{\frac{2}{3}} \left( \zeta^2_{\text{unif}}  + \frac{\sigma^2}{2 B} \right)^{\frac{1}{3}} T^{\frac{1}{3}} }, \right.\\ 
\left. \frac{ \sqrt{ 2 D^F }  }{  B \sqrt{L} \sqrt{ \frac{2 \zeta^2_T (\alpha, \beta)}{K} +  \frac{\sigma^2}{K B \alpha } + \sqrt{ \frac{ \frac{1}{2} \log M + \beta \log (M/\alpha) }{T} } } \sqrt{T} } \right\}, 
\end{multline*}
we then have
\begin{align*}
\mathbb{E} \left[ \left\Vert \nabla F(w^R)  \right\Vert^2_2 \right] 
\lesssim \frac{D^F L \sqrt{2 + \frac{1}{\alpha}}}{T} + \frac{ \left( 4 + \frac{2}{\alpha}  \right)^{\frac{1}{3}} (D^F)^{\frac{2}{3}} L^{\frac{2}{3}} \left( \zeta^{\frac{2}{3}}_{\text{unif}}  + \frac{\sigma^{\frac{2}{3}}}{ B^{\frac{1}{3}} } \right)  }{ T^{\frac{2}{3}} } \\
+ \frac{ \sqrt{ D^F } \sqrt{ L } }{ \sqrt{T} } \left( \frac{\zeta_T (\alpha, \beta)}{\sqrt{K}} + \frac{\sigma}{\sqrt{K B \alpha}} + \left( \frac{ \frac{1}{2} \log M + \beta \log (M/\alpha) }{T} \right)^{\frac{1}{4}} \right) \\
\lesssim \frac{D^F L \sqrt{2 + \frac{1}{\alpha}}}{T} + \frac{ \left( 4 + \frac{2}{\alpha}  \right)^{\frac{1}{3}} (D^F L)^{\frac{2}{3}} \zeta^{\frac{2}{3}}_{\text{unif}}  }{ T^{\frac{2}{3}} } + \frac{ \left( 4 + \frac{2}{\alpha}  \right)^{\frac{1}{3}} (D^F L)^{\frac{2}{3}}  \sigma^{\frac{2}{3}}  }{ B ^{\frac{1}{3}} T^{\frac{2}{3}} } \\
+ \frac{ \sqrt{ D^F L }\zeta_T (\alpha, \beta) }{ \sqrt{T K} } + \frac{ \sqrt{ D^F L } \sigma  }{ \sqrt{ T K B \alpha } } + \frac{ \sqrt{D^F L } }{ \sqrt{T} } \left( \frac{ \frac{1}{2} \log M + \beta \log (M/\alpha) }{T} \right)^{\frac{1}{4}}.
\end{align*}

\subsection{Proof of Theorem~\ref{thm:upp-bd-general}}
\label{sec:proof-thm-bd-1}

We first state a proposition that will be used to prove Theorem~\ref{thm:upp-bd-general}. The key difference between Theorem~\ref{thm:upp-bd-general} and Proposition~\ref{prop:upp-bd-general-simplex} is that in Proposition~\ref{prop:upp-bd-general-simplex} the comparator sequence lies in $\mathcal{A}$, and, as a result, there is no projection error.

\begin{proposition}\label{prop:upp-bd-general-simplex}
Suppose the conditions of Theorem~\ref{thm:upp-bd-general} hold. For any comparator sequence $q^{1:T}$ with $q^t \in \mathcal{A}$, $t\in[T]$, we have
\begin{equation*}
\text{D-Regret}_T (q^{1:T})  \leq \frac{ \log M  }{ \eta } + \frac{ 2 \log (M / \alpha) }{ \eta } \mathbb{E} \left[\text{TV} \left( q^{1:T} \right) \right] + \frac{ \eta M^6}{ 2 K^2 \alpha^6 } \sum^T_{t=1} \mathbb{E} \left[  \left( a^t_{\max}  \right)^2 \right].
\end{equation*}
\end{proposition}
\begin{proof}
By Lemma~\ref{lemma:convx-opt-cond} and the definition of $\hat{p}^{t+1}$ in Step~\ref{alg:solve-mirror-descent} of Algorithm~\ref{alg:OSMD-sampler}, we have
\begin{equation}\label{eq:proof-thm2-1}
\langle \hat{p}^{t+1} - q^t,  \nabla \hat{l}_t ( \hat{p}^t ; \hat{p}^t ) \rangle \leq \frac{1}{\eta} \langle  \nabla \Phi (\hat{p}^{t}) - \nabla \Phi (\hat{p}^{t+1}), \hat{p}^{t+1} - q^t \rangle.
\end{equation}
By the convexity of $\hat{l}_t ( \cdot ; \hat{p}^t )$, we have
\begin{equation*}
\hat{l}_t ( \hat{p}^t ; \hat{p}^t ) - \hat{l}_t ( q^t ; \hat{p}^t ) \leq \langle \nabla  \hat{l}_t ( \hat{p}^t ; \hat{p}^t ), \hat{p}^t - q^t  \rangle = \langle \nabla  \hat{l}_t ( \hat{p}^t ; \hat{p}^t ), \hat{p}^{t+1} - q^t  \rangle + \langle \nabla  \hat{l}_t ( \hat{p}^t ; \hat{p}^t ), \hat{p}^t - \hat{p}^{t+1}  \rangle.
\end{equation*}
Then, by \eqref{eq:proof-thm2-1}, we further have
\begin{equation*}
\hat{l}_t ( \hat{p}^t ; \hat{p}^t ) - \hat{l}_t ( q^t ; \hat{p}^t ) \leq \frac{1}{\eta} \langle  \nabla \Phi (\hat{p}^{t}) - \nabla \Phi (\hat{p}^{t+1}), \hat{p}^{t+1} - q^t \rangle + \langle \nabla  \hat{l}_t ( \hat{p}^t ; \hat{p}^t ), \hat{p}^t - \hat{p}^{t+1}  \rangle.
\end{equation*}
From the definition of $\mathcal{D}$, we have 
\[
D_{\Phi} \left(x_1 \,\Vert\, x_2 \right) = D_{\Phi} \left(x_3 \,\Vert\, x_2 \right) + D_{\Phi} \left(x_1 \,\Vert\, x_3 \right) + \langle \nabla \Phi(x_2) - \nabla \Phi(x_3), x_3 - x_1 \rangle,
\qquad x_1,x_2,x_3\in\mathcal{D}.
\]
Then
\begin{align*}
\hat{l}_t ( \hat{p}^t ; \hat{p}^t ) & - \hat{l}_t ( q^t ; \hat{p}^t ) \\
& \leq \frac{1}{\eta} \left[  D_{\Phi} \left( q^t \Vert \hat{p}^t \right) -  D_{\Phi} \left( q^t \Vert \hat{p}^{t+1} \right) - D_{\Phi} \left( \hat{p}^{t+1} \Vert \hat{p}^t \right) \right] + \langle \nabla  \hat{l}_t ( \hat{p}^t ; \hat{p}^t ), \hat{p}^t - \hat{p}^{t+1}  \rangle \\
& =  \frac{1}{\eta} \left[ D_{\Phi} \left( q^t \Vert \hat{p}^t \right) - D_{\Phi} \left( q^{t+1} \Vert \hat{p}^{t+1} \right) \right] + \frac{1}{\eta} \left[ D_{\Phi} \left( q^{t+1} \Vert \hat{p}^{t+1} \right) - D_{\Phi} \left( q^t \Vert \hat{p}^{t+1} \right)  \right] \\
& \qquad -  \frac{1}{\eta} D_{\Phi} \left( \hat{p}^{t+1} \Vert \hat{p}^t \right) + \langle \nabla  \hat{l}_t ( \hat{p}^t ; \hat{p}^t ), \hat{p}^t - \hat{p}^{t+1}  \rangle. \numberthis \label{eq:proof-thm2-2}
\end{align*}
We bound the second term in \eqref{eq:proof-thm2-2} as
\begin{align*}
 D_{\Phi} \left( q^{t+1} \Vert \hat{p}^{t+1} \right) - D_{\Phi} \left( q^t \Vert \hat{p}^{t+1} \right)  
& =  \Phi (q^{t+1}) - \Phi (q^t) - \langle \nabla \Phi (\hat{p}^{t+1}), q^{t+1} - q^t \rangle  \\
& \overset{(a)}{\leq}    \langle \nabla \Phi (q^{t+1}) -  \nabla \Phi (\hat{p}^{t+1}), q^{t+1} - q^t   \rangle \\
& \overset{(b)}{\leq}  \Vert \nabla \Phi (q^{t+1}) -  \nabla \Phi (\hat{p}^{t+1}) \Vert_{\infty} \Vert q^{t+1} - q^t \Vert_1 \\
& \overset{(c)}{\leq} 2 \log ( M /\alpha) \Vert q^{t+1} - q^t \Vert_1, \numberthis \label{eq:proof-thm2-3}
\end{align*}
where (a) follows from the convexity of $\Phi(\cdot)$, (b) follows from the fact that the dual norm of $\Vert \cdot \Vert_1$ is $\Vert \cdot \Vert_{\infty}$, and (c) follows from the following fact that
\begin{equation*}
\Vert \nabla \Phi (p) \Vert_{\infty} = \max_{1 \leq m \leq M} \vert \log(p_m) \vert \leq \log (M /\alpha ) \text{ for all } p \in \mathcal{A}.
\end{equation*}

Besides, by Pinsker's inequality, we have $D_{\Phi} \left(p \,\Vert\, q\right) \geq \frac{1}{2} \Vert p - q \Vert^2_1$ for all $p,q \in \mathcal{P}_{M-1}$.
Thus, $\Phi(\cdot)$ is $1$-strongly convex, we can bound the third and fourth term in \eqref{eq:proof-thm2-2} as
\begin{equation*}
- \frac{1}{\eta} D_{\Phi} \left( \hat{p}^{t+1} \Vert \hat{p}^t \right) + \langle \nabla  \hat{l}_t ( \hat{p}^t ; \hat{p}^t ), \hat{p}^t - \hat{p}^{t+1}  \rangle \leq - \frac{1}{2 \eta} \Vert \hat{p}^{t+1} - \hat{p}^t \Vert^2_1 + \Vert \nabla  \hat{l}_t ( \hat{p}^t ; \hat{p}^t ) \Vert_{\infty} \Vert \hat{p}^t - \hat{p}^{t+1} \Vert_1.
\end{equation*}
Since $ab\leq a^2 / (2\epsilon) + b^2 \epsilon / 2$, $a,b,\epsilon>0$, we further have
\begin{align*}
- \frac{1}{\eta} D_{\Phi} \left( \hat{p}^{t+1} \Vert \hat{p}^t \right) + \langle \nabla  \hat{l}_t ( \hat{p}^t ; \hat{p}^t ), \hat{p}^t - \hat{p}^{t+1}  \rangle \leq  \frac{\eta}{2} \Vert \nabla  \hat{l}_t ( \hat{p}^t ; \hat{p}^t ) \Vert^2_{\infty}.
\end{align*}
Let
\begin{equation*}
Q_t=M^3 a^t_{\max} / (K \alpha^3),
\end{equation*}
we then have
\begin{equation*}
\left\vert \left[ \nabla \hat{l}_t (q ; p^t ) \right]_m \right\vert = \frac{1}{K^2} \cdot \frac{ a^t_m }{ q^2_m p^t_m } \mathcal{N} \left\{ m \in S^t \right\} \leq \frac{1}{K^2} \cdot \frac{a^t_{\max}}{ \alpha^3 / M^3 } \cdot K \leq \frac{M^3 a^t_{\max} }{ K \alpha^3 } = Q_t
\end{equation*}
for all $m \in [M]$, thus we have $\Vert \nabla \hat{l}_t (q ; p^t ) \Vert_{\infty} \leq Q_t$, which then implies that
\begin{equation}
\label{eq:proof-thm2-4}
- \frac{1}{\eta} D_{\Phi} \left( \hat{p}^{t+1} \Vert \hat{p}^t \right) + \langle \nabla  \hat{l}_t ( \hat{p}^t ; \hat{p}^t ), \hat{p}^t - \hat{p}^{t+1}  \rangle \leq \frac{\eta}{2}  Q_t^2.
\end{equation}
Combining \eqref{eq:proof-thm2-2}-\eqref{eq:proof-thm2-4}, we have
\begin{equation*}
\hat{l}_t ( \hat{p}^t ; \hat{p}^t ) - \hat{l}_t ( q^t ; \hat{p}^t ) \leq \frac{D_{\Phi} \left( q^t \Vert \hat{p}^t \right)}{\eta}   - \frac{D_{\Phi} \left( q^{t+1} \Vert \hat{p}^{t+1} \right)}{\eta}  + \frac{2\log (M/\alpha)}{\eta} \Vert q^{t+1} - q^t \Vert_1  + \frac{\eta}{2} Q^2_t.
\end{equation*}
This implies that
\begin{align*}
\sum^T_{t=1}\hat{l}_t ( \hat{p}^t ; \hat{p}^t ) &- \sum^T_{t=1} \hat{l}_t ( q^t ; \hat{p}^t ) \\
& \leq  \frac{D_{\Phi} \left( q^1 \Vert \hat{p}^1 \right)}{\eta} - \frac{D_{\Phi} \left( q^{T+1} \Vert \hat{p}^{T+1} \right)}{\eta} + \frac{2\log (M/\alpha)}{\eta} \sum^T_{t=1} \Vert q^{t+1} - q^t \Vert_1 + \frac{\eta}{2} \sum^T_{t=1}  Q^2_t \\
& \leq  \frac{ D_{\Phi} \left( q^1 \Vert \hat{p}^1 \right) }{\eta} + \frac{2 \log (M/\alpha)}{\eta} \sum^T_{t=1} \Vert q^{t+1} - q^t \Vert_1 + \frac{\eta}{2} \sum^T_{t=1}  Q^2_t.
\end{align*}
Since $\hat{p}_1$ is the uniform distribution, we have that
\begin{equation*}
D_{\Phi}(q \,\Vert\, \punif ) = \log M + \sum^M_{m=1} q_m \log q_m \leq \log M \text{ for all } q \in \mathcal{P}_{M-1}.
\end{equation*}
Thus, we have
\begin{equation}
\label{eq:proof-thm2-5}
\sum^T_{t=1}\hat{l}_t ( \hat{p}^t ; \hat{p}^t ) - \sum^T_{t=1} \hat{l}_t ( q^t ; \hat{p}^t ) \leq  \frac{ \log M }{\eta} + \frac{2 \log (M/\alpha)}{\eta} \text{TV} \left( q^{1:T} \right) + \frac{\eta}{2} \sum^T_{t=1} Q^2_t
\end{equation}

Finally, note that
\begin{align*}
\mathbb{E} \left[ \sum^T_{t=1} \hat{l}_t ( \hat{p}^t ; \hat{p}^t ) - \sum^T_{t=1} \hat{l}_t ( q^t ; \hat{p}^t )  \right] & =  \sum^T_{t=1} \mathbb{E} \left[  \mathbb{E}_{S^t} \left[  \hat{l}_t ( \hat{p}^t ; \hat{p}^t ) \right] - \mathbb{E}_{S^t} \left[ \hat{l}_t ( q^t ; \hat{p}^t ) \right]  \right] \\
& =  \sum^T_{t=1} \mathbb{E} \left[ l_t ( \hat{p}^t ) -  l_t ( q^t ) \right] \\
& =  \text{D-Regret}_T (q^{1:T}).
\end{align*}
The conclusion follows by taking expectation on right hand side of~\eqref{eq:proof-thm2-5}.
\end{proof}

We are now ready to  prove Theorem~\ref{thm:upp-bd-general}.
\begin{proof}[Proof of Theorem~\ref{thm:upp-bd-general}]
For any comparator sequence $q^{1:T}$ with $q^t \in \mathcal{P}_{M-1}$, $t \in [T]$, we prove Theorem~\ref{thm:upp-bd-general} by first constructing a suitable sequence $\tilde{q}^{1:T}$ that is defined as
\begin{equation}\label{eq:proj-seq}
\tilde{q}^t_m =
\begin{cases}
\alpha/M & \text{if } q^t_m < \alpha/M, \\
  q^t_m - \omega(q^t, \alpha) \left( q^t_m - \frac{\alpha}{M} \right)  & \text{if } q^t_m \geq \alpha/M,
\end{cases}
\end{equation}
where $\omega(q^t, \alpha)$ is defined in~\eqref{eq:def-psi-alpha}. We now show that $\tilde{q}^t \in \mathcal{A}$, $t \in [T]$, by showing that $\tilde{q}^t_m \geq \alpha / M$, $m \in [M]$, and $\sum_{m\in[M]} \tilde{q}^t_m = 1$. For $m \in [M]$ such that $q^t_m < \alpha/M$, we have from~\eqref{eq:proj-seq} that $\tilde{q}^t_m=\alpha/M$. For $m \in [M]$ such that $q^t_m \geq \alpha/M$, by~\eqref{eq:proj-seq}, we have $\tilde{q}^t_m - {\alpha}/{M} = \left( 1 - \omega(q^t, \alpha) \right) \left( q^t_m - {\alpha}/{M} \right)$.
Thus, we proceed to show that $\omega(q^t, \alpha) \leq 1$. Since
\begin{multline*}
\sum^M_{m=1} q^t_m \mathds{1} \left\{ q^t_m < \frac{\alpha}{M} \right\} + \sum^M_{m=1} q^t_m \mathds{1} \left\{ q^t_m \geq \frac{\alpha}{M} \right\} = 1 \\
\geq \alpha=\sum^M_{m=1} \frac{\alpha}{M} \mathds{1} \left\{ q^t_m < \frac{\alpha}{M} \right\} + \sum^M_{m=1} \frac{\alpha}{M}\mathds{1} \left\{ q^t_m \geq \frac{\alpha}{M} \right\},
\end{multline*}
we have
\begin{equation*}
\sum^M_{m=1} \left( q^t_m - \frac{\alpha}{M} \right) \mathds{1} \left\{ q^t_m \geq \frac{\alpha}{M} \right\} \geq \sum^M_{m=1} \left( \frac{\alpha}{M} - q^t_m \right) \mathds{1} \left\{ q^t_m < \frac{\alpha}{M} \right\}.
\end{equation*}
Therefore, $0 \leq \omega(q^t, \alpha) \leq 1$. Furthermore, $\omega(q^t, 0)=1$ and $\omega(q^t, 1)=1$. Finally, we show that $\sum^M_{m=1} \tilde{q}^t_m = 1$. By~\eqref{eq:proj-seq} and the definition of $\omega(q^t, \alpha)$ in~\eqref{eq:def-psi-alpha}, we have
\begin{align*}
\sum^M_{m=1} \tilde{q}^t_m 
& = \sum^M_{m=1} \frac{\alpha}{M} \mathds{1} \left\{ q^t_m < \frac{\alpha}{M} \right\} + \sum^M_{m=1} q^t_m \mathds{1} \left\{ q^t_m \geq \frac{\alpha}{M} \right\} \\
& \quad - \omega(q^t, \alpha) \sum^M_{m=1} \left( q^t_m - \frac{\alpha}{M}  \right) \mathds{1} \left\{ q^t_m \geq \frac{\alpha}{M} \right\} \\
& = \sum^M_{m=1} \frac{\alpha}{M} \mathds{1} \left\{ q^t_m < \frac{\alpha}{M} \right\} + \sum^M_{m=1} q^t_m \mathds{1} \left\{ q^t_m \geq \frac{\alpha}{M} \right\} - \sum^M_{m=1} \left( \frac{\alpha}{M} - q^t_m \right) \mathds{1} \left\{ q^t_m < \frac{\alpha}{M} \right\} \\
& = \sum^M_{m=1} q^t_m \mathds{1} \left\{ q^t_m \geq \frac{\alpha}{M} \right\} + \sum^M_{m=1} q^t_m \mathds{1} \left\{ q^t_m < \frac{\alpha}{M} \right\} \\
& = 1.
\end{align*}
Therefore, $\tilde{q}^t \in \mathcal{A}$ for any $t \in [T]$.

Note that we then have
\begin{equation}
\label{eq:proof-thm1-1}
\text{D-Regret}_T (q^{1:T}) 
=  \mathbb{E} \left[ \sum^T_{t=1} l_t (\hat{p}^t)  - \sum^T_{t=1} l_t (\tilde{q}^t) + \sum^T_{t=1} l_t (\tilde{q}^t) - \sum^T_{t=1} l_t (q^t) \right].
\end{equation}
By Proposition~\ref{prop:upp-bd-general-simplex}, we further have that
\begin{multline}
\label{eq:proof-thm1-2} 
\mathbb{E} \left[ \sum^T_{t=1} l_t (\hat{p}^t)  - \sum^T_{t=1} l_t (\tilde{q}^t) \right] \leq 
\frac{ \log M }{\eta} + \frac{2 \log (M/\alpha) }{\eta} \mathbb{E} \left[ \text{TV} \left( q^{1:T} \right) \right] \\
+ \frac{ \eta M^6}{ 2 K^2 \alpha^6 } \sum^T_{t=1} \mathbb{E} \left[  \left( a^t_{\max}  \right)^2 \right] + \frac{2 \log (M/\alpha) }{\eta} \mathbb{E} \left[ \text{TV} \left( \tilde{q}^{1:T} \right) - \text{TV} \left( q^{1:T} \right)  \right]. 
\end{multline}
Therefore, to prove Theorem~\ref{thm:upp-bd-general}, we need to bound the terms $\sum^T_{t=1} l_t (\tilde{q}^t) - \sum^T_{t=1} l_t (q^t)$ and $\text{TV} \left( \tilde{q}^{1:T} \right) - \text{TV} \left( q^{1:T} \right)$.

We first bound $\sum^T_{t=1} l_t (\tilde{q}^t) - \sum^T_{t=1} l_t (q^t)$. 
When $q^t_m < \alpha / M$, then $1/\tilde{q}^t_m - 1/q^t_m < 0$; and when $q^t_m \geq \alpha / M$, then
\begin{equation*}
\frac{1}{\tilde{q}^t_m} - \frac{1}{q^t_m} = \frac{1}{q^t_m} \cdot \left[ \frac{1}{1-\omega(q^t, \alpha)\left( 1 - \frac{\alpha}{M q^t_m} \right)} - 1\right] = \frac{1}{q^t_m} \cdot \frac{ \omega(q^t, \alpha)\left( 1 - \frac{\alpha}{M q^t_m} \right) }{ 1-\omega(q^t, \alpha)\left( 1 - \frac{\alpha}{M q^t_m} \right) }.
\end{equation*}
Since 
\[
\omega(q^t, \alpha)\left( 1 - \frac{\alpha}{M q^t_m} \right) \leq \omega(q^t, \alpha)
\quad 
\text{and}
\quad
1-\omega(q^t, \alpha)\left( 1 - \frac{\alpha}{M q^t_m} \right)  \geq 1 - \omega(q^t, \alpha) + \frac{\omega(q^t, \alpha) \alpha}{M}
\]
as $q^t_m \leq 1$, we have
\begin{equation*}
\frac{1}{\tilde{q}^t_m} - \frac{1}{q^t_m} \leq \frac{1}{q^t_m} \cdot \frac{ \omega(q^t, \alpha) }{ 1 - \omega(q^t, \alpha) \left( 1 - \frac{\alpha}{M} \right)  }
= \frac{\phi(q^t, \alpha)}{q^t_m}.
\end{equation*}
Thus,
\begin{align*}
\sum^T_{t=1} l_t (\tilde{q}^t) - \sum^T_{t=1} l_t (q^t) & = \frac{1}{K} \sum^T_{t=1} \sum^M_{m=1} a^t_m \left( \frac{1}{\tilde{q}^t_m} - \frac{1}{q^t_m} \right) \\
& \leq \frac{1}{K} \sum^T_{t=1} \sum^M_{m=1} a^t_m \left( \frac{1}{\tilde{q}^t_m} - \frac{1}{q^t_m} \right) \mathds{1} \left\{ q^t_m \geq \frac{\alpha}{M} \right\} \\
& \leq \frac{1}{K} \sum^T_{t=1} \phi(q^t, \alpha) \sum^M_{m=1} \frac{a^t_m}{q^t_m} \mathds{1} \left\{ q^t_m \geq \frac{\alpha}{M} \right\} \\
& \leq \frac{1}{K} \sum^T_{t=1} \phi(q^t, \alpha) l_t (q^t). \numberthis \label{eq:thm1-proof-helper5}
\end{align*}

Next, we bound $\text{TV} \left( \tilde{q}^{1:T} \right) - \text{TV} \left( q^{1:T} \right)$. Note that
\begin{align*}
\text{TV} \left( \tilde{q}^{1:T} \right) & = \sum^T_{t=2} \left\Vert \tilde{q}^t - \tilde{q}^{t-1} \right\Vert_1 \\
& = \sum^T_{t=2} \left\Vert \tilde{q}^t - q^t + q^t - q^{t-1} + q^{t-1} - \tilde{q}^{t-1} \right\Vert_1 \\
& \leq \sum^T_{t=2} \left\Vert \tilde{q}^t - q^t \right\Vert_1 + \sum^T_{t=2} \left\Vert q^t - q^{t-1} \right\Vert_1 + \sum^T_{t=2} \left\Vert q^{t-1} - \tilde{q}^{t-1} \right\Vert_1 \\
& \leq \text{TV} \left( q^{1:T} \right) + 2 \sum^T_{t=1} \left\Vert \tilde{q}^t - q^t \right\Vert_1.
\end{align*}
We now upper bound $\sum^T_{t=1} \left\Vert \tilde{q}^t - q^t \right\Vert_1$. If $q^t_m < \alpha/M$, then $\vert \tilde{q}^t_m - q^t_m  \vert=\alpha/M - q^t_m$. If $q^t_m \geq \alpha/M$, by~\eqref{eq:proj-seq}, we have $\left\vert \tilde{q}^t_m - q^t_m  \right\vert = \omega(q^t, \alpha) \left( q^t_m - {\alpha}/{M} \right)$. Therefore, recalling the definition of $\psi(q^t, \alpha)$ in~\eqref{eq:def-psi-alpha}, we have
\begin{align*}
\left\Vert \tilde{q}^t - q^t \right\Vert_1 &= \sum^M_{m=1} \left( \frac{\alpha}{M} - q^t_m \right) \mathds{1} \left\{ q^t_m < \frac{\alpha}{M} \right\} + \omega(q^t, \alpha) \sum^M_{m=1} \left( q^t_m  -\frac{\alpha}{M} \right) \mathds{1} \left\{ q^t_m \geq \frac{\alpha}{M} \right\} \\
& = 2 \sum^M_{m=1} \left( \frac{\alpha}{M} - q^t_m \right) \mathds{1} \left\{ q^t_m < \frac{\alpha}{M} \right\} \\
& = 2 \psi(q^t, \alpha)
\end{align*}
and 
\begin{equation}
\label{eq:thm1-proof-helper4}
\text{TV} \left( \tilde{q}^{1:T} \right) - \text{TV} \left( q^{1:T} \right) \leq 4 \sum^T_{t=1}\psi(q^t, \alpha).
\end{equation}

Combining \eqref{eq:proof-thm1-1}, \eqref{eq:proof-thm1-2}, \eqref{eq:thm1-proof-helper5}, and \eqref{eq:thm1-proof-helper4}, and taking expectation on both sides, we obtain the final result.
\end{proof}

\subsection{Proof of Theorem~\ref{thm:thm-dap-OSMD-sampler}}
\label{sec:proof-thm-dap-OSMD-sampler}

Given a comparator sequence $q^{1:T}$, where $q^t$ is allowed to be random, such that $q^t \in \mathcal{A}$ for all $t \in [T]$ and $\mathbb{E} [\text{TV}(q^{1:T})] \leq \beta$, 
let
\begin{equation}
\label{eq:eta-star-comp}
\eta^{\star} = \frac{ K \alpha^3 }{ M^3 A_{\max} } \sqrt{ \frac{ 2\log M + 4 \beta \log (M / \alpha) }{ T  } }.
\end{equation}
The proof proceeds in two steps. First, we show that there exists an expert learning rate $\eta_e \in \mathcal{E}$ such that the regret bound for $\hat{p}^{1:T}_e$ is close to~\eqref{eq:regret-meta-pre}. That is, we show that there exists $\eta_e \in \mathcal{E}$ such that
\begin{equation}
\label{eq:proof-dap-1}
\mathbb{E} \left[ \sum^T_{t=1} \hat{l}_t (\hat{p}^t_e ; \hat{p}^t )  - \sum^T_{t=1} l_t (q^t) \right] \leq \frac{ 3 M^3 A_{\max} }{ K \alpha^3 } \sqrt{ T \left[ \frac{1}{2}\log M + \beta \log \left( M / \alpha \right)   \right] }.
\end{equation}
Note that $S^t \sim \hat{p}^t$. Second, we show that the output of meta-algorithm can track the best expert with small regret. That is, we show that
\begin{equation}
\label{eq:proof-dap-2}
\mathbb{E} \left[ \sum^T_{t=1} l_t (\hat{p}^t ) \right] - \mathbb{E} \left[ \sum^T_{t=1} \hat{l}_t (\hat{p}^t_e ; \hat{p}^t ) \right] \leq \frac{M}{\alpha} \sqrt{ \frac{T A_{\max} }{ 8 K } } (1 + 2 \log E ), \qquad e \in [E].
\end{equation}
The theorem then follows by combining \eqref{eq:proof-dap-1} and \eqref{eq:proof-dap-2}.

We first prove \eqref{eq:proof-dap-1}. Since $0 \leq \beta \leq 2 (T-1)$, we have
\begin{equation*}
\min \mathcal{E} = \frac{ K \alpha^3 }{ M^3 A_{\max} } \sqrt{ \frac{2 \log M}{ T  }  } \leq \eta^{\star} \leq \frac{ K \alpha^3 }{ M^3 A_{\max} } \sqrt{ \frac{ 2 \log M + 8 \log (M/\alpha)(T-1)  }{ T }  } \leq \max \mathcal{E},
\end{equation*}
where $\eta^{\star}$ is defined as in~\eqref{eq:eta-star-comp}. Thus, there exists $\eta_e \in \mathcal{E}$, such that $\eta_e \leq \eta^{\star} \leq 2 \eta_e$. Repeating the proof of~\eqref{eq:proof-thm2-5}, we can show that
\begin{equation*}
\sum^T_{t=1}\hat{l}_t ( \hat{p}^t_e ; \hat{p}^t ) - \sum^T_{t=1} \hat{l}_t ( q^t ; \hat{p}^t ) \leq \frac{ \log M  }{ \eta_e } + \frac{ 2 \log (M / \alpha) }{ \eta_e } \text{TV} \left( q^{1:T} \right) +  \frac{ \eta_e M^6}{ 2 K^2 \alpha^6 } \sum^T_{t=1}   \left( a^t_{\max}  \right)^2,
\end{equation*}
which then implies that
\begin{align*}
\mathbb{E} \left[\sum^T_{t=1}\hat{l}_t ( \hat{p}^t_e ; \hat{p}^t ) - \sum^T_{t=1} \hat{l}_t ( q^t ; \hat{p}^t ) \right] &\leq \frac{ \log M  }{ \eta_e } + \frac{ 2 \log (M / \alpha) }{ \eta_e } \mathbb{E} \left[\text{TV} \left( q^{1:T} \right) \right] \\
& \quad\quad +  \frac{\eta_e M^6}{ 2 K^2 \alpha^6 } \sum^T_{t=1}  \mathbb{E} \left[\left( a^t_{\max} \right)^2 \right] \\
& \leq \frac{ \log M  }{ \eta_e } + \frac{ 2 \beta \log (M / \alpha) }{ \eta_e } +  \frac{\eta_e M^6 T A^2_{\max} }{ 2 K^2 \alpha^6 } .
\end{align*}
Since $\eta^{\star}/2 \leq \eta_e \leq \eta^{\star}$, we further have
\begin{align*}
& \quad \mathbb{E} \left[\sum^T_{t=1}\hat{l}_t ( \hat{p}^t_e ; \hat{p}^t ) - \sum^T_{t=1} \hat{l}_t ( q^t ; \hat{p}^t ) \right] \\
& \leq  \frac{ 2 \log M  }{ \eta^{\star} } + \frac{ 4 \beta \log (M / \alpha) }{ \eta^{\star} }  +   \frac{\eta^{\star} M^6 T A^2_{\max} }{ 2 K^2 \alpha^6 } \\
& = \frac{ 3 M^3 A_{\max} }{ K \alpha^3 } \sqrt{ T \left[ \frac{1}{2}\log M + \beta \log \left( M / \alpha \right)   \right] }.
\end{align*}
Now, \eqref{eq:proof-dap-1} follows, since
\begin{equation*}
\mathbb{E} \left[ \sum^T_{t=1} \hat{l}_t (q^t ; \hat{p}^t )  \right] = \sum^T_{t=1} \mathbb{E} \left[ \mathbb{E}_{S^t} \left[ \hat{l}_t ( q^t ; \hat{p}^t ) \right]  \right] = \sum^T_{t=1} \mathbb{E} \left[ l_t (q^t)  \right].
\end{equation*}

We prove~\eqref{eq:proof-dap-2} next. Let
\begin{equation*}
\hat{L}^e_t = \sum^t_{s=1} \hat{l}_s ( \hat{p}^s_e ; \hat{p}^s ) \quad e \in [E], t \in [T].
\end{equation*}
Recall the update for $\theta^t_e$ in Step~\ref{step:weight-update} of Alg~\ref{alg:adap-OSMD-sampler-meta}. We have
\begin{equation*}
\theta^t_e = \frac{ \theta^1_e \exp \left( - \gamma \hat{L}^e_{t-1} \right) }{ \sum^E_{b=1} \theta^1_b \exp \left( - \gamma \hat{L}^b_{t-1} \right)  }, \qquad t = 2, \dots T.
\end{equation*}
Let $\Theta_t = \sum^E_{b=1} \theta^1_b \exp \left\{ - \gamma \hat{L}^b_t \right\}$. Then
\begin{align*}
\log \Theta_1 & = \log \left( \sum^E_{b=1} \theta^1_b  \exp \left\{ -\gamma \hat{L}^b_1 \right\}  \right) \\
\intertext{and, for $t \geq 2$, }
\log \left( \frac{\Theta_t}{\Theta_{t-1}} \right) 
& =  \log \left( \frac{ \sum^E_{b=1} \theta^1_b \exp \left\{ -\gamma \hat{L}^b_{t-1} \right\}  \exp \left\{  - \gamma \hat{l}_t (\hat{p}^t_b ; \hat{p}^t) \right\}  }{ \sum^E_{b=1} \theta^1_b \exp \left\{ -\gamma \hat{L}^b_{t-1} \right\} } \right) \\
& =  \log \left( \sum^E_{b=1} \theta^t_b \exp \left\{  - \gamma \hat{l}_t (\hat{p}^t_b ; \hat{p}^t) \right\} \right).
\end{align*}
We have
\begin{align*}
\log \Theta_T 
& = \log \Theta_1 + \sum^T_{t=1} \log \left( \frac{\Theta_t}{\Theta_{t-1}} \right) \\
& = \sum^T_{t=1} \log \left( \sum^E_{b=1} \theta^t_b \exp \left\{ -\gamma \hat{l}_t (\hat{p}^t_b ; \hat{p}^t)  \right\}  \right) \\
& \leq \sum^T_{t=1} \left( 
- \gamma \sum^E_{b=1} \theta^t_b \hat{l}_t (\hat{p}^t_b ; \hat{p}^t) + \frac{\gamma^2 M^2 a^t_{\max} }{ 8 K \alpha^2 } \right) && \text{(Lemma~\ref{lemma:Hoeffding-inequality})} \\
& \leq 
- \gamma \sum^T_{t=1} \hat{l}_t (\hat{p}^t ; \hat{p}^t) + \frac{\gamma^2 M^2 \left( \sum^T_{t=1} a^t_{\max} \right) }{ 8 K \alpha^2 } && \text{(Jensen's inequality)}
\end{align*}
and
\begin{multline*}
\log \left( \Theta_T  \right) = \log \left( \sum^E_{b=1} \theta^1_b \exp \left\{ - \gamma \hat{L}^b_T \right\} \right) \\
\geq \log \left(  \max_{1\leq b \leq E} \theta^1_b \exp \left\{ - \gamma \hat{L}^b_T \right\} \right) 
= - \gamma \min_{1 \leq b \leq E} \left\{  \hat{L}^b_T + \frac{1}{\gamma} \log \frac{1}{\theta^1_b} \right\}.
\end{multline*}
Combining the last two displays, we have
\begin{equation*}
- \gamma \min_{1 \leq b \leq E} \left\{  \hat{L}^b_T + \frac{1}{\gamma} \log \frac{1}{\theta^1_b} \right\}
\leq - \gamma \sum^T_{t=1} \hat{l}_t (\hat{p}^t ; \hat{p}^t) + \frac{\gamma^2 M^2 \left( \sum^T_{t=1} a^t_{\max} \right) }{ 8 K \alpha^2 },
\end{equation*} 
which implies that
\begin{equation*}
\sum^T_{t=1} \hat{l}_t (\hat{p}^t ; \hat{p}^t) - \hat{L}^e_T \leq \frac{\gamma M^2 \left( \sum^T_{t=1} a^t_{\max} \right) }{ 8 K \alpha^2 } + \frac{1}{\gamma} \log \frac{1}{\theta^1_e} \leq \frac{\gamma M^2 T A_{\max} }{ 8 K \alpha^2 } + \frac{1}{\gamma} \log \frac{1}{\theta^1_e},
\quad e \in [E].
\end{equation*}
Taking expectation on both sides, we then have
\begin{equation*}
\mathbb{E} \left[ \sum^T_{t=1} \hat{l}_t (\hat{p}^t ; \hat{p}^t) - \hat{L}^e_T \right] \leq \frac{\gamma M^2 T A_{\max} }{ 8 K \alpha^2 } + \frac{1}{\gamma} \log \frac{1}{\theta^1_e}.
\end{equation*}

Since $\theta^1_e \geq \frac{1}{E^2}$, $\log 1/\theta^1_e \leq 2 \log E$. Let $\gamma=\sqrt{8 K \alpha^2 / (T M^2 A_{\max} ) }$ to minimize the right hand side of the above inequality with $\log 1/\theta^1_e$ substituted by 1. Then 
\begin{align*}
\mathbb{E} \left[ \sum^T_{t=1} \hat{l}_t (\hat{p}^t ; \hat{p}^t) - \hat{L}^e_T \right] & = \mathbb{E} \left[ \sum^T_{t=1} \hat{l}_t (\hat{p}^t ; \hat{p}^t) - \sum^T_{t=1} \hat{l}_t ( \hat{p}^t_e ; \hat{p}^t ) \right] \\
& \leq \frac{M}{\alpha} \sqrt{ \frac{T A_{\max} }{ 8 K } }  \left( 1 + 2 \log E \right),
\quad e \in [E].
\end{align*}

\subsection{Proof of Theorem~\ref{thm:adp-osmd-opt-thm}}

For Mini-batch SGD, following~\eqref{eq:proof-opt-new-1} and Theorem~\ref{thm:thm-dap-OSMD-sampler}, we have
\begin{align*}
& \mathbb{E} \left[ \left\Vert \nabla F \left( w^R \right) \right\Vert^2_2 \right] \\ 
& \leq \frac{2 D^F}{T \mu} + L \mu \left( \frac{2 \sigma^2}{K B \alpha} + \frac{2 \zeta^2_T(\alpha,\beta)}{K} + \frac{  M G^2 }{ K \alpha^3 } \sqrt{ \frac{ \frac{1}{2}\log M + \beta \log \left( M / \alpha \right) }{T}  } \right. \\
& \quad\quad\quad\quad\quad\quad\quad \left. + \frac{G}{\alpha} \sqrt{  \frac{1}{8 K T} } \left( 1 + 2 \log E \right) \right).
\end{align*}
Let
\begin{equation*}
\frac{2 D^F}{T \mu} = \frac{\mu L G}{\alpha} \sqrt{  \frac{1}{8 K T} } \left( 1 + 2 \log E \right),
\end{equation*}
we have
\begin{equation*}
\mu = \sqrt{  \frac{ 2 \alpha D^F }{ L G } } \left( \frac{8 K}{T} \right)^{\frac{1}{2}} \sqrt{  \frac{1}{1 + 2 \log E} },
\end{equation*}
and
\begin{equation*}
\frac{2 D^F}{T \mu} + \frac{\mu L G}{\alpha} \sqrt{  \frac{1}{8 K T} } \left( 1 + 2 \log E \right) \leq \sqrt{ \frac{ 2 D^F L G }{ \alpha } } \left( \frac{1}{8K} \right)^{\frac{1}{4}} \left( \frac{1}{T} \right)^{\frac{3}{4}} \sqrt{ 1 + 2 \log E }.
\end{equation*}
Then follow the same argument as in the proof of Theorem~\ref{thm:convg-thm-mini-batch}, when
\begin{multline*}
\mu=\min \left\{\frac{1}{L}, \frac{1}{\sigma}\sqrt{ \frac{ D^F K B \alpha  }{ L T }}, \frac{1}{\zeta_T(\alpha,\beta)} \sqrt{  \frac{D^F K}{L T} }, \right. \\
\left. \frac{ \sqrt{D^F K} \alpha^\frac{3}{2} }{\sqrt{L M} T^{\frac{1}{4}}  G \left( \frac{1}{2}\log M + \beta \log \left( M / \alpha \right) \right)^{\frac{1}{4}} }, \sqrt{  \frac{ \alpha D^F }{ L G } } \left( \frac{K}{T} \right)^{\frac{1}{2}} \sqrt{  \frac{1}{1 + 2 \log E} } \right\},
\end{multline*}
we have
\begin{equation*}
\begin{aligned}
& \quad \mathbb{E} \left[ \left\Vert \nabla F \left( w^R \right) \right\Vert^2 \right] \\
& \lesssim \frac{D^F L}{T} + \frac{ \sigma \sqrt{D^F L} }{ \sqrt{TKB \alpha } } + \frac{\zeta_T(\alpha,\beta) \sqrt{D^F L }}{\sqrt{T K}} + \frac{ \sqrt{ D^F L} M^{\frac{1}{2}} G }{ T^{\frac{3}{4}} K^{\frac{1}{2}} \alpha^{\frac{3}{2}} } \left( \frac{1}{2}\log M + \beta \log \left( M / \alpha \right)\right)^{\frac{1}{4}} \\
& \quad + \sqrt{ \frac{ D^F L G }{ \alpha } } \left( \frac{1}{K} \right)^{\frac{1}{4}} \left( \frac{1}{T} \right)^{\frac{3}{4}} \sqrt{ 1 + 2 \log E } .
\end{aligned}
\end{equation*}

For FedAvg, following~\eqref{eq:proof-thm-fedavg-new-1} and Theorem~\ref{thm:thm-dap-OSMD-sampler}, we have
\begin{equation*}
\begin{aligned}
\mathbb{E} \left[ \left\Vert \nabla F(w^R)  \right\Vert^2_2 \right] & \leq \frac{ 8 D^F}{\tilde{\mu} B T } + 8 \left( 2 + \frac{1}{\alpha}  \right) \tilde{\mu}^2 B^2 L^2 \left( \zeta^2_{\text{unif}}  + \frac{\sigma^2}{2 B} \right) \\
& \quad +  4 \tilde{\mu} B L  \left( \frac{2 \zeta^2_T (\alpha, \beta)}{K} + \frac{\sigma^2}{K B \alpha } + \frac{ 3 M G^2 }{ K \alpha^3 } \sqrt{ \frac{ \frac{1}{2} \log M + \beta \log (M/\alpha) }{T}   } \right. \\
& \quad\quad\quad\quad\quad\quad \left. + \frac{G}{\alpha} \sqrt{ \frac{ 1 }{ 8 K B T} } \left( 1 + 2 \log E \right) \right).
\end{aligned} 
\end{equation*}
Let
\begin{equation*}
\frac{ 8 D^F}{\tilde{\mu} B T } = 4 \tilde{\mu} B L \times \frac{G}{\alpha}  \sqrt{ \frac{ 1 }{ 8 K B T} } \left( 1 + 2 \log E \right),
\end{equation*}
we have
\begin{equation*}
\tilde{\mu} = \sqrt{ \frac{ 2 \alpha D^F  }{ L G \left( 1 + 2 \log E \right)  }  } \left( \frac{1}{B} \right)^{\frac{3}{4}} \left( \frac{8 K}{T} \right)^{\frac{1}{4}},
\end{equation*}
and
\begin{equation*}
\frac{ 8 D^F}{\tilde{\mu} B T } + 4 \tilde{\mu} B L \times \frac{G}{\alpha}  \sqrt{ \frac{ 1 }{ 8 K B T} } \left( 1 + 2 \log E \right) = 8^{\frac{3}{4}} \sqrt{2} \sqrt{ \frac{D^F L G}{\alpha}  } \left(  \frac{1}{K B}\right)^{\frac{1}{4}} \left( \frac{1}{T} \right)^{\frac{3}{4}} \sqrt{1 + 2 \log E}.
\end{equation*}
Then follow the same argument as in the proof of Theorem~\ref{thm:convg-thm-fedavg}, when $\mu \geq 1$ and
\begin{multline*}
\tilde{\mu} = \min \left\{ \frac{1}{4 B L} \sqrt{\frac{1}{1+1/(2\alpha)}},  \frac{(D^F)^{\frac{1}{3}}}{  \left( 2 + \frac{1}{\alpha}  \right)^{\frac{1}{3}} B L^{\frac{2}{3}} \left( \zeta^2_{\text{unif}}  + \frac{\sigma^2}{2 B} \right)^{\frac{1}{3}} T^{\frac{1}{3}} }, \right.\\ 
\left. \frac{ \sqrt{ 2 D^F }  }{  B \sqrt{L} \sqrt{ \frac{2 \zeta^2_T (\alpha, \beta)}{K} +  \frac{\sigma^2}{K B \alpha } + \sqrt{ \frac{ \frac{1}{2} \log M + \beta \log (M/\alpha) }{T} } } \sqrt{T} }, \right. \\
\left. \sqrt{ \frac{ 2 \alpha D^F  }{ L G \left( 1 + 2 \log E \right)  }  } \left( \frac{1}{B} \right)^{\frac{3}{4}} \left( \frac{8 K}{T} \right)^{\frac{1}{4}} \right\}, 
\end{multline*}
we have
\begin{multline*}
\mathbb{E} \left[ \left\Vert \nabla F(w^R)  \right\Vert^2_2 \right]
\lesssim \frac{D^F L \sqrt{1 + \frac{1}{2 \alpha}}}{T} + \frac{ \left( 2 + \frac{1}{\alpha}  \right)^{\frac{1}{3}} (D^F L)^{\frac{2}{3}} \zeta^{\frac{2}{3}}_{\text{unif}}  }{ T^{\frac{2}{3}} } + \frac{ \left( 2 + \frac{1}{\alpha}  \right)^{\frac{1}{3}} (D^F L)^{\frac{2}{3}}  \sigma^{\frac{2}{3}}  }{ B ^{\frac{1}{3}} T^{\frac{2}{3}} } \\
+ \frac{ \sqrt{ D^F L }\zeta_T (\alpha, \beta) }{ \sqrt{T K} } + \frac{ \sqrt{ D^F L } \sigma  }{ \sqrt{ T K B \alpha } } + \frac{ \sqrt{D^F L } }{ \sqrt{T} } \left( \frac{ \frac{1}{2} \log M + \beta \log (M/\alpha) }{T} \right)^{\frac{1}{4}} \\
+ \sqrt{ \frac{D^F L G}{\alpha}  } \left(  \frac{1}{K B}\right)^{\frac{1}{4}} \left( \frac{1}{T} \right)^{\frac{3}{4}} \sqrt{1 + 2 \log E}.
\end{multline*}

\subsection{Proof of Theorem~\ref{thm:adap-OSMD-sampler-doub-trick}}
\label{sec:proof-thm-5}

Recall that $T_b=2^{b-1}$.
Let $B=\lceil \log_2 (T+1) \rceil$, we then have $T_B \leq T \leq T_{B+1} - 1$, which implies that $1 \leq T - T_B + 1 \leq T_{B+1} - T_B = 2^B $.

Note that $\hat{p}^{T_b}$ is reinitialized as the uniform distribution. Let
\begin{equation*}
\text{D-Regret}_b  = \mathbb{E} \left[ \sum^{T_{b+1}-1}_{t=T_b} l_t (\hat{p}^t)  - \sum^{T_{b+1}-1}_{t=T_b} l_t (q^t) \right].
\end{equation*}
Similar to the proof of~\eqref{eq:proof-dap-1} and~\eqref{eq:proof-dap-2}, we have
\begin{align*}
& \quad \text{D-Regret}_b \\
& \leq  \frac{ 3 M^3 A_{\max} }{ K \alpha^3 } \sqrt{ \left[ \frac{1}{2}\log M + \log \left( M / \alpha \right) \mathbb{E} \left[ \text{TV} \left( q^{T_b:(T_{b+1}-1)} \right) \right] \right] (T_{b+1} - T_b) } \\
& \qquad \qquad + \frac{M}{\alpha} \sqrt{ \frac{ (T_{b+1}-T_b) A_{\max} }{ 8 K } } (1 + 2 \log E_b ) \\
& =  \frac{ 3 M^3 A_{\max} }{ K \alpha^3 } \sqrt{ \left[ \frac{1}{2}\log M + \log \left( M / \alpha \right) \mathbb{E} \left[ \text{TV} \left( q^{T_b:(T_{b+1}-1)} \right) \right] \right] } \left( \sqrt{2} \right)^{b-1} \\
& \qquad \qquad + \frac{M}{\alpha} \sqrt{ \frac{ A_{\max} }{ 8 K } } (1 + 2 \log E_b ) \left( \sqrt{2} \right)^{b-1} \\
& \leq  \frac{ 3 M^3 A_{\max} }{ K \alpha^3 } \sqrt{ \left[ \frac{1}{2}\log M + \log \left( M / \alpha \right) \beta \right] } \left( \sqrt{2} \right)^{b-1} + \frac{M}{\alpha} \sqrt{ \frac{ A_{\max} }{ 8 K } } (1 + 2 \log E ) \left( \sqrt{2} \right)^{b-1},
\end{align*}
where $E_b$ is defined in~\eqref{eq:doub-trick-set-length} and $E$ is defined in~\eqref{eq:expert-rates-set-length}.
We can similarly obtain
\begin{multline*}
\mathbb{E} \left[ \sum^T_{t=T_B} l_t (\hat{p}^t)  - \sum^T_{t=T_B} l_t (q^t) \right] \leq  \frac{ 3 M^3 A_{\max} }{ K \alpha^3 } \sqrt{ \left[ \frac{1}{2}\log M + \log \left( M / \alpha \right) \beta \right] }  \left( \sqrt{2} \right)^B \\
+ \frac{M}{\alpha} \sqrt{ \frac{ A_{\max} }{ 8 K } } (1 + 2 \log E ) \left( \sqrt{2} \right)^B.
\end{multline*}

\begingroup
\allowdisplaybreaks

Combine the last two displays, we have
\begin{align*}
\text{D-Regret}_T (q^{1:T}) &=  \mathbb{E} \left[ \sum^{T}_{t=1} l_t (\hat{p}^t)  - \sum^{T}_{t=1} l_t (q^t) \right] \\
& = \sum^{B-1}_{b=1} \mathbb{E} \left[ \sum^{T_{b+1}-1}_{t=T_b} l_t (\hat{p}^t)  - \sum^{T_{b+1}-1}_{t=T_b} l_t (q^t) \right] + \mathbb{E} \left[ \sum^T_{t=T_B} l_t (\hat{p}^t)  - \sum^T_{t=T_B} l_t (q^t) \right] \\
& \leq \frac{ 3 M^3 A_{\max} }{ K \alpha^3 } \sqrt{ \left[ \frac{1}{2}\log M + \log \left( M / \alpha \right) \beta \right] } \left( \sum^{B-1}_{b=1} \left( \sqrt{2} \right)^{b-1} + \left( \sqrt{2} \right)^{B} \right) \\ 
& \quad + \frac{M}{\alpha} \sqrt{ \frac{ A_{\max} }{ 8 K } } (1 + 2 \log E ) \left( \sum^{B-1}_{b=1} \left( \sqrt{2} \right)^{b-1} + \left( \sqrt{2} \right)^{B} \right) \\
& = \frac{ 3 M^3 A_{\max} }{ K \alpha^3 } \sqrt{ \left[ \frac{1}{2}\log M + \log \left( M / \alpha \right) \beta \right] } \frac{ \left( \sqrt{2} \right)^{B+1} - \sqrt{2} }{\sqrt{2}-1} \\ 
& \quad + \frac{M}{\alpha} \sqrt{ \frac{ A_{\max} }{ 8 K } } (1 + 2 \log E ) \frac{ \left( \sqrt{2} \right)^{B+1} - \sqrt{2} }{\sqrt{2}-1} \\
& \leq \frac{ 3 M^3 A_{\max} }{ K \alpha^3 } \sqrt{ \left[ \frac{1}{2}\log M + \log \left( M / \alpha \right) \beta \right] } \frac{ 2 }{\sqrt{2}-1} \sqrt{ T } \\ 
& \quad + \frac{M}{\alpha} \sqrt{ \frac{ A_{\max} }{ 8 K } } (1 + 2 \log E )  \frac{ 2 }{\sqrt{2}-1} \sqrt{ T } .
\end{align*}

\endgroup

\section{Useful Lemmas}

\begin{lemma}\label{lemma:convx-opt-cond}
Suppose that $f$ is a differentiable convex function defined on $\text{dom}f$, and $\mathcal{X} \subseteq \text{dom}f$ is a closed convex set. Then $x$ is the minimizer of $f$ on $\mathcal{X}$ if and only if
\begin{equation*}
\nabla f (x)^{\top} (y-x) \geq 0 \quad \text{for all } y \in \mathcal{X}. 
\end{equation*}
\end{lemma}
\begin{proof}
See Section 4.2.3 of~\cite{boyd2004convex}.
\end{proof}

\begin{lemma}
\label{lemma:lips-upper-bound}
Let $f:\mathcal{W} \to \mathbb{R}$ be defined on $\mathcal{W} \subseteq \mathbb{R}^d$, where $\mathcal{W}$ is a convex set. Suppose that $f$ is continuously differentiable and first-order $L$-smooth, that is,
\begin{equation*}
\Vert \nabla f(x) - \nabla f(y) \Vert \leq L \Vert x - y \Vert \quad \text{ for all } x,y \in \mathcal{W},
\end{equation*}
then we have
\begin{equation*}
\left\vert f(y) - f(x) - \left\langle \nabla f(x), y - x \right\rangle \right\vert \leq \frac{L}{2} \left\Vert x -y \right\Vert^2_2 \quad \text{ for all } x,y \in \mathcal{W}.
\end{equation*}
\begin{proof}
We follow Theorem 2.1.5 of~\cite{nesterov2018lectures}.
Note that
\begin{align*}
& \quad \left\vert f(y) - f(x) - \left\langle \nabla f(x), y - x \right\rangle \right\vert \\
& = \left\vert \int^1_0 \left\langle \nabla f\left(x + t (y-x) \right) , y - x \right\rangle dt - \left\langle \nabla f(x),  y - x\right\rangle \right\vert \\
& = \left\vert \int^1_0 \left\langle \nabla f\left(x + t (y-x) \right) - \nabla f(x) , y - x \right\rangle dt \right\vert \\
& \leq \int^1_0 \left\Vert \nabla f\left(x + t (y-x) \right) - \nabla f(x)  \right\Vert \left\Vert y - x  \right\Vert dt \\
& \leq \int^1_0 L t \left\Vert y - x  \right\Vert^2_2 dt \\
& = \frac{L}{2} \left\Vert y - x  \right\Vert^2_2,
\end{align*}
which then implies the conclusion.
\end{proof}
\end{lemma}

\begin{lemma}[Relaxed Triangle Inequality]
\label{lemma:relax-triangle-ineq}
This is Lemma~3 of~\cite{karimireddy2020scaffold}. Let $\{v_1,\ldots,v_\tau\}$ be $\tau$ vectors in $\mathbb{R}^d$. Then the following are true:
\begin{align*}
\Vert v_i + v_j \Vert^2 & \leq (1+a) \Vert v_i \Vert^2 + \left( 1+ \frac{1}{a} \right) \Vert v_j \Vert^2 \quad \text{ for any } a > 0, \\
\left\Vert \sum^{\tau}_{i=1}  v_i \right\Vert^2_2 & \leq \tau \sum^{\tau}_{i=1}  \left\Vert  v_i \right\Vert^2_2.
\end{align*}
\begin{proof}
See Lemma~3 of~\cite{karimireddy2020scaffold}.
\end{proof}

\end{lemma}

\begin{lemma}
\label{lemma:diam-max-Phi}
For $q \in \mathcal{P}_{M-1}$ we have $D_{\Phi}(q \, \Vert \, \punif) \leq \log M$, where $\Phi$ is the unnormalized negative entropy.
\end{lemma}
\begin{proof}
Since $\Phi(q)=\sum^M_{m=1} q_m (\log q_m - 1) \leq 0$, $\Phi(\punif) = - \log M$, and 
\[
\langle \nabla \Phi (\punif), q - \punif \rangle = \sum^M_{m=1} (q_m - \frac{1}{M}) \log \frac{1}{M} = 0,
\]
we have $D_{\Phi}(q \, \Vert \, p) \leq \log M$.
\end{proof}

\begin{lemma}[Hoeffding’s Inequality]
\label{lemma:Hoeffding-inequality}
Let $X$ be a random variable with $a \leq X \leq b$ for $a,b \in \mathbb{R}$. Then for all $s \in \mathbb{R}$, we have
\begin{equation*}
\log \mathbb{E} \left[ e^{sX} \right] \leq s \mathbb{E} [X] + \frac{s^2(b-a)^2}{8}.
\end{equation*}
\end{lemma}
\begin{proof}
See Section 2 of~\cite{wainwright2019high}.
\end{proof}

\begin{lemma}[Based on Exercise 26.12 of \cite{lattimore2020bandit}]
\label{lemma:mirror-solver-lemma}
Let $\alpha \in [0,1]$, $\mathcal{A}=\mathcal{P}_{M-1} \cap [ \alpha/M, 1 ]^M$, $\mathcal{D}=[0,\infty)^M$, and $\Phi$ is the unnormalised entropy on $\mathcal{D}$. For $y \in [0,\infty)^M$, let $x=\arg\min_{v \in \mathcal{A}} D_{\Phi}(v \Vert y)$. Suppose $y_1\leq y_2 \leq \dots \leq y_M$. Let $m^{\star}$ be the smallest value such that
\begin{equation*}
y_{m^{\star}} \left( 1 -  \frac{ m^{\star}-1 }{M} \alpha \right) > \frac{  \alpha}{M} \sum^M_{m=m^{\star}} y_m.
\end{equation*}
Then 
\begin{equation*}
x_m = 
\begin{cases}
\frac{\alpha}{M} & \text{if } m < m^{\star} \\
\frac{(1 - \frac{m^{\star} -1}{M} \alpha ) y_m}{\sum^M_{n=m^{\star}} y_n } & \text{otherwise}.
\end{cases}
\end{equation*}
\end{lemma}
\begin{proof}
Consider the following constrained optimization problem:
\begin{align*}
\min_{ u \in [0,\infty)^M } \, & \sum^M_{m=1} u_m \log \frac{u_m}{y_m}, \\
\text{s.t. } & \sum^M_{m=1} u_m = 1, \\
& u_m \geq \frac{\alpha}{M}, \qquad m \in [M].
\end{align*}
Since $x$ is the solution to this problem, by the optimality condition, there exists $\lambda,\nu_1,\dots,\nu_M \in \mathbb{R}$ such that
\begin{align}
\log \frac{x_m}{y_m} + 1- \lambda - \nu_m & =  0, 
\qquad m \in [M], \label{eq:solv-lemma-1} \\
\sum^M_{m=1} x_m & = 1, \label{eq:solv-lemma-2} \\
x_m - \frac{\alpha}{M} & \geq 0, \qquad m \in [M], \label{eq:solv-lemma-3} \\
\nu_m & \geq 0,  \qquad m \in [M], \label{eq:solv-lemma-4} \\
\nu_m \left( x_m - \frac{\alpha}{M} \right) &= 0, 
\qquad m \in [M]. \label{eq:solv-lemma-5}
\end{align}
By~\eqref{eq:solv-lemma-1}, we have $x_m=y_m \exp ( -1 + \lambda + \nu_m )$. By~\eqref{eq:solv-lemma-4} and \eqref{eq:solv-lemma-5}, when $x_m=\alpha/M$, we have $x_m=y_m \exp ( -1 + \lambda + \nu_m ) \geq y_m \exp ( -1 + \lambda )$; when $x_m>\alpha/M$, we have $x_m = y_m \exp ( -1 + \lambda )$.
Assume that $x_1 = \dots = x_{m^{\star}-1} = \alpha/M < x_{m^{\star}} \leq \dots \leq x_M$. Then
\begin{equation*}
1 = \sum^M_{m=1} x_m = ( m^{\star} -1 ) \frac{\alpha}{M} + \exp (-1+\lambda) \cdot \sum^M_{m=m^{\star}} y_m,
\end{equation*}
which implies that
\begin{equation}\label{eq:lemma4-proof-helper1}
\exp (-1+\lambda) = \frac{ 1 - ( m^{\star} -1 ) \frac{\alpha}{M}  }{  \sum^M_{m=m^{\star}} y_m }.
\end{equation}
Thus, we have
\begin{equation*}
x_{m^{\star}} = y_{m^{\star}} \exp (-1+\lambda) = y_{m^{\star}} \frac{ 1 - ( m^{\star} -1 ) \frac{\alpha}{M}  }{  \sum^M_{m=m^{\star}} y_m } > \frac{\alpha}{M}, 
\end{equation*}
which implies that
\begin{equation}\label{eq:solv-lemma-6}
y_{m^{\star}} \left( 1 -  \frac{ m^{\star}-1 }{M} \alpha \right) > \frac{  \alpha}{M} \sum^M_{m=m^{\star}} y_m .
\end{equation}

To complete the proof, we then only need to show that
\begin{equation}\label{eq:solv-lemma-7}
y_{m^{\prime}} \left( 1 -  \frac{ m^{\prime}-1 }{M} \alpha \right) \leq \frac{  \alpha}{M} \sum^M_{m=m^{\prime}} y_m
\end{equation}
for all $1 \leq m^{\prime} \leq m^{\star}-1$. The result then follows from~\eqref{eq:solv-lemma-6} and \eqref{eq:solv-lemma-7}.
To prove\eqref{eq:solv-lemma-7}, recall that for any $1 \leq m^{\prime} \leq m^{\star}-1$, we have $\alpha/M=y_{m^{\prime}} \exp ( -1 + \lambda + \nu_{m^{\prime}} )$, and because $y_1 \leq \cdots \leq y_M$, we have $\nu_1 \geq \cdots \geq \nu_{m^{\star}-1}$. This way, we have
\begin{equation}\label{eq:lemma4-proof-helper2}
\left( m^{\star}-m^{\prime} \right) \frac{\alpha}{M} = \sum^{m^{\star}-1}_{m=m^{\prime}} y_m \exp \left( -1 + \lambda + \nu_m \right) \leq \exp \left( -1 + \lambda + \nu_{m^{\prime}} \right) \sum^{m^{\star}-1}_{m=m^{\prime}} y_m.
\end{equation}
On the other hand, by~\eqref{eq:lemma4-proof-helper1}, we have
\begin{equation}\label{eq:lemma4-proof-helper3}
1 - ( m^{\star} -1 ) \frac{\alpha}{M} = \exp (-1+\lambda) \sum^M_{m=m^{\star}} y_m \leq \exp \left( -1 + \lambda + \nu_{m^{\prime}} \right) \sum^M_{m=m^{\star}} y_m.
\end{equation}
Combining \eqref{eq:lemma4-proof-helper2} and \eqref{eq:lemma4-proof-helper3}, we have
\begin{align*}
\frac{ 1 - ( m^{\prime} -1 ) \frac{\alpha}{M}  }{  \sum^M_{m=m^{\prime}} y_m } &=
\frac{ 1 - ( m^{\star} -1 ) \frac{\alpha}{M} + ( m^{\star} -m^{\prime} ) \frac{\alpha}{M} }{  \sum^M_{m=m^{\prime}} y_m } \\
& \leq \frac{ \exp \left( -1 + \lambda + \nu_{m^{\prime}} \right) \left( \sum^{m^{\star}-1}_{m=m^{\prime}} y_m + \sum^M_{m=m^{\star}} y_m \right)  }{  \sum^M_{m=m^{\prime}} y_m } \\
& = \exp \left( -1 + \lambda + \nu_{m^{\prime}} \right) \\
& = \frac{ \frac{  \alpha}{M} }{ y_{m^{\prime}} }, 
\end{align*}
which then implies~\eqref{eq:solv-lemma-7}.
\end{proof}

\section{Additional Future Directions}
\label{sec:future-directions}

In this section, we discuss two additional future directions.
In Section~\ref{sec:sampling-personalized}, we discuss the design of sampling algorithms for minimizing personalized FL objectives. Besides, in Section~\ref{sec:samp-phy-constraint}, we discuss sampling with physical constraint in FL system.

\subsection{Client Sampling with Personalized FL Objective}
\label{sec:sampling-personalized}

Data distributions across clients are often heterogeneous. Personalized FL has emerged as one effective way to handle such heterogeneity~\citep{kulkarni2020survey}. \cite{hanzely2021personalized} illustrated how many existing approaches to personalization can be studied through a unified framework, and, in this section, we discuss a natural extension of Adaptive-OSMD Sampler to this personalized objective. Specifically, we study the following optimization problem
\begin{equation}\label{eq:PrimProblem}
\min_{w, \beta}  F(w,\beta) \coloneqq \sum^{M}_{m=1} \lambda_m \phi(w,\beta_m ; \mathcal{D}_m),
\end{equation}
where $w\in \mathbb{R}^{d_0}$ corresponds to the shared parameter and $\beta = (\beta_1, \dots, \beta_M)$ with $\beta_m\in \mathbb{R}^{d_m}$ corresponds to the local parameters. The objective in \eqref{eq:PrimProblem} coves a wide range of personalized federated learning problems~\citep{hanzely2021personalized}. We further generalize the approach and study the following bilevel optimization problem:
\begin{equation}\label{eq:PrimProblem-bilevel}
\begin{aligned}
\min_{w}\  & h(w) \coloneqq \sum^{M}_{m=1} \lambda_m F_m(w,\hat{\beta}_m(w)) \coloneqq \sum^{M}_{m=1} \lambda_m \phi(w,\hat{\beta}_m (w) ; \mathcal{D}_m) \\
& \text{subject to}\quad \hat{\beta}_m  (w) = \arg\min_{\beta_m} G_m(w, \beta_m) \coloneqq \phi(w,\beta_m ; \bar{\mathcal{D}}_m).
\end{aligned}
\end{equation}
When $\mathcal{D}_m=\bar{\mathcal{D}}_m$, then~\eqref{eq:PrimProblem-bilevel} recovers~\eqref{eq:PrimProblem}. When $\bar{\mathcal{D}}_m \neq \mathcal{D}_m$, we then optimize the shared and local parameters on different data sets, which may prevent overfitting. The formulation in~\eqref{eq:PrimProblem-bilevel} is closely related to the implicit MAML~\citep{rajeswaran2019meta}.

In the following, we use $\nabla_{w}$ to denote a partial derivative with respect to $w$ with $\beta_m$ fixed, $\nabla_{\beta_m}$ to denote a partial derivative with respect to $\beta_m$ with $w$ fixed, and $\nabla$ to denote a derivative with respect to $w$ where $\beta_m (w)$ is treated as a function of $w$. Let $\nabla^2_{\beta_m \beta^{\top}_m} G_m (w, \beta_m) \in \mathbb{R}^{d_m\times d_m}$ be the Hessian matrix of $G_m$ with respect to $\beta_m$ where $w$ is fixed, and $\nabla^2_{w \beta^{\top}_m} G_m (w, \beta_m) \in \mathbb{R}^{d_0 \times d_m}$ be the Hessian matrix of $G_m$ with respect to $w$ and $\beta_m$, that is,
\begin{align*}
\left[ \nabla^2_{\beta_m \beta^{\top}_m} G_m (w, \beta_m) \right]_{i,j} & =  \frac{ \partial G_m (w, \beta_m) }{ \partial \beta_{m,i} \beta_{m,j} } \quad \text{for all} \, i,j = 1,2,\dots, d_m, \\
\left[ \nabla^2_{w \beta^{\top}_m} G_m (w, \beta_m) \right]_{i,j} & =  \frac{ \partial G_m (w, \beta_m) }{ \partial w_i \beta_{m,j} } \quad \text{for all} \, i = 1,2,\dots, d_0, j=1,2,\dots,d_m.
\end{align*}
By the implicit function theorem, we have
\begin{equation}
\label{eq:graidents-two-components}
\nabla h (w) = \underbrace{\frac{1}{M} \sum^M_{m=1} \lambda_m \nabla_1 F_m (w, \hat{\beta}_m (w) )}_{\nabla_1 h (w)} + \underbrace{\frac{1}{M} \sum^M_{m=1} \lambda_m \nabla_2 F_m (w, \hat{\beta}_m (w) )}_{\nabla_2 h (w)}
\end{equation}
where
\begin{align*}
\nabla_1 F_m (w, \hat{\beta}_m (w) ) & \coloneqq  \nabla_{w} F_m (w, \hat{\beta}_m (w)), \\
\nabla_2 F_m (w, \hat{\beta}_m (w) ) & \coloneqq - \nabla^2_{w \beta^{\top}_m} G_m (w, \hat{\beta}_m (w)) \left[ \nabla^2_{\beta_m \beta^{\top}_m} G_m (w, \hat{\beta}_m (w)) \right]^{-1} \nabla_{\beta_m} F_m (w, \hat{\beta}_m (w)).
\end{align*}
There are two parts to $\nabla h (w^t)$ and, therefore, instead of choosing a single subset of clients for computing both parts, we decouple $S^t$ into two subsets $S^t_1$ and $S^t_2$, $S^t=S^t_1 \cup S^t_2$. We use clients in $S^t_1$ to compute local updates of the first part, and clients in $S^t_2$ to compute the local updates of the second part. 
To get an estimate of $\nabla h (w)$, we can estimate $\nabla_1 h (w)$ and $\nabla_2 h (w)$ separably and then combine.
Assume that $g^t_{1,m}$ is an estimate of $\nabla_1 F_m (w, \hat{\beta}_m (w) )$ and $g^t_{2,m}$ is an estimate of $\nabla_2 F_m (w, \hat{\beta}_m (w) )$, we can then construct estimates of $\nabla_1 h (w)$ and $\nabla_2 h (w)$ as
\begin{align*}
g^t_1 =  \frac{1}{K_1} \sum_{m \in S^t_1} \lambda_m \frac{g^t_{1,m}}{p^t_{1,m}}, \qquad 
g^t_2 =  \frac{1}{K_2} \sum_{m \in S^t_2} \lambda_m \frac{g^t_{2,m}}{p^t_{2,m}},
\end{align*}
where $K_1=\vert S^t_1 \vert$ and $K_2=\vert S^t_2 \vert$. Then $g^t=g^t_1+g^t_2$ is an estimate of $\nabla h (w)$.

We design $p^t_1$ and $p^t_2$ to choose $S^t_1$ and $S^t_2$ by minimizing the variance of the gradients. Note that 
\begin{multline*}
\min_{ p^t_1} \mathbb{E}_1 \left[ \mathbb{E}_{S^t_1} \left[ \left\Vert g^t_1  - \nabla_1 h (w^t)  \right\Vert^2  \right] \right] + \min_{ p^t_2} \mathbb{E}_2 \left[ \mathbb{E}_{S^t_2} \left[ \left\Vert g^t_2  - \nabla_2 h (w^t)  \right\Vert^2  \right] \right] \\
\leq  \min_{ p^t_1 = p^t_2 = p^t } \mathbb{E} \left[ \mathbb{E}_{S^t} \left[ \left\Vert g^t_1  - \nabla_1 h (w^t)  + g^t_2  - \nabla_2 h (w^t)  \right\Vert^2  \right] \right],
\end{multline*}
so that the decomposition allows us to better minimize the variance. We term this approach as \emph{doubly variance reduction for personalized Federated Learning}. The first part minimizes the variance of updates to the shared global parameter, when the best local parameters are fixed; and the second part minimizes the variance of updates to local parameters, when the global part is fixed. While these two parts are related, any given machine will have different contributions to these two tasks.

Adaptive-OSMD Sampler can be used to minimize the variance for both parts of the gradient. We note that this is a heuristic approach to solving the client sampling problem when minimizing a personalized FL objective. Personalized FL objectives have additional structures that should be used to design more efficient sampling strategies. Furthermore, designing sampling strategies that improve the statistical performance of trained models, rather than improving computational speed, is important in the heterogeneous setting. Addressing these questions is an important area for future research.

\subsection{Sampling with Physical Constraint in FL System}
\label{sec:samp-phy-constraint}

In this paper, we assume that all clients are available in each round. However, in practical FL applications, a subset of the clients may be inactive due to physical constraints, thus we have to assign zero probabilities to them. In this section, we propose a simple extension of our proposed sampling method to such case. 

Specifically, denote the subset of clients that are active at the beginning of round $t$ as $I^t \subseteq [M]$. If we have $\vert I^t \vert \leq K$, we can then use all clients in $I^t$ to make updates in round $t$; otherwise, we would like to choose a smaller subset $S^t \subseteq I^t$ to participate. This can be achieved by rescaling the output sampling distribution of any of our proposed methods, which we denote as $\hat{p}^t$. We let $\tilde{p}^t_m = \hat{p}^t_m / (\sum_{i \in I^t} \hat{p}^t_i)$ and $\tilde{p}^t_m=0$ for all $m \notin I^t$. We can then use $\tilde{p}^t_m$ to choose $S^t$ from $I^t$.

However, analyzing such a method in terms of convergence and regret guarantee is highly non-trivial. Typically, for general active clients sequence $\{I^t\}^T_{t=1}$, the optimization algorithms are not guaranteed to converge even if we involve all clients in $I^t$ in each round. This can happen, for example, if a client is active for only once in the whole training process. Thus, to ensure convergence, we need additional assumptions about $\{I^t\}^T_{t=1}$. Moreover, deriving regret bound is also very challenging, as assigning zero probability to any client will make the variance-reduction loss unbounded and thus the regret can be arbitrarily large. To achieve such theoretical result, one may need to appropriately redefine the regret concept. Such an analysis is beyond the scope of this paper and we leave it for future research.

\section{Application of OSMD Sampler on SCAFFOLD}

{
We explore the potential of implementing our proposed sampler in more sophisticated federated learning optimization techniques. Beyond merely utilizing gradients, numerous state-of-the-art algorithms incorporate additional auxiliary variables to achieve a faster rate of convergence. These auxiliary variables complicate the process of constructing the appropriate surrogate variance reduction loss as outlined in~\eqref{eq:var-reduc-loss}, thus making the extension of our proposed sampling method to these algorithms challenging. Although a thorough exploration in this area is essential for future studies, in this section, we aim to provide insight by focusing on one of the most widely adopted state-of-the-art algorithms, SCAFFOLD~\citep{karimireddy2020scaffold}, as a case in point. We propose a straightforward strategy to incorporate our OSMD sampler into SCAFFOLD. Through simulation experiments, we validate our approach, with the results offering encouraging evidence of its effectiveness.

In addition to utilizing gradients, SCAFFOLD incorporates control variates on both the server and client sides to mitigate client-side heterogeneity. When implementing the OSMD sampler as described in Algorithm~\ref{alg:OSMD-sampler} within the SCAFFOLD framework, it is essential to define environment feedback $\{ a^t_m \}_{m \in S^t}$. For our experiments, we chose to set $a^t_m = \lambda^2_m \left\Vert w^{t+}_m - w^t \right\Vert^2$, where $w^t$ denotes the global model parameter at the start of round $t$, and $w^{t+}_m$ represents the updated local model parameter for client $m$ after performing local mini-batch SGD during round $t$. The learning rate for the OSMD sampler was set as $10^{-3}$. For all additional aspects of the SCAFFOLD algorithm, such as hyperparameter selection and tuning methodology, we adhered to the guidelines in the original paper.

Our experimental setup mirrors that of Section~\ref{sec:AdaDoubOSMD-uniform}. Figure~\ref{fig:loss-scaffold} presents the results, indicating that the OSMD sampler performs marginally better when the heterogeneity is high. This observation implies that utilizing adaptive sampling could potentially enhance SCAFFOLD's effectiveness. A more detailed investigation into this possibility is reserved for subsequent studies.
}

\begin{figure}[t]
\centering
\includegraphics[width=\textwidth]{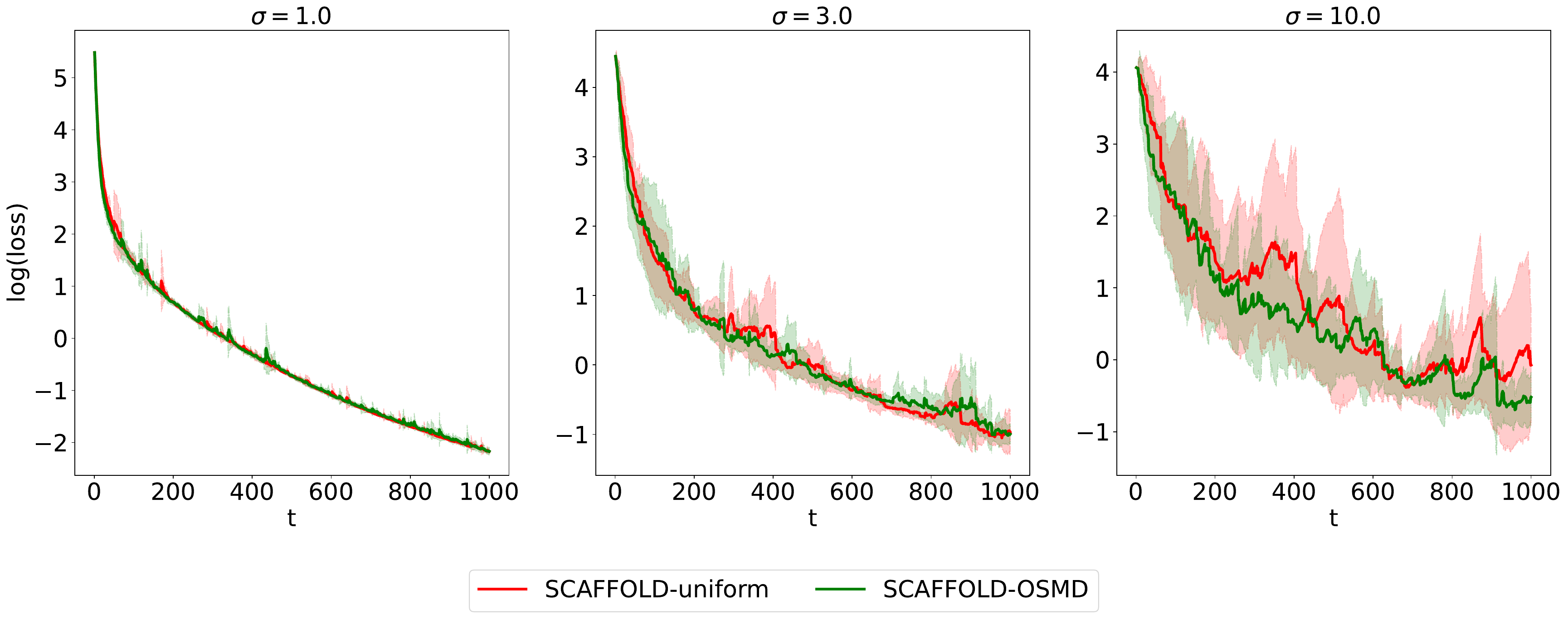}
\caption{The training loss is compared between SCAFFOLD with uniform sampling and SCAFFOLD with the OSMD sampler. Solid lines represent the mean values, while shaded regions indicate $\text{mean} \pm \text{standard deviation}$ across independent runs.}
\label{fig:loss-scaffold}
\end{figure}

\begingroup
\allowdisplaybreaks

\endgroup

\newpage

\bibliography{boxinz-papers}

\end{document}